\documentclass{article}

\usepackage{amsthm, amsmath, amssymb, bbm, mathabx}
\allowdisplaybreaks 
\usepackage{graphicx}
\usepackage{enumerate}
\usepackage{pifont}
\usepackage{booktabs}
\usepackage{multirow}
\usepackage{multicol}
\usepackage{stfloats}
\usepackage{xspace}
\usepackage{thmtools}
\usepackage{thm-restate}
\usepackage{hyperref}
\usepackage{cleveref}
\usepackage{anyfontsize} 
\usepackage{makecell}
\usepackage{hhline}
\usepackage{colortbl}
\usepackage{xpatch}
\usepackage{scalerel,stackengine}
\usepackage{sidecap}

\makeatletter
\def\@fnsymbol#1{\ensuremath{\ifcase#1\or *\or \dagger\or \ddagger\or
  \mathsection\or \mathparagraph\or \|\or \diamond \or **\or \dagger\dagger
  \or \ddagger\ddagger \else\@ctrerr\fi}}
\makeatother

\makeatletter
\newcommand{\printfnsymbol}[1]{%
  \textsuperscript{\@fnsymbol{#1}}%
}
\makeatother

\newtheorem{theorem}{Theorem}
\newtheorem{lemma}[theorem]{Lemma}

\crefname{condition}{Condition}{Conditions}

\crefname{assumption}{Assumption}{Assumptions}

\theoremstyle{definition}

\newcommand{\floor}[1]{\left\lfloor #1 \right\rfloor}
\newcommand{\ceil}[1]{\left\lceil #1 \right\rceil}
\newcommand{\abs}[1]{\left| #1 \right|}
\newcommand{\norm}[1]{\left\| #1 \right\|}

\newcommand\myeqii{\mathrel{\stackrel{\makebox[0pt]{\mbox{\normalfont\tiny (ii)}}}{=}}}

\newcommand\mylei{\mathrel{\stackrel{\makebox[0pt]{\mbox{\normalfont\tiny (i)}}}{\le}}}
\newcommand\myleii{\mathrel{\stackrel{\makebox[0pt]{\mbox{\normalfont\tiny (ii)}}}{\le}}}
\newcommand\myleiii{\mathrel{\stackrel{\makebox[0pt]{\mbox{\normalfont\tiny (iii)}}}{\le}}}
\newcommand\myleiv{\mathrel{\stackrel{\makebox[0pt]{\mbox{\normalfont\tiny (iv)}}}{\le}}}
\newcommand\mylev{\mathrel{\stackrel{\makebox[0pt]{\mbox{\normalfont\tiny (v)}}}{\le}}}

\newcommand\mygei{\mathrel{\stackrel{\makebox[0pt]{\mbox{\normalfont\tiny (i)}}}{\ge}}}
\newcommand\mygeii{\mathrel{\stackrel{\makebox[0pt]{\mbox{\normalfont\tiny (ii)}}}{\ge}}}
\newcommand\mygeiii{\mathrel{\stackrel{\makebox[0pt]{\mbox{\normalfont\tiny (iii)}}}{\ge}}}
\newcommand\mygeiv{\mathrel{\stackrel{\makebox[0pt]{\mbox{\normalfont\tiny (iv)}}}{\ge}}}

\newcommand{\wt}[1]{\widetilde{#1}}
\newcommand{\wh}[1]{\widehat{#1}}
\newcommand{\wc}[1]{\widecheck{#1}}
\newcommand{\ov}[1]{\overline{#1}}

\newcommand{\cA}{\mathcal{A}}
\newcommand{\cB}{\mathcal{B}}

\newcommand{\cE}{\mathcal{E}}
\newcommand{\cF}{\mathcal{F}}
\newcommand{\cG}{\mathcal{G}}

\newcommand{\cK}{\mathcal{K}}
\newcommand{\cL}{\mathcal{L}}
\newcommand{\cM}{\mathcal{M}}

\newcommand{\cR}{\mathcal{R}}
\newcommand{\cS}{\mathcal{S}}

\newcommand{\cV}{\mathcal{V}}

\newcommand{\cX}{\mathcal{X}}

\newcommand{\bbE}{\mathbb{E}}

\newcommand{\bbP}{\mathbb{P}}
\newcommand{\bbQ}{\mathbb{Q}}
\newcommand{\bbR}{\mathbb{R}}

\newcommand{\bbV}{\mathbb{V}}

\newcommand{\dif}{\mathop{}\!\mathrm{d}}
\newcommand{\II}{\mathbbm{1}} 

\newcommand{\KL}{\mathsf{KL}}
\newcommand{\kl}{\mathsf{kl}}

\newcommand{\poly}{\mathsf{poly}}
\newcommand{\polylog}{\mathsf{polylog}}

\newcommand{\Regret}{\mathsf{Regret}}

\newcommand{\supp}{\mathsf{supp}}
\newcommand{\boldone}{{\boldsymbol 1}}
\newcommand{\boldpi}{{\boldsymbol \pi}}

\newcommand{\Var}{\mathsf{Var}}
\newcommand{\VarR}{\mathsf{VarR}}

\newcommand{\bc}{\mathsf{bc}}
\newcommand{\BC}{\mathsf{BC}}

\newcommand{\SA}{\cS \times \cA}
\newcommand{\SAS}{\cS \times \cA \times \cS}
\newcommand{\sahk}{{s_h^k, a_h^k\xspace}}
\newcommand{\satk}{{s_t^k, a_t^k\xspace}}

\newcommand{\sumkh}{\sum_{k=1}^K \sum_{h=1}^H\xspace}

\newcommand{\multistepvar}{\Var_K^\Sigma}
\newcommand{\multistepvarepisode}{\Var_{(k)}^\Sigma}
\newcommand{\maxvar}{\Var^\star}

\newcommand{\tref}{{\mathsf{ref}}}
\newcommand{\REF}{{\mathsf{REF}}}

\newcommand{\sref}{\sigma^{\tref}}
\newcommand{\mref}{\mu^{\tref}}
\newcommand{\nref}{\nu^{\tref}}

\newcommand{\mc}{\wc{\mu}}
\newcommand{\nc}{\wc{\nu}}

\newcommand{\Vrl}{V_{h + 1}^{\tref, l_i}}
\newcommand{\Vrlc}{V_{h + 1}^{\tref, \wc{l}_i}}

\newcommand{\Brl}{B_{h + 1}^{\tref, l_i}}

\newcommand{\Vl}{V_{h + 1}^{l_i}}
\newcommand{\Vlc}{V_{h + 1}^{\wc{l}_i}}

\newcommand{\sli}{s_{h + 1}^{l_i}}
\newcommand{\slc}{s_{h + 1}^{\wc{l}_i}}

\newcommand{\is}{{i^\star}}

\newcommand{\euler}{{\texttt{Euler}}}
\newcommand{\mvp}{{\texttt{MVP}}}
\newcommand{\ucbadv}{{\texttt{UCB-Advantage}}}
\newcommand{\mvpalt}{{\texttt{MVP-V}}}
\newcommand{\ucbadvalt}{{\texttt{UCB-Advantage-V}}}
\newcommand{\qlucbb}{{\texttt{Q-learning (UCB-B)}}}
\newcommand{\qesadv}{{\texttt{Q-EarlySettled-Advantage}}}

\usepackage[colorinlistoftodos, textwidth=18mm]{todonotes}

\def\shownotes{1}
\ifnum\shownotes=0
\newcommand{\todorz}[1]{}
\newcommand{\todorzout}[1]{}
\newcommand{\todossdout}[1]{}
\newcommand{\todossd}[1]{}

\newcommand{\todoydt}[1]{}
\else
\newcommand{\todorz}[1]{\todo[color=blue!10, inline]{\small RZ: #1}}
\newcommand{\todorzout}[1]{\todo[color=blue!10]{\scriptsize RZ: #1}}
\newcommand{\todossdout}[1]{\todo[color=red!10]{\scriptsize SSD: #1}}
\newcommand{\todossd}[1]{\todo[color=red!10, inline]{\small SSD: #1}}

\newcommand{\todoydt}[1]{\todo[color=red!10, inline]{\small Yuandong: #1}}
\fi

\def\icml{1}

\usepackage[accepted]{icml2023/icml2023}

\icmltitlerunning{Near-Optimal Variance-Aware Bounds for Tabular MDPs}

\begin{document}

\twocolumn[
\icmltitle{Sharp Variance-Dependent Bounds in Reinforcement Learning: Best of Both Worlds in Stochastic and Deterministic Environments}



\icmlsetsymbol{equal}{*}

\begin{icmlauthorlist}
\icmlauthor{Runlong Zhou}{uwcse}
\icmlauthor{Zihan Zhang}{}
\icmlauthor{Simon S. Du}{uwcse}
\end{icmlauthorlist}

\icmlaffiliation{uwcse}{Paul G. Allen School of Computer Science \& Engineering, University of Washington, Seattle, WA, USA}

\icmlcorrespondingauthor{Simon S. Du}{ssdu@cs.washington.edu}

\icmlkeywords{variance, reinforcement learning, markov decision process}

\vskip 0.3in
]



\printAffiliationsAndNotice{}  

\begin{abstract}
We study variance-dependent regret bounds for Markov decision processes (MDPs).
Algorithms with variance-dependent regret guarantees can automatically exploit environments with low variance (e.g., enjoying constant regret on deterministic MDPs).
The existing algorithms are either variance-independent or suboptimal.
We first propose two new environment norms to characterize the fine-grained variance properties of the environment.
For model-based methods, we design a variant of the \mvp~algorithm~\citep{paper:mvp}.
We apply new analysis techniques to demonstrate that this algorithm enjoys variance-dependent bounds with respect to the norms we propose.
In particular, this bound is \emph{simultaneously minimax optimal for both stochastic and deterministic MDPs}, the first result of its kind.
We further initiate the study on model-free algorithms with variance-dependent regret bounds by designing a reference-function-based algorithm with a novel \emph{capped-doubling reference update schedule}. 
Lastly, we also provide lower bounds to complement our upper bounds.

\end{abstract}

{

\section{Introduction} \label{sec:intro}

We consider episodic reinforcement learning (RL) on tabular Markov Decision Processes (MDPs).
Existing algorithms can be categorized into two classes: model-based methods whose space complexity scales quadratically with the number of states \citep{paper:ucrl2,paper:optimistic_posterior_sampling,paper:ucbvi,paper:ubev,paper:orlc,paper:euler,paper:mvp} and model-free methods whose space complexity scales linearly with the number of states \citep{paper:q_learning,paper:q_learning_low_switch_cost,paper:ucbadv,paper:model_free_breaking}.

The MDPs in practice often enjoy benign structures, so problem-dependent regret bounds are of great interest~\citep{paper:euler}.
RL algorithms often perform far better on these MDPs than what their worst-case guarantees would suggest.
Motivated by this observation, we want to systematically study algorithms with regrets depending on quantities that characterizes the hardness of MDPs.
Ideally, such algorithms should \emph{automatically exploit the MDP instance without the prior knowledge of problem-dependent quantities}.
As a motivating example, for time-homogeneous MDPs with total reward bounded by $1$, the minimax regret bound for deterministic MDPs is $O (S A)$ where $S$ and $A$ are number of states and actions, respectively and the worst-case minimax optimal regret bound for stochastic MDPs is $\wt{O}\left(\sqrt{SAK}\right)$ where $K$ is the number of episodes.
Many problems can be formulated as deterministc MDPs, such as shortest path (maze, real world navigation), combinatorial optimization, Atari games \citep{mnih2013playing} and many games (mountain car, lunar lander, robotics, etc.) in OpenAI Gym \citep{brockman2016openai}.
Deterministic systems can also approximate stochastic systems well (see Section\,2 and 6 in \citet{book:dp_opt_control}).
We want an algorithm designed for generic stochastic MDPs with worst-case minimax optimal regret bound while enjoying the $O(SA)$ bound when the MDP is deterministic.

\citet{paper:euler} is a pioneer work which provides a model-based algorithm whose regret scales with variance-depedent quantities.
They defined a quantity, $\bbQ^\star$, named \emph{the maximum per-step conditional variance} to characterize the randomness of the MDP instance, and showed a regret bound of $\wt{O} (\sqrt{H \bbQ^\star \cdot S A K} + H^{5/2} S^2 A)$, where $H$ is the planning horizon.
This bound is still not satisfactory because:
\ding{172} There exist MDPs with $\bbQ^\star = \Omega (1)$, so the regret reduces to $\wt{O} (\sqrt{H S A K})$ which does not match the minimax optimal bound $\wt{O} (\sqrt{S A K})$.
\ding{173} For deterministic MDPs ($\bbQ^\star = 0$), the regret reduces to $\wt{O} (H^{5/2} S^2 A)$, which does not match the optimal $O (S A)$ bound.

\subsection{Contributions}
This paper makes the following contributions which significantly advance our understanding of problem-dependent bounds in reinforcement learning.


$\bullet$\quad First, We introduce \emph{the total multi-step conditional variance}, $\multistepvar$ and \emph{the maximum policy-value variance}, $\maxvar$, to provide fine-grained characterizations of the variance in the MDP  (see Section~\ref{sec:var} for the formal definitions).
Importantly, regret bounds that depend on these quantities will reduce to the minimax optimal bound in the worst case whereas the existing notion $H \bbQ^\star$ cannot.

$\bullet$\quad Second, for model-based methods, we identify the obstacles preventing the current state-of-the-art minimax optimal algorithm, \mvp~\cite{paper:mvp}, from being variance-dependent. 
We make necessary improvements and introduce a truncation method to bound the total variance.
We show the regret bound of the improved algorithm, \mvpalt, scales with $\maxvar$ or $\multistepvar$.
In particular, these bounds imply that, \mvpalt~is \emph{minimax optimal for both the classes of stochastic and deterministic MDPs}.
To our knowledge, this is first result of such kind. 
See Table~\ref{tab:mb} for comparions between model-based methods.

$\bullet$\quad Third, we initiate the study of model-free algorithms with variance-dependent regrets.
We explain why existings model-free algorithms cannot be variance-dependent.
We futher propose a new model-free algorithm, \ucbadvalt, which relies on a \emph{a capped-doubling manner of updates for reference values}. We further utilize a novel analysis technique which \emph{bounds value gaps directly from the existing uniform convergence bound} to give the first variance-dependent bound for model-free algorithms.
Importantly, this bound reduces to the minimax optimal bound for the worst-case MDPs.
See Table~\ref{tab:mf} for comparisons between model-free algorithms.

$\bullet$\quad Lastly, we prove minimax regret lower bounds for the class of MDPs with bounded variances.
We show that the main order terms of our regret upper bounds match these lower bounds, so our proposed algorithms are \emph{minimax optimal for the class of variance-bounded MDPs}.

\subsection{Technical Overview}
For model-based algorithms, existing state-of-the-art work \citep{paper:mvp} fails to be variance-dependent.
It is hard to bound the total variance by its expectation using martingale concentration inequalities directly, while avoiding an $H$-dependency.
This is because the total variance within an episode can be as large as $\Omega(H)$.
We introduce a novel analysis technique which \emph{truncates} the total variance of each episode to a constant and apply martingale concentration inequalities on this sequence, and show that with high probability there is no truncation.
We also apply a more refined concentration inequality to the transition model to have a dependency on the maximum support instead of the size of the state space. This step is crucial in obtaining the tight bound for deterministic MDPs.

For the model-free algorithm, existing work \citep{paper:ucbadv} upper-bounds all the four bias terms in their Equation\,(13) by variance-independent main order terms.
We identify the problem incurred by the large bias in reference values, and replace the update with \emph{a capped-doubling manner}.
Since too frequent updates discard past data very often, this method balances between the summation of gaps of value functions and the waste of data.
We \emph{integrate directly over the error} between the estimated value and the optimal value to bound the total squared gaps between them, whereas \citet{paper:ucbadv} bound them with a coarse binary gap of either $H$ or the final gap.
Combined with many other finer-grained analyses throughout the proof, we can finally remove all the variance-independent main order terms except for the total variance.

\subsection{Paper Overview}
The paper is organized as follows.
We first list basic concepts of MDPs in \Cref{sec:preliminaries}, then define variance quantities in \Cref{sec:var}.
Our main results then come in three sections: \Cref{sec:mb_results,sec:mf_results} show the algorithms, theorems, corollaries and proof sketches of our model-based and model-free methods, respectively.
\Cref{sec:lb_results} shows our lower bounds for the class of variance-bounded MDPs.

\section{Related Works} \label{sec:rel}

\paragraph{Minimax optimal regret bounds.}
Algorithms for regret minimization can be categorized into two classes: model-based and model-free.
Being model-free means the space complexity is $O (H S A)$, prohibiting the estimation of the whole transition model $P_h (s' | s, a)$.
For model-based methods, there are previous work \citep{paper:mvp,paper:horizon_free,wang2020long} achieving a property called \emph{horizon-free}, which allows only logarithmic dependency on $H$ for regrets.
As explained in \citet{jiang2018open}, in many scenarios with a long planning horizon, the interesting regime is $K \ll H$.
This underscores the importance of being horizon-free, because for $H$-dependent bounds, only when $K \gg H$ they become sub-linear in $K$.
Being horizon-free is challenging, because it requires utilizing transition data for the same state-action pair from different steps and handling a spike in rewards.
There are many works other than those we cite in \Cref{sec:intro} giving nearly minimax optimal  bounds: \citet{paper:regal,osband2013more,osband2017posterior,fruit2018near,paper:variance_aware_kl_ucrl,paper:non_asymptotic_gap,russo2019worst,zhang2019regret,paper:mdp_optimism,paper:random_explore,paper:opt_mb_rl}.
We compare our results with the state-of-the-art in \Cref{tab:mb} (model-based) and \Cref{tab:mf} (model-free).

\colorlet{shadecolor}{gray!40}
\begin{table*}[t!] 
    \centering
    \small
    \resizebox{0.99\linewidth}{!}{%
        \renewcommand{\arraystretch}{1.5}
        \begin{tabular}{|c|c|c|c|c|c|}
            \hline 
            \textbf{Algorithm} &  \textbf{Regret} & \scriptsize \makecell{\textbf{Variance-} \\ \textbf{Dependent}} & \scriptsize \makecell{\textbf{Stochastic-} \\ \textbf{Optimal}} & \scriptsize \makecell{\textbf{Deterministic-} \\ \textbf{Optimal}} & \scriptsize \makecell{\textbf{Horizon-} \\ \textbf{Free}}\\
            
            \hhline{|=|=|=|=|=|=|}
            \multirow{2}{*}{\makecell{\euler \\ \citet{paper:euler}}}  &  $\wt{O} (\sqrt{H \bbQ^\star \cdot S A K} + H^{5/2} S^2 A) $ &  \textbf{Yes} & No & No & No \\
            \hhline{~-|-|-|-|-}
            &$\wt{O} (\sqrt{S A K} + H^{5/2} S^2 A) $ & No & \textbf{Yes} & No & No \\
            
            \hline
            \makecell{\mvp \\ \citet{paper:mvp}} &  $ \wt{O} (\sqrt{S A K} + S^2 A) $ & No & \textbf{Yes} & No & \textbf{Yes} \\

            \hline
            \rowcolor{shadecolor} \Gape[0pt][2pt]{\makecell{\mvpalt \\ This work}} & $\wt{O} (\sqrt{\min\{\multistepvar, \maxvar K \} S A} + \Gamma S A)$ & \textbf{Yes} & \textbf{Yes} & \textbf{Yes} & \textbf{Yes} \\
            \hline
        \end{tabular}
    }
    \caption{
    Comparisons between model-based algorithms for time-homogeneous MDPs with total reward bounded by $1$.
    $\wt{O}$ hides logarithm factors.
    $S$, $A$, $\Gamma$, $H$ and $K$ are number of states, actions, maximum support of the transition model, planning horizon and interaction episodes.
    $\bbQ^\star$, $\multistepvar$ and $\maxvar$ are variance notations in \Cref{sec:var}.
    $\bbQ^\star$ and $\multistepvar$ are upper bounded by $1$ in the worst case and become $0$ when the MDP is deterministic.
    An ``Yes'' in each column means:
    Variance-Dependent: the regret has a main order term scaling with any variance notation.
    Stochastic-Optimal: the regret has a main order term of $\wt{O} (\sqrt{S A K})$ which matches the minimax lower bound.
    Deterministic-Optimal: the regret is $\wt{O} (S A)$ on deterministic MDPs (with variance equal to $0$).
    Horizon-Free: the regret has only logarithmic dependency on $H$.
    }
    \label{tab:mb}
\end{table*}

\colorlet{shadecolor}{gray!40}
\begin{table*}[t!] 
    \centering
    \small
    \resizebox{0.7\linewidth}{!}{%
        \renewcommand{\arraystretch}{1.5}
        \begin{tabular}{|c|c|c|c|}
            \hline 
            \textbf{Algorithm} &  \textbf{Regret} & \scriptsize \makecell{\textbf{Variance-} \\ \textbf{Dependent}} & \scriptsize \makecell{\textbf{Stochastic-} \\ \textbf{Optimal}} \\
            
            \hhline{|=|=|=|=|}
            \makecell{\qlucbb \\ \citet{paper:q_learning}} &  $ \wt{O} (\sqrt{H^4 S A K} + H^{9/2} S^{3/2} A^{3/2}) $  & No & No \\

            \hline
            \makecell{\ucbadv \\ \citet{paper:ucbadv}} &  $ \wt{O} (\sqrt{H^3 S A K} + \sqrt[4]{H^{33} S^8 A^6 K}) $  & No & \textbf{Yes} \\

            \hline
            \makecell{\texttt{Q-EarlySettled-} \\ \texttt{Advantage} \\ \citet{paper:model_free_breaking}} &  $ \wt{O} (\sqrt{H^3 S A K} + H^6 S A) $  & No & \textbf{Yes} \\

            \hline
            \rowcolor{shadecolor} \Gape[0pt][2pt]{\makecell{\ucbadvalt \\ This work}} & \Gape[0pt][2pt]{\makecell{$\wt{O} (\sqrt{\min\{ \multistepvar, \maxvar K\} H S A} $ \\ $+ \sqrt[4]{H^{15} S^5 A^3 K})$}}  & \textbf{Yes} & \textbf{Yes} \\
            
            \hline
        \end{tabular}
    }
    \caption{
    Comparison between model-free algorithms for time-inhomogeneous MDPs with every reward bounded by $1$.
    An ``Yes'' in each column means: 
    Variance-Dependent: the bound scales with the variance term that characterizes the randomness of the environment;
    Stochastic-Optimal: in the worst-case, the regret's dominating term becomes $\wt{O} (\sqrt{H^3 S A K})$ which matches the minimax lower bound.
    }
    \label{tab:mf}
\vspace{-0.5cm}
\end{table*}

\paragraph{Variance-dependent results.}
\citet{paper:variance_aware_kl_ucrl} provides a problem-dependent regret bound that scales with the variance of the next step value functions under strong assumptions on ergodicity of the MDP.
Namely, they define $\boldsymbol{V}_{s, a}^\star$ for each $(s, a)$ pair and derives a regret of $\tilde{O} (\sqrt{S \sum_{s, a} \boldsymbol{V}_{s, a}^\star T})$ under the infinite horizon setting.

\citet{paper:non_asymptotic_gap} combines gap-dependent regret with variances.
The standard notation $\texttt{gap} (s, a)$ is the gap between the optimal value function and the optimal $Q$-function, and $\texttt{gap}_{\min}$ is the minimum non-zero gap.
Let $\texttt{Var}_{h, s, a}^\star$ be the variance of optimal value function at $(h, s, a)$ triple, their regret approximately scales as
\begin{align*}
    \tilde{O} \left( \sum_{s, a} \frac{H \max_h \texttt{Var}_{h, s, a}^\star}{\max\{\texttt{gap} (s, a),\ \texttt{gap}_{\min}\} } \log (K) \right).
\end{align*}
Variance-aware bounds exist in bandits~\citep{kim2021improved,zhang2021improved,zhou2021nearly,zhao2023variance,zhao2022bandit}.
We notice two concurrent works: \citet{zhao2023variance} studies variance-dependent regret upper bounds for linear bandits and linear mixture MDPs, and \citet{li2023variance} studies linear bandits and linear MDPs.
They both define the same variance as $\multistepvar$ (one of the two quantities also proposed by us under the tabular setting).
More recent work generalized variance-aware bound from MDPs to latent MDPs~\citep{paper:horizon_free_lmdp}.

\paragraph{Other problem-dependent results.}
Most problem-dependent results prior to \citet{paper:euler} focus on the infinite-horizon setting.
Some depend on the range of value functions \citep{paper:regal,paper:scal}.
\citet{paper:distribution_norm} introduces a hardness measure called \emph{distribution norm}.
There are gap-dependent results for multi-armed bandits and RL \citep{paper:action_elimination,paper:ucrl2,paper:gap_dependent_multi_step_bootstrap,paper:q_learning_log_regret}.

\citet{paper:reward_free_exploration} shows that with a slight modification, the algorithm in \citet{paper:euler} can have a \emph{first-order regret}, with the main order term depending on the optimal value function.
\citet{paper:first_order_lfa} provides a first-order regret for linear MDPs.
When the total reward is bounded by $1$ almost surely, for any policy its variance is not larger than this value.
This means a first-order dependency is weaker than a variance-dependency.

\section{Preliminaries} \label{sec:preliminaries}

\paragraph{Notations.}
For any event $\cE$, we use $\II [\cE]$ to denote the indicator function.
For any random variable $X$, we use $\bbV (X)$ to denote its variance.
For any set $\cX$, we use $\Delta(\cX)$ to denote the probability simplex over $\cX$.
For any positive integer $n$, we denote $[n] := \{1, 2, \ldots, n\}$.
For any probability distribution $P$, we use $\supp (P) = \sum_x \II [P(x) > 0]$ to denote the size of its support.
Suppose $x$ and $y$ are $n$-dimensional vectors, we denote $x y := \sum_{i = 1}^n x_i y_i$ and $x^q := (x_1^q, x_2^q, \ldots, x_n^q)$ for any number $q$.
If $x \in \Delta ([n])$, we use $\bbV (x, y) = \sum_i x_i (y_i - x y)^2 = x y^2 - (x y)^2$ to denote the variance of $y$ under $x$.
We use $\boldone_k$ to denote a vector with all $0$ but an only $1$ on the $k$-th position.

\paragraph{Finite-horizon MDPs.}
A finite-horizon MDP can be described by a tuple $M = (\cS, \cA, P, R, H)$.
$\cS$ is the finite state space with size $S$ and $\cA$ is the finite action space with size $A$.
For any $h \in [H]$, $P_h: \SA \to \Delta (\cS)$ is the transition function and $R_h: \SA \to \Delta ([0, 1])$ is the reward distribution with mean $r_h: \SA \to [0, 1]$.
$H$ is the planning horizon, i.e., episode length.
We denote $\Gamma := \max_{h, s, a} \supp ( P_h (\cdot | s, a) )$ as the maximum support of the transition function, and $P_{s, a, h} := P_h (\cdot | s, a)$.

\paragraph{Conditions for MDPs.}
We have two conditions more general than the ordinary setting.
As explained below them, getting tight regret bounds are harder when they are met.


\begin{restatable} {condition}{conBoundedReward} \label{asp:bounded_reward}
For any policy $\pi$, the total reward in a single episode is upper-bounded by $1$ almost surely.
\end{restatable}

For those MDPs not satisfying \Cref{asp:bounded_reward}, we can normalize all the rewards by $1 / H$.
Such a conversion usually multiplies a factor of $1 / H$ to the regret, but cannot take into account a spike in rewards, e.g., some $r_h (s, a) = 1$. 

\begin{restatable} {condition}{conHomogeneous} \label{asp:homogeneous}
The MDP is time-homogeneous.
Namely, there exist $P: \SA \to \Delta(\cS)$, $R: \SA \to \Delta([0, 1])$ and $r: \SA \to [0, 1]$ such that for any $(s, a) \in \SA$, $P_h (\cdot | s, a) = P (\cdot | s, a)$, $R_h (s, a) = R (s, a)$ and $r_h (s, a) = r (s, a)$ for any $h \in [H]$.
\end{restatable}

For simplicity, we denote $P_{s, a} := P (\cdot | s, a)$ and $P_{s, a, s'} := P (s' | s, a)$.
Any time-inhomogeneous MDP can be instantiated by a time-homogeneous one to satisfy \Cref{asp:homogeneous}.
Let a mega state space $\cS = \cup_{h = 1}^H \cS_h$, where each $\cS_h$ corresponds to the state space of the time-inhomogeneous MDP.
For any $(h, s, a)$, we construct $P (s_{h + 1}' | s_h, a) = P_h (s' | s, a)$ and $R (s_h, a) = R_h (s, a)$, where $s_h$ is the corresponding state of $s$ in $\cS_h$.
$S$ is multiplied by $H$ while $\Gamma$ remains the same, and the regret changes accordingly.
This condition underscores the algorithm's ability to use information of the same state-action pair from different steps.

We will introduce quantities in \Cref{sec:var} to quantify determinism, but a fully-deterministic MDP is very worth studying because the regret lower bound is the well-known $\Omega (S A)$ (under \Cref{asp:bounded_reward,asp:homogeneous}).
Thus, we care about whether the algorithms can have a constant regret (up to logarithm factors) under the assumption of determinism.

\begin{restatable} {assumption}{aspDeterm} \label{asp:determ}
The MDP is deterministic.
Namely, for any $(h, s, a) \in [H] \times \SA$, $R_h (s, a)$ and $P_h (\cdot | s, a)$ map to a single real number and a single state respectively.
\end{restatable}

\paragraph{Policies.}
A history-independent deterministic policy $\pi$ chooses an action based on the current state and time step.
Formally, $\pi = \{ \pi_h \}_{h \in [H]}$ where $\pi_h : \cS \to \cA$ maps a state to an action.
We use $\Pi$ to denote the set of all such policies.

\paragraph{Episodic RL on MDPs.}
Upon choosing action $a$ at state $s$ when it is the $h$-th step in an episode, the agent observes a reward $r \sim R_h (s, a)$ and the next state $s' \sim P_h (\cdot | s, a)$.
When $h = H$, the episode ends after this observation.
Thus, a policy $\pi$ induces a (random) trajectory $(\{s_h, a_h, r_h\}_{h \in [H]}, s_{H + 1})$ where $s_1$ is exogenously generated, $a_h = \pi_h (s_h), r_h \sim R_h (s_h, a_h)$ and $s_{h + 1} \sim P_h (\cdot | s_h, a_h)$ for $h \in [H]$.
The episodic RL on MDPs proceeds in a total of $K$ episodes.
When one episode ends, a new initial state $s_1$ is generated.
The agent should (adaptively) choose a policy $\pi^k$ for the $k$-th episode, put it into action and cannot change it within an episode.

\paragraph{Value functions and $Q$-functions.}
Given a policy $\pi$, we define its value function and $Q$-function as
\begin{gather*}
    V_h^\pi (s) := \bbE_\pi \left[\left. \sum_{t = h}^H r_t\ \right|\ s_h = s \right], \\
    Q_h^\pi (s, a) := \bbE_\pi \left[\left. \sum_{t = h}^H r_t\ \right|\ (s_h, a_h) = (s, a) \right].
\end{gather*}
It is easy to verify that $Q_h^\pi (s, a) = r_h (s, a) + P_{s, a, h} V_{h + 1}^\pi$.

\paragraph{Performance measure.}
The goal of episodic RL on MDPs is to find the policy which maximizes the total reward collected in an episode, regardless of the initial state.
Given $M$, such a goal can be achieved using dynamic programming.
Given this, we denote $V^\star := V^{\pi^\star}$ and $Q^\star := Q^{\pi^\star}$.
We use cumulative regret as a performance measure:
\begin{align*}
    \Regret(K) := \sum_{k = 1}^K (V_1^\star (s_1^k) - V_1^{\pi^k} (s_1^k)).
\end{align*}

\section{Variance Quantities for MDPs} \label{sec:var}

We use the notion of variance to quantify the degree of determinism of MDPs.
The first is called \emph{the maximum per-step conditional variance} \citep{paper:euler}, which is only relevant to the optimal value function.

\begin{restatable} {definition}{defQStar} \label{def:Q_star}
The maximum per-step conditional variance for a particular MDP is defined as:
\begin{align*}
    \bbQ^\star := \max_{h, s, a} \{ \bbV (R_h (s, a)) + \bbV (P_{s, a, h}, V_{h + 1}^\star) \}.
\end{align*}
\end{restatable}

\citet{paper:euler} directly use $H \bbQ^\star$ to upper-bound the total per-step conditional variances in an episode, a quantity which can be upper-bounded by a constant (see \Cref{lem:M6,lem:mf_sum_opt_var}).
So when  $\bbQ^\star \ge \Omega (H)$ (or $\Omega (1 / H)$ under \Cref{asp:bounded_reward}), $H \bbQ^\star$ is not tight.
In light of this, we define \emph{the total multi-step conditional variance} as a better notation in place of $H \bbQ^\star$.
This quantity is also proposed in concurrent works \citep{zhao2023variance,li2023variance}.

\begin{restatable} {definition}{defMultistepVar} \label{def:multistep_var}
The total multi-step conditional variance for a trajectory $\tau = \{s_h, a_h\}_{h \in [H]}$ is defined as:
\begin{align*}
    \Var_\tau^\Sigma := \sum_{h = 1}^H (\bbV (R_h (s_h, a_h)) + \bbV (P_{s_h, a_h, h}, V_{h + 1}^\star)).
\end{align*}
During the learning process, let the trajectory of the $k$-th episode be $\tau^k$, then we denote $\multistepvarepisode := \Var_{\tau^k}^\Sigma$, and
$\multistepvar := \sum_{k = 1}^K \multistepvarepisode$.
\end{restatable}


We introduce another type of variance, called \emph{the maximum policy-value variance}, which is novel in the literature.

\begin{restatable} {definition}{defVStar} \label{def:V_star}
For any policy $\pi \in \Pi$, its maximum value variance is defined as $\Var^\pi := \max_{s \in \cS} \Var_1^\pi (s)$, where
\begin{align*}
    &\Var_1^\pi (s) := \\
    &\bbE_\pi \left[ \left. \sum_{h = 1}^H \left( \bbV (R_h (s_h, a_h)) + \bbV (P_{s_h, a_h, h}, V_{h + 1}^\pi) \right) \ \right|\ s_1 = s\right].
\end{align*}
The maximum policy-value variance for a particular MDP is defined as:
\begin{align*}
    \maxvar := \max_{\pi \in \Pi} \Var^\pi.
\end{align*}
\end{restatable}

$\Var_1^\pi (s)$ is the variance of total reward of $\pi$ starting from $s$, and the justification can be found in \Cref{sec:justify_V_star}.

Under \Cref{asp:bounded_reward}, by \Cref{lem:bhatia_davis} we know that $\Var_1^\pi (s) \le V_1^\pi (s) \le V_1^\star (s)$.
So $\maxvar \le V_1^\star (s)$.
This means a variance-dependent regret is better than a first-order regret.

\subsection{Comparing $\multistepvarepisode$ and $\maxvar$} \label{sec:compare_var}

We use this subsection to demonstrate that a small $\multistepvarepisode$ does not imply a small $\maxvar$, and vice versa.

Deterministic MDPs have $\multistepvarepisode = \maxvar = 0$.
Trivially, $\maxvar = 0 \implies \multistepvarepisode = 0$, while the reverse is not ture.

\paragraph{When $\multistepvarepisode = 0 < \maxvar$.}
Consider the following \emph{time-homogeneous} MDP with 
horizon $H$:
For each state $s$ there is a good action $a_1$ with a \emph{deterministic} reward $r (s, a_1) = 1 / H$, and all other actions $a'$ have a \emph{deterministic} reward $r (s, a') = 0$.
For any state-action pair $(s, a)$, the transition is identically $P_{s, a, s'} = 1 / S$.

The optimal policy always chooses $a_1$ at any state $s$, so for any $s$ and $h$, $V_h^\star (s) = (H - h + 1) / H$.
For any $(h, s, a)$, 
\begin{align*}
    \bbV (R (s, a)) + \bbV (P_{s, a}, V_{h + 1}^\star) = 0,
\end{align*}
which means $\multistepvarepisode = 0$.
However, let $\pi$ be a policy with $\pi_{H} (s_1) = a'$ for a certain state $s_1$, and $\pi_h (s) = a_1$ for any other $h$ or $s$.
Then $\pi$ yields cumulative rewards of $1$ and $1 - 1 / H$, each with non-zero probabilities.
So $\maxvar > 0$.

This example shows that deterministic MDPs are not the only MDPs satisfying $\multistepvarepisode = 0$, and that $\multistepvarepisode = 0$ does not imply $\maxvar = 0$.

\ifnum \icml=1

\begin{SCfigure}
    \includegraphics[scale=0.6]{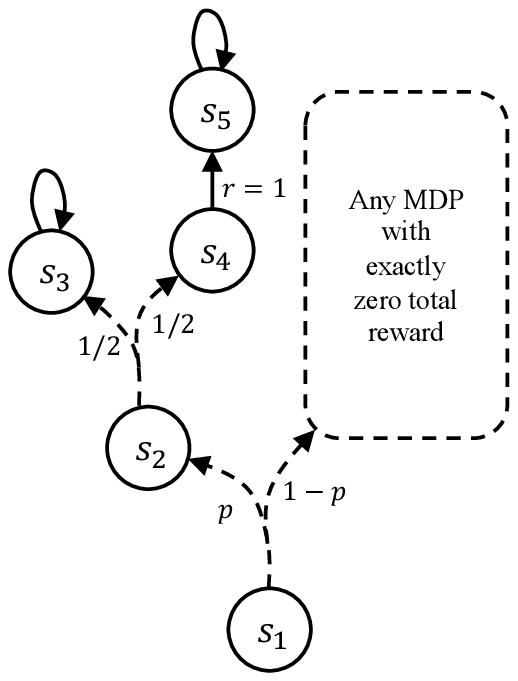}
    \caption{Example of $\maxvar$ being arbitrarily smaller than $\multistepvarepisode$.
    Dashed arrows represent probabilistic transitions and solid arrows represent deterministic ones.
    The only reward comes at state $s_4$ and on choosing any action.}
    \label{fig:small_vstar}
\end{SCfigure}

\else 

\begin{figure}
    \centering
    \includegraphics[scale=0.8]{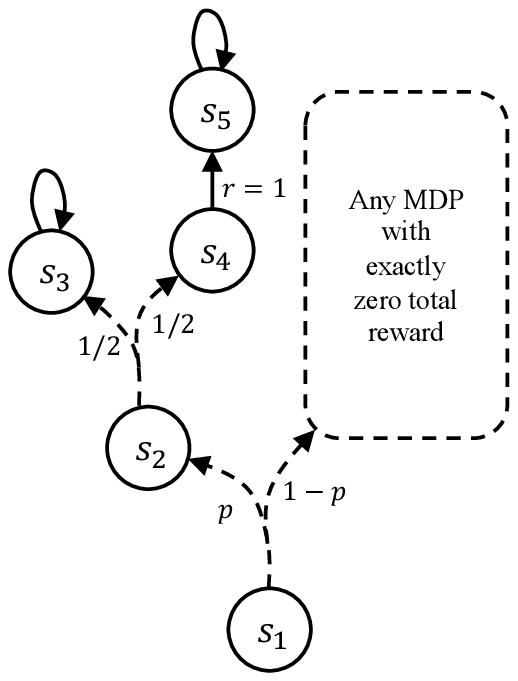}
    \caption{Example of $\maxvar$ being arbitrarily smaller than $\multistepvarepisode$.
    Dashed arrows represent probabilistic transitions and solid arrows represent deterministic ones.
    The only reward comes at state $s_4$ and on choosing any action.}
    \label{fig:small_vstar}
\end{figure} 

\fi

\paragraph{When $\multistepvarepisode = 1 / 4 > \maxvar$.}
Consider the \emph{time-homogeneous} MDP in \Cref{fig:small_vstar}:
$P_{s_1, a, s_2} = p$ for any $a \in \cA$, and the rest probability is into an MDP with \emph{no reward at all}.
$s_2$ is a state which we want to have a high $\multistepvarepisode$: $P_{s_2, a, s_3} = P_{s_2, a, s_4} = 1 / 2$, where $s_3$ and $s_4$ are states with value $0$ and $1$ respectively.
Thus at $s_2, a, h = 3$,
\begin{align*}
    \multistepvarepisode \ge \bbV( R (s_2, a) ) + \bbV ( P_{s_2, a}, V_3^\star ) = \frac{1}{4}.
\end{align*}
We also have that for any policy $\pi$, $V_1^\pi (s_1) = p / 2$, so by \Cref{lem:bhatia_davis}, $\maxvar \le p / 2$.
Taking $p$ arbitrarily small gives an arbitrarily large gap between $\multistepvarepisode$ and $\maxvar$.

This example shows that a small $\maxvar$ does not imply a small $\multistepvarepisode$, so using only $\multistepvarepisode$ is insufficient.

\section{Results of the Model-Based Algorithm} \label{sec:mb_results}

We propose \mvpalt~(\Cref{alg:mvp_alt}, where ``V'' stands for ``Variance-dependent''), a model-based algorithm with a variance-dependent regret bound.
Based on \mvp~\citep{paper:mvp} which is minimax optimal under \Cref{asp:bounded_reward,asp:homogeneous}, we make necessary alterations to make the regret variance-dependent.

\paragraph{Common notations.}
These are notations shared with our model-free algorithm.
Let $\sahk$ and $r_h^k$ denote the state, action and reward at the $h$-th step of the $k$-th episode.
Let $V_h^k$ and $Q_h^k$ denote $V_h$ and $Q_h$ at the beginning of the $k$-th episode.
$\wt{O}$ hides $\polylog (H, S, A, K, 1 / \delta)$ factors.

\ifnum \icml=1
\begin{algorithm}[t!]
	\caption{\mvpalt \label{alg:mvp_alt}}
	\begin{algorithmic}[1] \small
		\STATE{ \textbf{Input and initialize:} Logarithm term $\iota$; Trigger set $\mathcal{L} \gets \{ 2^{i-1}\ |\ 2^{i}\leq KH, i=1,2,\ldots \}$.}
        \STATE{ Set all $N (s, a)$, $n (s, a)$, $\theta (s, a)$, $\phi (s, a)$, $N (s, a, s')$, $\wh{P}_{s,a,s'}$ to be $0$ and all $Q_h (s, a)$, $V_h (s)$ to be $1$.}
		\FOR {$k=1,2,\ldots, K$}
          \STATE{ Observe $s_1^k$.}
		\FOR {$h=1,2,\ldots, H$}
		\STATE{ Take action $ a_h^k= \arg\max_{a}Q_h(s_h^k,a)$;}
		\STATE{ Receive reward $r_h^k$ and observe $s_{h+1}^k$;}
		\STATE{ Set $(s,a,s',r)\gets (s_h^k,a_h^k,s_{h+1}^k,r_h^k)$;}
		\STATE{ Set $N(s,a) \stackrel{+}{\gets} 1$, $\theta(s,a)\stackrel{+}{\gets} r$, $\phi(s, a)\stackrel{+}{\gets} r^2$, $N(s,a,s') \stackrel{+}{\gets} 1$.}
		\IF {$N(s,a)\in \mathcal{L}$}
		\STATE{ Set $\wh{r}(s,a)\gets \theta(s,a) / N(s,a)$;}
        \STATE{ Set $\wh{\VarR}(s,a)\gets \phi(s,a) / N(s,a) - \wh{r}(s, a)^2$;}
		\STATE Set $\wh{P}_{s,a,\wt{s}} \gets  N(s,a,\wt{s}) /N(s,a)$ for all $\wt{s} \in \cS$;
		\STATE{ Set $n(s,a)\gets N(s,a)$;}
		\STATE{ Set TRIGGERED = TRUE.}
		\ENDIF
		\ENDFOR
		\IF {TRIGGERED}
		\FOR{$h=H,H-1,...,1$ and $(s,a)\in \cS \times \cA $}
		\STATE 	 {
			Set 
                \vspace{-0.5cm}
			\begin{align*} 
			&b_h(s,a)\gets 4 \sqrt{\frac{   \mathbb{ V}(\wh{P}_{s,a} ,V_{h+1}) \iota  }{ \max\{n(s,a),1 \} }} \\
                &\quad +2 \sqrt{ \frac{ \wh{\VarR}(s,a) \iota }{\max\{n(s,a),1 \} } } +\frac{21 \iota}{ \max\{n(s,a) ,1\}  }; \\
			& Q_h(s,a)\gets \min\{  1, \\
                &\quad \wh{r}(s,a)+\wh{P}_{s,a} V_{h+1} +b_h(s,a)\}; \\
			& V_{h}(s) \gets \max_{a}Q_h(s,a).
			\end{align*}
                \vspace{-0.7cm}
		}
		\ENDFOR
		\STATE{ Set TRIGGERED = FALSE.}
		\ENDIF
		\ENDFOR
	\end{algorithmic}
\end{algorithm}
\else
\begin{algorithm}[t!]
	\caption{\mvpalt \label{alg:mvp_alt}}
	\begin{algorithmic}[1]
		\STATE{ \textbf{Input and initialize:} Logarithm term $\iota$; Trigger set $\mathcal{L} \gets \{ 2^{i-1}\ |\ 2^{i}\leq KH, i=1,2,\ldots \}$.}
		\FOR{$(s,a,s',h)\in \cS \times \cA \times\cS \times[H]$}
		\STATE{ $N(s,a)\gets 0$, $\theta(s,a)\gets 0$, $\phi(s, a)\gets 0$, $n(s,a)\gets 0$; }
		\STATE{ $N(s,a,s')\gets 0$, $\wh{P}_{s,a,s'}\gets 0$, $Q_h(s,a)\gets 1$, $V_h(s)\gets 1$.}
		\ENDFOR
        \STATE{ \verb|\\| \emph{Main algorithm begins}}
		\FOR {$k=1,2,\ldots, K$}
		\STATE{ Observe $s_1^k$.}
        \FOR {$h=1,2,\ldots, H$}
		\STATE{ Take action $ a_h^k= \arg\max_{a}Q_h(s_h^k,a)$;}
		\STATE{ Receive reward $r_h^k$ and observe $s_{h+1}^k$;}
		\STATE{ Set $(s,a,s',r)\gets (s_h^k,a_h^k,s_{h+1}^k,r_h^k)$;}
		\STATE{ Set $N(s,a) \gets  N( s,a )+1$, $\theta(s,a)\gets \theta(s,a)+r$, $\phi(s, a)\gets \phi(s, a) + r^2$, $N(s,a,s') \gets N(s,a,s')+1$.}
		\STATE{ \verb|\\| \emph{Update empirical reward and transition probability}}
		\IF {$N(s,a)\in \mathcal{L}$}
		\STATE{ Set $\wh{r}(s,a)\gets \theta(s,a) / N(s,a)$;}
        \STATE{ Set $\wh{\VarR}(s,a)\gets \phi(s,a) / N(s,a) - \wh{r}(s, a)^2$;}
		\STATE Set $\wh{P}_{s,a,\wt{s}} \gets  N(s,a,\wt{s}) /N(s,a)$ for all $\wt{s} \in \cS$;
		\STATE{ Set $n(s,a)\gets N(s,a)$;}
		\STATE{ Set TRIGGERED = TRUE.}
		\ENDIF
		\ENDFOR
		\STATE{ \verb|\\| \emph{Update $Q$-function}}
		\IF {TRIGGERED}
		\FOR{$h=H,H-1,...,1$}
		\FOR{$(s,a)\in \cS \times \cA $}
		\STATE 	 {
			Set
			\begin{align*} 
			&b_h(s,a)\gets 4 \sqrt{\frac{   \mathbb{ V}(\wh{P}_{s,a} ,V_{h+1}) \iota  }{ \max\{n(s,a),1 \} }}+ 2 \sqrt{ \frac{ \wh{\VarR}(s,a) \iota }{\max\{n(s,a),1 \} } }+\frac{21 \iota}{ \max\{n(s,a) ,1\}  }; \\
			& Q_h(s,a)\gets \min\{    \wh{r}(s,a)+\wh{P}_{s,a} V_{h+1} +b_h(s,a)    ,1\}; \\
			& V_{h}(s) \gets \max_{a}Q_h(s,a).
			\end{align*}
		}
		\ENDFOR
		\ENDFOR
		\STATE{ Set TRIGGERED = FALSE.}
		\ENDIF
		\ENDFOR
	\end{algorithmic}
\end{algorithm}
\fi

\paragraph{Algorithm description.}
\mvpalt~re-plans whenever a state-action pair's visitation is doubled.
The bonus function depends on the variance of the next-step value functions.
It achieves variance-dependent regret by using the empirical variances of rewards in the bonus, as opposed to using the empirical rewards themselves in \mvp.
\mvpalt~is capable of handling \Cref{asp:bounded_reward,asp:homogeneous,asp:determ}.

\begin{restatable}{theorem}{thmMainMb}\label{thm:main_mb}
Assume that \Cref{asp:bounded_reward,asp:homogeneous} hold.
Let $\delta \in (0, 1)$ be the confidence parameter and that $H, S, A, K, \delta$ be known.
With probability at least $1 - \delta$, the regret of \mvpalt~(\Cref{alg:mvp_alt}) run with
\begin{align*}
    \iota = 99 \left( \ln \left( \frac{3000^2 H^5 S^7 A^5 K^5}{\delta^2} \right) + 1 \right)
\end{align*}
is bounded by
\begin{align*}
    \Regret (K) \le \wt{O} (\sqrt{\min\{\multistepvar, \maxvar K \} S A} + \Gamma S A).
\end{align*}
\end{restatable}

When \Cref{asp:bounded_reward} holds, we have $\maxvar \le 1$.
Thus, our result is better than the $\wt{O} (\sqrt{S A K} + S^2 A)$ regret of \mvp, and achieves the \emph{horizon-free} (only logarithm dependency on $H$) property.
It is also strictly better than the $\wt{O} (\sqrt{H \bbQ^\star \cdot S A K} + H^{5/2} S^2 A))$ regret in \citet{paper:euler}.

\paragraph{Proof sketch.}
See \Cref{sec:mb_proof} for the rigorous proof.
We follow the outline in \citet{paper:mvp}, realizing that the total regret is upper-bounded by $M_1 + M_2 + M_3$, where
\begin{align*}
    &M_1 \approx \sumkh (P_{\sahk} - \boldone_{s_{h + 1}^k}) V_{h + 1}^k, \\
    &M_2 \approx \sumkh (V_h^k (s_h^k) - r (\sahk) - P_{\sahk} V_{h + 1}^k), \\
    &M_3 \approx \sum_{k = 1}^K \left( \sum_{h = 1}^H r (\sahk) - V_1^{\pi^k} (s_1^k) \right).
\end{align*}
We expand $r (\sahk)$ by Bellman equation to derive a tighter bound for $M_3$. 
This change is necessary to remove a variance-independent $\wt{O} (\sqrt{K})$ term.
$M_1, M_2, M_3$ can be then related to a crucial variance term
\begin{align*}
    M_4 \approx \sumkh (\bbV (R (\sahk)) + \bbV (P_{\sahk}, V_{h + 1}^k))
\end{align*}
so the regret is approximately $\wt{O} (\sqrt{S A M_4})$.
The difference between $\multistepvar$ and $M_4$ is of a lower order.
To upper bound $M_4$ directly, we introduce \emph{bonus-correction} terms
\begin{align*}
    \bc_h^k (s, a) := V_h^k (s) - P_{s, a} V_{h + 1}^k - r (s, a).
\end{align*}
Let $\BC_h^k (s) := \bc_h^k (s, a) + P_{s, a} \BC_{h + 1}^k$ with $a = \pi_h^k (s)$, then it can be proven that $\BC_h^k (s) = V_h^k (s) - V_h^{\pi^k} (s)$.
Thus, $M_4$ can be bounded by the sum of
\begin{align*}
    Z \approx \sumkh \bbV (P_{\sahk}, \BC_{h + 1}^k)
\end{align*}
and
\begin{align*}
    W = \sumkh (\bbV (R (\sahk)) + \bbV (P_{\sahk}, V_{h + 1}^{\pi^k})),
\end{align*}
where $Z$ is of a lower order and $W \le \wt{O} (\maxvar K)$.
However, the bound of $W$ cannot be derived using martingale concentration inequalities directly, because the summation of variances within an episode can be of order $\Omega (H)$, which will introduce a constant term of $H$, ruining the \emph{horizon-free} property.
We first prove that the total variance in an episode is bounded by $\wt{O} (1)$ with high probability, then the martingale concentration inequality can be applied to terms \emph{truncated} to $\wt{O} (1)$.
To get the $\Gamma$-dependency in the lower order term, we observe that $P_{s, a} = 0 \implies \wh{P}_{s, a} = 0$ and put this into concentration bounds.

\paragraph{Corollaries.}
We study deterministic MDPs first.

\begin{restatable}{corollary}{corMbDeterm}\label{cor:mb_determ}
Assume that \Cref{asp:bounded_reward,asp:homogeneous,asp:determ} hold.
Let $\delta \in (0, 1)$ be the confidence parameter and that $H, S, A, K, \delta$ be known.
With probability at least $1 - \delta$, the regret of \mvpalt~(\Cref{alg:mvp_alt}) run with
$\iota = 99 ( \ln (3000^2 H^5 S^7 A^5 K^5 / \delta^2) + 1 )$
is bounded by
$\Regret (K) \le \wt{O} (S A)$.
\end{restatable}

This is because $\maxvar = 0$ and $\Gamma = 1$ when the MDP is deterministic.
With a more refined analysis, we can totally eliminate the dependency on $\delta$.
Up to logarithm factors, \mvpalt~matches the lower bound of $\Omega (S A)$.
So \mvpalt~is \emph{minimax optimal for the class of deterministic MDPs}.

Another corollary arises when we remove \Cref{asp:bounded_reward,asp:homogeneous}.
For \mvpalt~to work properly, we need to apply the conversion methods written below the conditions.

\begin{restatable}{corollary}{corMbInhomogeneous}\label{cor:mb_inhomogenous}
Let $\delta \in (0, 1)$ be the confidence parameter and that $H, S, A, K, \delta$ be known.
With probability at least $1 - \delta$, the regret of \mvpalt~(\Cref{alg:mvp_alt}) run with
$\iota = 99 ( \ln ( 3000^2 H^{12} S^7 A^5 K^5 / \delta^2 ) + 1 )$
is bounded by
\begin{align*}
    \Regret (K) \le \wt{O} (\sqrt{\min\{ \multistepvar, \maxvar K \} H S A} + H^2 \Gamma S A).
\end{align*}\
\end{restatable}

Readers may notice that the scaling in main order term is not standard.
This is because when removing \Cref{asp:bounded_reward}, $\multistepvar$ and $\maxvar$ automatically scale by $H^2$.

\section{Results of the Model-Free Algorithm} \label{sec:mf_results}

We propose \ucbadvalt~(\Cref{alg:ucbadv_alt}) to initiate the study of model-free algorithms with variance-dependent regrets.

\ifnum \icml=1
\begin{algorithm}[t!]
	\caption{\ucbadvalt} \label{alg:ucbadv_alt}
	\begin{algorithmic}[1] \small
		\STATE{\textbf{Input and initialize:} Logarithm term $\iota$; Stage lengths $e_1 = H$, $e_{i + 1} = \floor{(1 + 1 / H) e_i}$ and stage trigger set $\cL \gets \{ \sum_{i = 1}^j e_i \ |\ j = 1, 2, \ldots\}$; Reference trigger set $\cR \gets \{ 60000 \cdot 2^{2 i} S A H^3 \iota \ |\ i=1,2,\ldots,\is \}$.}
        \STATE{ Set all $N_h (s, a)$, $\wc{N}_h(s,a)$, $\theta_h (s, a)$, $\phi_h (s, a)$, $\wc{\upsilon}_h (s, a)$, $\wc{\mu}_h (s, a)$, $\wc{\sigma}_h (s, a)$, $\mu_h^\tref (s, a)$, $\sigma_h^\tref (s, a)$ to be $0$ and all $V_{h}(s)$, $Q_{h}(s,a)$, $V^{\tref}_{h}(s,a)$ to be $H$.}
        
        
	\FOR{$k = 1,2,\ldots,K$}
            \STATE{ Observe $s_1^k$.}
            \FOR{$h = 1,2,\dots,H$}
    		\STATE{ Take action $ a_h^k= \arg\max_{a}Q_h(s_h^k,a)$;}
    		\STATE{ Receive reward $r_h^k$ and observe $s_{h+1}^k$;}
      
                \STATE{
                Update accumulators:
                \vspace{-0.4cm}
                \begin{align*}
                    & n := N_{h}(\sahk)\stackrel{+}{\gets}1,\ \wc{n} := \wc{N}_{h}(\sahk)\stackrel{+}{\gets}1; \\
                    & \theta := \theta_{h}(\sahk)\stackrel{+}{\gets} r_h^k,\ \phi := \phi_{h}(\sahk)\stackrel{+}{\gets} (r_h^k)^2; \\
                    & \wc{\upsilon} := \wc{\upsilon}_h (\sahk) \stackrel{+}{\gets} V_{h + 1} (s_{h + 1}^k); \\
                    & \wc{\mu} := \wc{\mu}_h (\sahk) \stackrel{+}{\gets} V_{h + 1} (s_{h + 1}^k) - V_{h + 1}^\tref (s_{h + 1}^k);\\
                    & \wc{\sigma} := \wc{\sigma}_h (\sahk) \stackrel{+}{\gets} (V_{h + 1} (s_{h + 1}^k) - V_{h + 1}^\tref (s_{h + 1}^k))^2; \\
                    & \mu^\tref := \mu_h^\tref (\sahk) \stackrel{+}{\gets} V_{h + 1}^\tref (s_{h + 1}^k);\\
                    & \sigma^\tref := \sigma_h^\tref (\sahk) \stackrel{+}{\gets} (V_{h + 1}^\tref (s_{h + 1}^k))^2.
                \end{align*}
                \vspace{-0.7cm}
                }
                
                \IF{$n \in \mathcal{L}$}
                    \STATE{
                    Set
                    \vspace{-0.5cm}
                    \begin{align*}
                        &\wh{r} \gets \frac{\theta}{n},\ \wh{\VarR} \gets \frac{\phi}{n} - \left( \frac{\theta}{n} \right)^2; \\
                        &\bar{b} \gets 2 \sqrt{\frac{H^2 \iota}{\wc{n}}}; \\
                        &\nref \gets \frac{\sref}{n} - \left( \frac{\mref}{n}\right)^2,\ \nc = \frac{\wc{\sigma}}{\wc{n}} - \left( \frac{\mc}{\wc{n}}\right)^2; \\
                        &b \gets 4 \sqrt{\frac{ \nref \iota }{n}} + 4 \sqrt{\frac{ \nc \iota }{\wc{n}}} + 2 \sqrt{\frac{\wh{\VarR} \iota}{n}} + \frac{90 H \iota}{\wc{n}}; \\
                        &Q_h (\sahk) \gets \min \left\{\rule{0cm}{0.5cm}\right. \wh{r} + \frac{\wc{\upsilon}}{\wc{n}} + \bar{b},\\
                        &\quad \wh{r} + \frac{\mref}{n} + \frac{\mc}{\wc{n}} + b ,\ Q_h (\sahk) \left.\rule{0cm}{0.5cm}\right\}; \\
                        &V_h (s_h^k) \gets \max_{a} Q_h (s_h^k, a).
                    \end{align*}
                    }
                    \vspace{-0.7cm}
                    \STATE{Set $\wc{N}_{h}(\sahk) \gets 0$, $\wc{\mu}_{h}(\sahk) \gets 0$, $\wc{\upsilon}_{h}(\sahk) \gets 0$, $\wc{\sigma}_{h}(\sahk) \gets 0$.}
                \ENDIF

                \STATE{ \textbf{if} $\sum_a N_{h}(s_h^k,a) \in \cR$ \textbf{then} $V^{\tref}_{h}(s_h^k) \gets V_{h}(s_{h}^k)$. } \label{line:ref_upd}
            \ENDFOR
	\ENDFOR
	\end{algorithmic}
\end{algorithm}

\else
\begin{algorithm}[p]
	\caption{\ucbadvalt} \label{alg:ucbadv_alt}
	\begin{algorithmic}[1]
		\STATE{\textbf{Input and initialize:} Logarithm term $\iota$; Stage lengths $e_1 = H$, $e_{i + 1} = \floor{(1 + 1 / H) e_i}$ and stage trigger set $\cL \gets \{ \sum_{i = 1}^j e_i \ |\ j = 1, 2, \ldots\}$; Reference trigger set $\cR \gets \{ 60000 \cdot 2^{2 i} S A H^3 \iota \ |\ i=1,2,\ldots,\is \}$.}
        \FOR{$(s,a,h)\in \cS \times \cA \times[H]$}
            \STATE{ $N_h(s,a)\gets 0$, $\wc{N}_h(s,a)\gets 0$;}
            \STATE{ $\theta_h (s, a) \gets 0$, $\phi_h (s, a) \gets 0$;}
		  \STATE{ $V_{h}(s)\gets H-h+1$, $Q_{h}(s,a)\gets H-h+1$, $V^{\tref}_{h}(s,a) \gets H$;}
            \STATE{ $\wc{\upsilon}_h (s, a) \gets 0$;}
            \STATE{ $\wc{\mu}_h (s, a) \gets 0$, $\wc{\sigma}_h (s, a) \gets 0$;}
            \STATE{ $\mu_h^\tref (s, a) \gets 0$, $\sigma_h^\tref (s, a) \gets 0$.}
		\ENDFOR
        
        \STATE{ \verb|\\| \emph{Main algorithm begins}}
        
	\FOR{$k = 1,2,\ldots,K$}
            \STATE{ Observe $s_1^k$.}
            \FOR{$h = 1,2,\dots,H$}
    		\STATE{ Take action $ a_h^k= \arg\max_{a}Q_h(s_h^k,a)$;}
    		\STATE{ Receive reward $r_h^k$ and observe $s_{h+1}^k$;}
      
                \STATE{
                Update accumulators:
                \begin{align*}
                    & n := N_{h}(\sahk)\stackrel{+}{\gets}1,\ \wc{n} := \wc{N}_{h}(\sahk)\stackrel{+}{\gets}1; \\
                    & \theta := \theta_{h}(\sahk)\stackrel{+}{\gets} r_h^k,\ \phi := \phi_{h}(\sahk)\stackrel{+}{\gets} (r_h^k)^2; \\
                    & \wc{\upsilon} := \wc{\upsilon}_h (\sahk) \stackrel{+}{\gets} V_{h + 1} (s_{h + 1}^k); \\
                    & \wc{\mu} := \wc{\mu}_h (\sahk) \stackrel{+}{\gets} V_{h + 1} (s_{h + 1}^k) - V_{h + 1}^\tref (s_{h + 1}^k),\ \wc{\sigma} := \wc{\sigma}_h (\sahk) \stackrel{+}{\gets} (V_{h + 1} (s_{h + 1}^k) - V_{h + 1}^\tref (s_{h + 1}^k))^2; \\
                    & \mu^\tref := \mu_h^\tref (\sahk) \stackrel{+}{\gets} V_{h + 1}^\tref (s_{h + 1}^k),\ \sigma^\tref := \sigma_h^\tref (\sahk) \stackrel{+}{\gets} (V_{h + 1}^\tref (s_{h + 1}^k))^2.
                \end{align*}
                }
                
                \STATE{ \verb|\\| \emph{Reaching the end of a stage, update $Q$-function}}
                \IF{$n \in \mathcal{L}$}
                    \STATE{
                    Set
                    \begin{align*}
                        &\wh{r}_h (\sahk )\gets \frac{\theta}{n},\ \wh{\VarR}(\sahk) \gets \frac{\phi}{n} - \left( \frac{\theta}{n} \right)^2; \\
                        &\bar{b} \gets 2 \sqrt{\frac{H^2 \iota}{\wc{n}}}; \\
                        &\nref \gets \frac{\sref}{n} - \left( \frac{\mref}{n}\right)^2,\ \nc = \frac{\wc{\sigma}}{\wc{n}} - \left( \frac{\mc}{\wc{n}}\right)^2; \\
                        &b \gets 4 \sqrt{\frac{ \nref \iota }{n}} + 4 \sqrt{\frac{ \nc \iota }{\wc{n}}} + 2 \sqrt{\frac{\wh{\VarR}_h \iota}{n}} + \frac{90 H \iota}{\wc{n}}; \\
                        &Q_h (\sahk) \gets \min \left\{ \wh{r}_h (\sahk) + \frac{\wc{\upsilon}}{\wc{n}} + \bar{b},\ \wh{r}_h (\sahk) + \frac{\mref}{n} + \frac{\mc}{\wc{n}} + b ,\ Q_h (\sahk) \right\}; \\
                        &V_h (s_h^k) \gets \max_{a} Q_h (s_h^k, a).
                    \end{align*}
                    }
                    \STATE{ \verb|\\| \emph{Reset intra-stage accumulators}}
                    \STATE{Set $\wc{N}_{h}(\sahk) \gets 0$, $\wc{\mu}_{h}(\sahk) \gets 0$, $\wc{\upsilon}_{h}(\sahk) \gets 0$, $\wc{\sigma}_{h}(\sahk) \gets 0$.}
                \ENDIF

                \STATE{ \verb|\\| \emph{Update reference value function}}
                \STATE{ \textbf{if} $\sum_{a \in \cA} N_{h}(s_h^k,a) \in \cR$ \textbf{then} $V^{\tref}_{h}(s_h^k) \gets V_{h}(s_{h}^k)$. } \label{line:ref_upd}
            \ENDFOR
	\ENDFOR
	\end{algorithmic}
\end{algorithm}

\fi

\paragraph{Algorithm description.}
In \ucbadvalt, the update of value functions is triggered by the beginning of stages for each $(s, a, h)$ triple separately, and the stage design approximately makes use of the last $1 / H$ fraction of data.
Besides, the algorithm maintains \emph{reference values} by assigning value functions to them at another frequency.
The update rule using \emph{the reference value decomposition} can be illustrated as:
{\ifnum \icml=1 \fontsize{7.9}{9.5} \fi
\begin{align*}
    Q_h (s, a)
    \gets \wh{P_{s, a, h} V_{h + 1}^{\tref}} + \wh{P_{s, a, h} (V_{h + 1} - V_{h + 1}^{\tref})} + \wh{r}_h (s, a) + b_h^k (s, a),
\end{align*}}%
where $b_h^k (s, a)$ is the bonus, $\wh{r}_h (s,a)$, $\wh{P_{s, a, h} V_{h + 1}^{\tref}}$ and $\wh{P_{s, a, h} (V_{h + 1} - V_{h + 1}^{\tref})}$ are empirical estimates of $r_h (s, a)$, $P_{s, a, h} V_{h + 1}^{\tref}$ and $P_{s, a, h} (V_{h + 1} - V_{h + 1}^{\tref})$ respectively.
In addition, a very simple update rule
\begin{align*}
    Q_h (s, a)
    \gets \wh{P_{s, a, h} V_{h + 1}} + \wh{r}_h (s, a) + b_h^k (s, a)
\end{align*}
is also in use to provide a guarantee of uniform convergence of estimated value functions.

We make three major alterations to \ucbadv:
\ding{172} 
We use empirical variances of rewards in bonuses.
\ding{173} Due to a more refined analysis, we remove the $\wt{O} (H (n^{-3/4} + \wc{n}^{-3/4}))$ term in bonuses.
\ding{174} The reference value functions are updated in \emph{a capped-doubling manner} (cf. Line\,\ref{line:ref_upd}). 

Alteration \ding{174} is crucial to make the main order term variance-dependent, because there exist constant gaps between reference values and optimal values, whose summation contributes to the regret as the main order term in \ucbadv.
By choosing an appropriate number of updates, we can \emph{balance between the total constant gap and the deviation introduced by frequent updates}, making the total variance the only factor in the main order term.

\begin{restatable}{theorem}{thmMainMf} \label{thm:main_mf}
Let $\delta \in (0, 1)$ be the confidence parameter and that $H, S, A, K, \delta$ be known.
With probability at least $1 - \delta$, the regret of \ucbadvalt~(\Cref{alg:ucbadv_alt}) run with
\begin{align*}
    \iota = 99 \left( \ln \left( \frac{7000^2 (H S A K)^5}{\delta^2} \right) + 1 \right)
\end{align*}
and
\begin{align*}
    \is = \ceil{\frac{1}{2} \log_2 \left(\frac{K}{H^5 S^3 A \iota^2}\right)}
\end{align*}
is bounded by
{\ifnum \icml=1 \fontsize{8.5}{9.5} \fi
\begin{align*}
    \Regret (K)
    \le \wt{O} (\sqrt{\min\{ \multistepvar, \maxvar K\} H S A} + \sqrt[4]{H^{15} S^5 A^3 K}).
\end{align*}}%
\end{restatable}

We have $\maxvar \le H^2$, so our result is strictly better than the $\wt{O} (\sqrt{H^3 S A K} + \sqrt[4]{H^{33} S^8 A^6 K})$ regret of \ucbadv.

\paragraph{Extra notations.}
Let $V_h^{\tref, k}$ denote $V_h^{\tref}$ at the beginning of the $k$-th episode, and $V_h^{\REF} := V_h^{\tref, K + 1}$ denote the final reference value function.
Let $\nu_h^{\tref, k}, \wc{\nu}_h^k, b_h^k$ denote $\nu^{\tref}, \wc{\nu}, b$ for the value of $Q_h^k (\sahk)$.
Let $N_h^k (s)$ denote $\sum_a N_h (s, a)$ at the beginning of the $k$-th episode.
Let $n_h^k$ and $\wc{n}_h^k$ be the total number of visits to $(\sahk, h)$ prior to the current stage and during the stage immediately before the current stage with respect to the same triple.

\paragraph{Proof sketch.}
See \Cref{sec:mf_proof} for the rigorous proof.
From \citet{paper:ucbadv},
the regret is roughly $\sumkh (\psi_{h + 1}^k + \xi_{h + 1}^k + \phi_{h + 1}^k + b_h^k)$, where
\begin{align*}
    &\psi_{h + 1}^k \approx V_{h + 1}^{\tref, k} (s_{h + 1}^k) - V_{h + 1}^{\REF} (s_{h + 1}^k), \\
    &\xi_{h + 1}^k \approx (P_{\sahk, h} - \boldone_{s_{h + 1}^k}) (V_{h + 1}^k - V_{h + 1}^\star), \\
    &\phi_{h + 1}^k = (P_{\sahk, h} - \boldone_{s_{h + 1}^k}) (V_{h + 1}^\star - V_{h + 1}^{\pi^k}).
\end{align*}
All these four terms are bounded loosely in \citet{paper:ucbadv} such that they are all main order terms.
To establish a variance-dependent regret, we prove that only the $b$ term is the main order term \emph{after our aforementioned alterations}.
The $\phi$ term is a martingale and shown to be $\wt{O} (H)$.
For the rest terms, we need the following argument: 
{\ifnum \icml=1 \fontsize{8.5}{9.5} \fi
\begin{align*}
    N_h^k (s) \ge N_0 (\epsilon) = \wt{O} \left( \frac{H^5 S A }{\epsilon^2} \right) &\implies 0 \le V_h^k (s) - V_h^\star (s) \le \epsilon.
\end{align*}}%
Notice that the reference trigger set $\cR$ in \Cref{alg:ucbadv_alt} is composed of $N_0 (\beta_i)$ for $i \in [\is]$ where $\beta_i := H / 2^i$.
There is a constant gap of at least $\beta_\is$ between $V_h^{\REF} (s)$ and $V_h^\star (s)$ in the worst case, because the number of updates is capped by $\is$.
This argument branches into two corollaries.
The first one is we can bound value gaps directly:
\begin{align*}
    \sumkh (V_h^k (s_h^k) - V_h^\star (s_h^k))^2 \le \wt{O} ( H^6 S A ).
\end{align*}
This can be utilized to bound the $\xi$ term. 
The second one is that, we define
\begin{align*}
    B_h^{\tref, k} (s) := \sum_{i = 1}^{\is} \beta_{i - 1} \II [ N_0 (\beta_{i - 1}) \le N_h^k (s) < N_0 (\beta_i) ],
\end{align*}
then $V_h^{\tref, k} (s) - V_h^{\REF} (s) \le B_h^{\tref, k} (s)$, $V_h^{\tref, k} (s) - V_h^\star (s) \le B_h^{\tref, k} (s) + \beta_\is$ and
{\ifnum \icml=1 \fontsize{8.5}{9.5} \fi
\begin{gather*}
    \sum_{k, h} B_h^{\tref, k} (s_h^k) \le \wt{O} (H^5 S^2 A 2^{\is}), 
    \sum_{k, h} (B_h^{\tref, k} (s_h^k))^2 \le \wt{O} (H^6 S^2 A \is).
\end{gather*}}%
So we can directly bound the $\psi$ term.
We show that
\begin{gather*}
    \nu_h^{\tref, k} \lessapprox \wt{O} ( \bbV (P_{\sahk, h}, V_{h + 1}^\star) + (B_{h + 1}^{\tref, k} (s_{h + 1}^k))^2  + \beta_\is^2), \\
    \nc_h^k \lessapprox O ( (B_{h + 1}^{\tref, k} (s_{h + 1}^k))^2  + \beta_\is^2).
\end{gather*}
    
The $b$ term is hence bounded by
\begin{align*}
    \wt{O} (\sqrt{\multistepvar H S A} + \sqrt{H^5 S A K / 2^{2 \is}}).
\end{align*}
Analogous to the proof of \Cref{thm:main_mb}, the difference between $\multistepvar$ and $\maxvar K$ is of a lower order.
Finally, the lower order terms are
\begin{align*}
    \wt{O} (\sqrt{H^5 S A K / 2^{2 \is}} + H^5 S^2 A 2^\is).
\end{align*}
We derive \Cref{thm:main_mf} by choosing the optimal $\is$.

\paragraph{Corollary.}
We study deterministic MDPs.

\begin{restatable}{corollary}{corMfDeterm} \label{cor:mf_determ}
Assume that \Cref{asp:determ} holds.
Let $\delta \in (0, 1)$ be the confidence parameter and that $H, S, A, K, \delta$ be known.
With probability at least $1 - \delta$, the regret of \ucbadvalt~(\Cref{alg:ucbadv_alt}) run with
$\iota = 99 ( \ln ( 7000^2 (H S A K)^5 / \delta^2 ) + 1 )$
and
$\is = \ceil{1/2 \cdot \log_2 (K / H^5 S^3 A \iota^2 )}$
is bounded by
\begin{align*}
    \Regret (K) \le \wt{O} (\sqrt[4]{H^{15} S^5 A^3 K}).
\end{align*}
\end{restatable}

Notice that the regret under \Cref{asp:determ} is not constant which we desire, this may be due to some fundamental limit of model-free algorithms.
However, since the research on model-free algorithms is still at its nascent stage and there lack thorough understanding, our result provides the first look into the potential of such algorithms.

Intuitively, for any algorithm to have a constant regret on deterministic MDPs, its value functions should also converge in a constant steps.
Previous model-free algorithms all use historical data to estimate value functions.
These data are biased because some of them are not up-to-date, making value functions hard to converge in a constant steps.
Here we identify difficulties for existing algorithms to be variance-dependent for all $K$-related terms.

\paragraph{\qlucbb~\citep{paper:q_learning}.}
In their proof of Lemma\,C.3, when bounding $\abs{P_3 - P_4}$, there is a variance-independent $1 / \sqrt{t}$ term in the gap between the estimations and true values.
Notice that their result is possible to be variance-dependent by not loosening $\sqrt{H^7 S A \iota / t} \le H + H^6 S A \iota / t$ above their Equation\,(C.10) while introducing a variance-independent $K^{1/4}$ term.

\paragraph{\ucbadv~\citep{paper:ucbadv}.}
There are biases in the reference value functions, because they are updated for only finite times.
If the update is not capped by a threshold, readers can easily verify that the $\psi$ term will become a variance-independent main order term.

\paragraph{\qesadv~\citep{paper:model_free_breaking}.}
There is a same issue about the constant gap between the reference value and the optimal value when bounding $\cR_3$ defined in their Equation\,(39c).

\section{Regret Lower Bounds} \label{sec:lb_results}

We show that for any algorithm and any variance $\cV$, there always exists an MDP such that the regret main order terms of \Cref{thm:main_mb}, \Cref{cor:mb_inhomogenous} and \Cref{thm:main_mf} are tight.
This means that \mvpalt~and \ucbadvalt~are \emph{minimax optimal for the class of variance-bounded MDPs}.
The proofs for this section are deferred to \Cref{sec:lb_proof}.

\begin{restatable}{theorem}{thmLowerBoundMvp} \label{thm:lower_bound_mvp}
Assume $S \ge 6$, $A \ge 2$, $H \ge 3 \floor{\log_2 (S - 2)}$ and $0 < \cV \le O (1)$.
For any algorithm $\boldpi$, there exists an MDP $\cM_{\boldpi}$ such that:

\begin{itemize}

\item It satisfies \Cref{asp:bounded_reward,asp:homogeneous};

\item $\Var_\tau^\Sigma, \maxvar = \Theta (\cV)$ for any possible trajectory $\tau$;

\item For $K \ge S A$, the expected regret of $\boldpi$ in $\cM_{\boldpi}$ after $K$ episodes satisfies
\end{itemize}
\vskip-0.5cm
\begin{align*}
    \bbE \left[ \left. \sum_{k = 1}^K (V_1^\star (s_1^k) - V_1^{\pi^k} (s_1^k)) \ \right|\ \cM_{\boldpi}, \boldpi \right] = \Omega ( \sqrt{\cV S A K} ).
\end{align*}
\end{restatable}

\begin{restatable}{theorem}{thmLowerBound} \label{thm:lower_bound}
Assume $S \ge 6$, $A \ge 2$, $H \ge 3 \floor{\log_2 (S - 2)}$ and $0 < \cV \le O (H^2)$.
For any algorithm $\boldpi$, there exists an MDP $\cM_{\boldpi}$ such that:

\begin{itemize}

\item $\Var_\tau^\Sigma, \maxvar = \Theta (\cV)$ for any possible trajectory $\tau$;

\item For $K \ge H S A$, the expected regret of $\boldpi$ in $\cM_{\boldpi}$ after $K$ episodes satisfies
\end{itemize}
\vskip-0.5cm
\begin{align*}
    \bbE \left[ \left. \sum_{k = 1}^K (V_1^\star (s_1^k) - V_1^{\pi^k} (s_1^k)) \ \right|\ \cM_{\boldpi}, \boldpi \right] = \Omega ( \sqrt{\cV H S A K} ).
\end{align*}

\end{restatable}


\section{Conclusion}

We systematically study variance-dependent regret bounds for MDPs by introducing new notions of variances, proposing model-based and model-free algorithms respectively, and providing regret lower bounds for the class of variance-bounded MDPs.
Our results improve upon the previous algorithms and achieves minimax optimal regrets for the class of variance-bounded MDPs.
Our model-based algorithm is minimax optimal for deterministic MDPs.
Finally, we identify some possible limit of current model-free algorithms.
One possible future direction is to find a new model-free algorithm with a constant regret for deterministic MDPs.

\section*{Acknowledgements}
SSD acknowledges the support of NSF IIS 2110170, NSF DMS 2134106, NSF CCF 2212261, NSF IIS 2143493, NSF CCF 2019844, NSF IIS 2229881.

}


\bibliography{ref}

\begin{thebibliography}{49}
\providecommand{\natexlab}[1]{#1}
\providecommand{\url}[1]{\texttt{#1}}
\expandafter\ifx\csname urlstyle\endcsname\relax
  \providecommand{\doi}[1]{doi: #1}\else
  \providecommand{\doi}{doi: \begingroup \urlstyle{rm}\Url}\fi

\bibitem[Agrawal \& Jia(2017)Agrawal and
  Jia]{paper:optimistic_posterior_sampling}
Agrawal, S. and Jia, R.
\newblock Optimistic posterior sampling for reinforcement learning: worst-case
  regret bounds.
\newblock \emph{Advances in Neural Information Processing Systems}, 30, 2017.

\bibitem[Auer et~al.(2008)Auer, Jaksch, and Ortner]{paper:ucrl2}
Auer, P., Jaksch, T., and Ortner, R.
\newblock Near-optimal regret bounds for reinforcement learning.
\newblock \emph{Advances in neural information processing systems}, 21, 2008.

\bibitem[Azar et~al.(2017)Azar, Osband, and Munos]{paper:ucbvi}
Azar, M.~G., Osband, I., and Munos, R.
\newblock Minimax regret bounds for reinforcement learning.
\newblock In \emph{ICML}, 2017.

\bibitem[Bai et~al.(2019)Bai, Xie, Jiang, and
  Wang]{paper:q_learning_low_switch_cost}
Bai, Y., Xie, T., Jiang, N., and Wang, Y.-X.
\newblock Provably efficient q-learning with low switching cost.
\newblock \emph{Advances in Neural Information Processing Systems}, 32, 2019.

\bibitem[Bartlett \& Tewari(2012)Bartlett and Tewari]{paper:regal}
Bartlett, P.~L. and Tewari, A.
\newblock Regal: A regularization based algorithm for reinforcement learning in
  weakly communicating mdps.
\newblock \emph{arXiv preprint arXiv:1205.2661}, 2012.

\bibitem[Bertsekas(2012)]{book:dp_opt_control}
Bertsekas, D.
\newblock \emph{Dynamic programming and optimal control: Volume I}, volume~1.
\newblock Athena scientific, 2012.

\bibitem[Brockman et~al.(2016)Brockman, Cheung, Pettersson, Schneider,
  Schulman, Tang, and Zaremba]{brockman2016openai}
Brockman, G., Cheung, V., Pettersson, L., Schneider, J., Schulman, J., Tang,
  J., and Zaremba, W.
\newblock Openai gym.
\newblock \emph{arXiv preprint arXiv:1606.01540}, 2016.

\bibitem[Chen et~al.(2021)Chen, Jafarnia-Jahromi, Jain, and Luo]{paper:lcb_ssp}
Chen, L., Jafarnia-Jahromi, M., Jain, R., and Luo, H.
\newblock Implicit finite-horizon approximation and efficient optimal
  algorithms for stochastic shortest path.
\newblock In \emph{NeurIPS}, 2021.

\bibitem[Dann et~al.(2017)Dann, Lattimore, and Brunskill]{paper:ubev}
Dann, C., Lattimore, T., and Brunskill, E.
\newblock Unifying pac and regret: Uniform pac bounds for episodic
  reinforcement learning.
\newblock In \emph{NIPS}, 2017.

\bibitem[Dann et~al.(2019)Dann, Li, Wei, and Brunskill]{paper:orlc}
Dann, C., Li, L., Wei, W., and Brunskill, E.
\newblock Policy certificates: Towards accountable reinforcement learning.
\newblock In \emph{International Conference on Machine Learning}, pp.\
  1507--1516. PMLR, 2019.

\bibitem[Domingues et~al.(2021)Domingues, M{\'e}nard, Kaufmann, and
  Valko]{paper:episodic_lower_bound}
Domingues, O.~D., M{\'e}nard, P., Kaufmann, E., and Valko, M.
\newblock Episodic reinforcement learning in finite mdps: Minimax lower bounds
  revisited.
\newblock In Feldman, V., Ligett, K., and Sabato, S. (eds.), \emph{Proceedings
  of the 32nd International Conference on Algorithmic Learning Theory}, volume
  132 of \emph{Proceedings of Machine Learning Research}, pp.\  578--598. PMLR,
  16--19 Mar 2021.
\newblock URL \url{https://proceedings.mlr.press/v132/domingues21a.html}.

\bibitem[Even-Dar et~al.(2006)Even-Dar, Mannor, Mansour, and
  Mahadevan]{paper:action_elimination}
Even-Dar, E., Mannor, S., Mansour, Y., and Mahadevan, S.
\newblock Action elimination and stopping conditions for the multi-armed bandit
  and reinforcement learning problems.
\newblock \emph{Journal of machine learning research}, 7\penalty0 (6), 2006.

\bibitem[Fruit et~al.(2018{\natexlab{a}})Fruit, Pirotta, and
  Lazaric]{fruit2018near}
Fruit, R., Pirotta, M., and Lazaric, A.
\newblock Near optimal exploration-exploitation in non-communicating markov
  decision processes.
\newblock \emph{Advances in Neural Information Processing Systems}, 31,
  2018{\natexlab{a}}.

\bibitem[Fruit et~al.(2018{\natexlab{b}})Fruit, Pirotta, Lazaric, and
  Ortner]{paper:scal}
Fruit, R., Pirotta, M., Lazaric, A., and Ortner, R.
\newblock Efficient bias-span-constrained exploration-exploitation in
  reinforcement learning.
\newblock In \emph{International Conference on Machine Learning}, pp.\
  1578--1586. PMLR, 2018{\natexlab{b}}.

\bibitem[Garivier et~al.(2016)Garivier, Ménard, and
  Stoltz]{paper:regret_bandit}
Garivier, A., Ménard, P., and Stoltz, G.
\newblock Explore first, exploit next: The true shape of regret in bandit
  problems.
\newblock \emph{Mathematics of Operations Research}, 44, 02 2016.
\newblock \doi{10.1287/moor.2017.0928}.

\bibitem[Jiang \& Agarwal(2018)Jiang and Agarwal]{jiang2018open}
Jiang, N. and Agarwal, A.
\newblock Open problem: The dependence of sample complexity lower bounds on
  planning horizon.
\newblock In \emph{Conference On Learning Theory}, pp.\  3395--3398. PMLR,
  2018.

\bibitem[Jin et~al.(2018)Jin, Allen-Zhu, Bubeck, and Jordan]{paper:q_learning}
Jin, C., Allen-Zhu, Z., Bubeck, S., and Jordan, M.~I.
\newblock Is q-learning provably efficient?
\newblock \emph{Advances in neural information processing systems}, 31, 2018.

\bibitem[Jin et~al.(2020)Jin, Krishnamurthy, Simchowitz, and
  Yu]{paper:reward_free_exploration}
Jin, C., Krishnamurthy, A., Simchowitz, M., and Yu, T.
\newblock Reward-free exploration for reinforcement learning.
\newblock In \emph{International Conference on Machine Learning}, pp.\
  4870--4879. PMLR, 2020.

\bibitem[Kim et~al.(2021)Kim, Yang, and Jun]{kim2021improved}
Kim, Y., Yang, I., and Jun, K.-S.
\newblock Improved regret analysis for variance-adaptive linear bandits and
  horizon-free linear mixture mdps.
\newblock \emph{arXiv preprint arXiv:2111.03289}, 2021.

\bibitem[Lattimore \& Szepesvári(2020)Lattimore and
  Szepesvári]{paper:bandit_algorithms}
Lattimore, T. and Szepesvári, C.
\newblock \emph{Bandit Algorithms}.
\newblock Cambridge University Press, 2020.
\newblock \doi{10.1017/9781108571401}.

\bibitem[Li et~al.(2021)Li, Shi, Chen, Gu, and Chi]{paper:model_free_breaking}
Li, G., Shi, L., Chen, Y., Gu, Y., and Chi, Y.
\newblock Breaking the sample complexity barrier to regret-optimal model-free
  reinforcement learning.
\newblock \emph{Advances in Neural Information Processing Systems},
  34:\penalty0 17762--17776, 2021.

\bibitem[Li \& Sun(2023)Li and Sun]{li2023variance}
Li, X. and Sun, Q.
\newblock Variance-aware robust reinforcement learning with linear function
  approximation with heavy-tailed rewards.
\newblock \emph{arXiv preprint arXiv:2303.05606}, 2023.

\bibitem[Maillard et~al.(2014)Maillard, Mann, and
  Mannor]{paper:distribution_norm}
Maillard, O.-A., Mann, T.~A., and Mannor, S.
\newblock How hard is my mdp?" the distribution-norm to the rescue".
\newblock \emph{Advances in Neural Information Processing Systems}, 27, 2014.

\bibitem[Maurer \& Pontil(2009)Maurer and Pontil]{paper:bennett}
Maurer, A. and Pontil, M.
\newblock Empirical bernstein bounds and sample-variance penalization.
\newblock In \emph{COLT}, 2009.

\bibitem[Mnih et~al.(2013)Mnih, Kavukcuoglu, Silver, Graves, Antonoglou,
  Wierstra, and Riedmiller]{mnih2013playing}
Mnih, V., Kavukcuoglu, K., Silver, D., Graves, A., Antonoglou, I., Wierstra,
  D., and Riedmiller, M.
\newblock Playing atari with deep reinforcement learning.
\newblock \emph{arXiv preprint arXiv:1312.5602}, 2013.

\bibitem[Neu \& Pike-Burke(2020)Neu and Pike-Burke]{paper:mdp_optimism}
Neu, G. and Pike-Burke, C.
\newblock A unifying view of optimism in episodic reinforcement learning.
\newblock \emph{Advances in Neural Information Processing Systems},
  33:\penalty0 1392--1403, 2020.

\bibitem[Osband \& Van~Roy(2017)Osband and Van~Roy]{osband2017posterior}
Osband, I. and Van~Roy, B.
\newblock Why is posterior sampling better than optimism for reinforcement
  learning?
\newblock In \emph{International conference on machine learning}, pp.\
  2701--2710. PMLR, 2017.

\bibitem[Osband et~al.(2013)Osband, Russo, and Van~Roy]{osband2013more}
Osband, I., Russo, D., and Van~Roy, B.
\newblock (more) efficient reinforcement learning via posterior sampling.
\newblock \emph{Advances in Neural Information Processing Systems}, 26, 2013.

\bibitem[Pacchiano et~al.(2020)Pacchiano, Ball, Parker-Holder, Choromanski, and
  Roberts]{paper:opt_mb_rl}
Pacchiano, A., Ball, P., Parker-Holder, J., Choromanski, K., and Roberts, S.
\newblock On optimism in model-based reinforcement learning.
\newblock \emph{arXiv preprint arXiv:2006.11911}, 2020.

\bibitem[Russo(2019)]{russo2019worst}
Russo, D.
\newblock Worst-case regret bounds for exploration via randomized value
  functions.
\newblock \emph{Advances in Neural Information Processing Systems}, 32, 2019.

\bibitem[Simchowitz \& Jamieson(2019)Simchowitz and
  Jamieson]{paper:non_asymptotic_gap}
Simchowitz, M. and Jamieson, K.~G.
\newblock Non-asymptotic gap-dependent regret bounds for tabular mdps.
\newblock \emph{Advances in Neural Information Processing Systems}, 32, 2019.

\bibitem[Talebi \& Maillard(2018)Talebi and
  Maillard]{paper:variance_aware_kl_ucrl}
Talebi, M.~S. and Maillard, O.-A.
\newblock Variance-aware regret bounds for undiscounted reinforcement learning
  in mdps.
\newblock In \emph{Algorithmic Learning Theory}, pp.\  770--805. PMLR, 2018.

\bibitem[Tarbouriech et~al.(2021)Tarbouriech, Zhou, Du, Pirotta, Valko, and
  Lazaric]{paper:eb_ssp}
Tarbouriech, J., Zhou, R., Du, S.~S., Pirotta, M., Valko, M., and Lazaric, A.
\newblock Stochastic shortest path: Minimax, parameter-free and towards
  horizon-free regret.
\newblock In \emph{Neural Information Processing Systems}, 2021.

\bibitem[Wagenmaker et~al.(2022)Wagenmaker, Chen, Simchowitz, Du, and
  Jamieson]{paper:first_order_lfa}
Wagenmaker, A.~J., Chen, Y., Simchowitz, M., Du, S., and Jamieson, K.
\newblock First-order regret in reinforcement learning with linear function
  approximation: A robust estimation approach.
\newblock In \emph{International Conference on Machine Learning}, pp.\
  22384--22429. PMLR, 2022.

\bibitem[Wang et~al.(2020)Wang, Du, Yang, and Kakade]{wang2020long}
Wang, R., Du, S.~S., Yang, L.~F., and Kakade, S.~M.
\newblock Is long horizon reinforcement learning more difficult than short
  horizon reinforcement learning?
\newblock \emph{arXiv preprint arXiv:2005.00527}, 2020.

\bibitem[Xiong et~al.(2021)Xiong, Shen, and Du]{paper:random_explore}
Xiong, Z., Shen, R., and Du, S.~S.
\newblock Randomized exploration is near-optimal for tabular mdp.
\newblock \emph{arXiv preprint arXiv:2102.09703}, 2021.

\bibitem[Xu et~al.(2021)Xu, Ma, and
  Du]{paper:gap_dependent_multi_step_bootstrap}
Xu, H., Ma, T., and Du, S.
\newblock Fine-grained gap-dependent bounds for tabular mdps via adaptive
  multi-step bootstrap.
\newblock In \emph{Conference on Learning Theory}, pp.\  4438--4472. PMLR,
  2021.

\bibitem[Yang et~al.(2021)Yang, Yang, and Du]{paper:q_learning_log_regret}
Yang, K., Yang, L., and Du, S.
\newblock Q-learning with logarithmic regret.
\newblock In \emph{International Conference on Artificial Intelligence and
  Statistics}, pp.\  1576--1584. PMLR, 2021.

\bibitem[Zanette \& Brunskill(2019)Zanette and Brunskill]{paper:euler}
Zanette, A. and Brunskill, E.
\newblock Tighter problem-dependent regret bounds in reinforcement learning
  without domain knowledge using value function bounds.
\newblock In \emph{ICML}, 2019.

\bibitem[Zhang \& Ji(2019)Zhang and Ji]{zhang2019regret}
Zhang, Z. and Ji, X.
\newblock Regret minimization for reinforcement learning by evaluating the
  optimal bias function.
\newblock \emph{Advances in Neural Information Processing Systems}, 32, 2019.

\bibitem[Zhang et~al.(2020)Zhang, Zhou, and Ji]{paper:ucbadv}
Zhang, Z., Zhou, Y., and Ji, X.
\newblock Almost optimal model-free reinforcement learning via
  reference-advantage decomposition.
\newblock \emph{Advances in Neural Information Processing Systems},
  33:\penalty0 15198--15207, 2020.

\bibitem[Zhang et~al.(2021{\natexlab{a}})Zhang, Ji, and Du]{paper:mvp}
Zhang, Z., Ji, X., and Du, S.~S.
\newblock Is reinforcement learning more difficult than bandits? a near-optimal
  algorithm escaping the curse of horizon.
\newblock In \emph{COLT}, 2021{\natexlab{a}}.

\bibitem[Zhang et~al.(2021{\natexlab{b}})Zhang, Yang, Ji, and
  Du]{zhang2021improved}
Zhang, Z., Yang, J., Ji, X., and Du, S.~S.
\newblock Improved variance-aware confidence sets for linear bandits and linear
  mixture mdp.
\newblock \emph{Advances in Neural Information Processing Systems},
  34:\penalty0 4342--4355, 2021{\natexlab{b}}.

\bibitem[Zhang et~al.(2021{\natexlab{c}})Zhang, Zhou, and
  Ji]{paper:model_free_rl}
Zhang, Z., Zhou, Y., and Ji, X.
\newblock Model-free reinforcement learning: from clipped pseudo-regret to
  sample complexity.
\newblock In \emph{Proceedings of the 38th International Conference on Machine
  Learning}, pp.\  12653--12662. PMLR, 2021{\natexlab{c}}.

\bibitem[Zhang et~al.(2022)Zhang, Ji, and Du]{paper:horizon_free}
Zhang, Z., Ji, X., and Du, S.~S.
\newblock Horizon-free reinforcement learning in polynomial time: the power of
  stationary policies.
\newblock In \emph{Annual Conference Computational Learning Theory}, 2022.

\bibitem[Zhao et~al.(2022)Zhao, Zhou, He, and Gu]{zhao2022bandit}
Zhao, H., Zhou, D., He, J., and Gu, Q.
\newblock Bandit learning with general function classes: Heteroscedastic noise
  and variance-dependent regret bounds.
\newblock \emph{arXiv preprint arXiv:2202.13603}, 2022.

\bibitem[Zhao et~al.(2023)Zhao, He, Zhou, Zhang, and Gu]{zhao2023variance}
Zhao, H., He, J., Zhou, D., Zhang, T., and Gu, Q.
\newblock Variance-dependent regret bounds for linear bandits and reinforcement
  learning: Adaptivity and computational efficiency.
\newblock \emph{arXiv preprint arXiv:2302.10371}, 2023.

\bibitem[Zhou et~al.(2021)Zhou, Gu, and Szepesvari]{zhou2021nearly}
Zhou, D., Gu, Q., and Szepesvari, C.
\newblock Nearly minimax optimal reinforcement learning for linear mixture
  markov decision processes.
\newblock In \emph{Conference on Learning Theory}, pp.\  4532--4576. PMLR,
  2021.

\bibitem[Zhou et~al.(2022)Zhou, Wang, and Du]{paper:horizon_free_lmdp}
Zhou, R., Wang, R., and Du, S.~S.
\newblock Horizon-free reinforcement learning for latent markov decision
  processes.
\newblock \emph{arXiv preprint arXiv:2210.11604}, 2022.

\end{thebibliography}
\bibliographystyle{icml2023/icml2023}

\newpage
\appendix
\onecolumn

\section{Technical Lemmas}


\begin{lemma}[Hoeffding's Inequality]\label{lem:hoeffding}
Let $Z,Z_1,\ldots,Z_n$ be i.i.d. random variables with values in $[0,b]$ and let $\delta>0$.
Then we have
\begin{align*}
    \mathbb{P}\left[ \left|\mathbb{E}[Z]-\frac{1}{n}\sum_{i=1}^n Z_i\right|>b \sqrt{\frac{\ln(2/\delta)}{2 n}}\right]\le \delta.
\end{align*}
\end{lemma}

\begin{lemma}[Bennett's Inequality, Theorem\,3 in \citet{paper:bennett}]\label{lem:bennett}
Let $Z,Z_1,\ldots,Z_n$ be i.i.d. random variables with values in $[0,b]$ and let $\delta>0$. Define $\mathbb{V}[Z]=\mathbb{E}[(Z-\mathbb{E}[Z])^2]$. Then we have
\begin{align*}
    \mathbb{P}\left[ \left|\mathbb{E}[Z]-\frac{1}{n}\sum_{i=1}^n Z_i\right|>\sqrt{\frac{2\mathbb{V}[Z] \ln(2/\delta)}{n}}+\frac{b \ln(2/\delta)}{n}\right]\le \delta.
\end{align*}
\end{lemma}

\begin{lemma}[Theorem\,4 in \citet{paper:bennett}]\label{lem:bennett_empirical}
Let $Z,Z_1,\ldots,Z_n\ (n\ge 2)$ be i.i.d. random variables with values in $[0,b]$ and let $\delta>0$. Define $\bar{Z}=\frac{1}{n}Z_i$ and $\hat V_n=\frac{1}{n}\sum_{i=1}^n (Z_i-\bar Z)^2$. Then we have
\begin{align*}
    \mathbb{P}\left[\left|\mathbb{E}[Z]-\frac{1}{n}\sum_{i=1}^n Z_i\right|>\sqrt{\frac{2\hat V_n \ln(2/\delta)}{n-1}}+\frac{7 b \ln(2/\delta)}{3(n-1)}\right]\le \delta.
\end{align*}
\end{lemma}

\begin{lemma}[Lemma\,11 in \cite{paper:model_free_rl}]\label{lem:martingale_bound_primal}
Let $(M_n)_{n\ge 0}$ be a martingale such that $M_0=0$ and $|M_n-M_{n-1}|\le c$ for some $c>0$ and any $n\ge 1$. Let $\mathrm{Var}_n=\sum_{k=1}^n \mathbb{E}[(M_k-M_{k-1})^2|\mathcal{F}_{k-1}]$ for $n\ge 0$, where $\mathcal{F}_k=\sigma(M_1,\ldots,M_k)$. Then for any positive integer $n$ and any $\epsilon,\delta>0$, we have that
\begin{align*}
    \mathbb{P}\left[|M_n|\ge 2\sqrt{2\mathrm{Var}_n\ln(1/\delta)}+2\sqrt{\epsilon\ln(1/\delta)}+2c\ln(1/\delta)\right]\le 2\left(\log_2\left(\frac{nc^2}{\epsilon}\right)+1\right)\delta.
\end{align*}
\end{lemma}

\begin{lemma} [Lemma\,10 in \citet{paper:horizon_free}]\label{lem:martingale_conc_mean}
Let $X_1, X_2, \ldots$ be a sequence of random variables taking values in $[0, l]$.
Define $\cF_k = \sigma (X_1, X_2, \ldots, X_{k - 1})$ and $Y_k = \bbE [X_k\ |\ \cF_k]$ for $k \ge 1$.
For any $\delta > 0$, we have that
\begin{align*}
    &\bbP \left[ \exists n, \sum_{k = 1}^n X_k \ge 3 \sum_{k = 1}^n Y_k + l \ln (1 / \delta) \right] \le \delta, \\
    &\bbP \left[ \exists n, \sum_{k = 1}^n Y_k \ge 3 \sum_{k = 1}^n X_k + l \ln (1 / \delta) \right] \le \delta.
\end{align*}
\end{lemma}

\begin{lemma}[Lemma\,30 in \citet{paper:lcb_ssp}] \label{lem:var_xy}
For any two random variables $X, Y$, we have
\begin{align*}
    \bbV (X Y) \le 2 \bbV (X) (\sup \abs{Y})^2 + 2 (\bbE [X])^2 \bbV (Y).
\end{align*}
Consequently, $\sup \abs{X} \le C$ implies $\bbV (X^2) \le 4 C^2 \bbV (X)$.
\end{lemma}

\begin{lemma}[Bhatia–Davis Inequality] \label{lem:bhatia_davis}
For any random variable $X$, $\bbV (X) \le (\sup X - \bbE[X])(\bbE[X] - \inf X)$.
\end{lemma}

\section{Missing Proofs}

\subsection{Justification for \Cref{def:V_star}} \label{sec:justify_V_star}

Let $X_h^\pi (s)$ denote the random variable of cumulative reward starting from $s$ as the $h$-th step.
Clearly, $V_h^\pi (s) = \bbE [X_h^\pi (s)]$.
We denote $\Var_h^\pi (s) := \bbV (X_h^\pi (s))$.
Since $\pi \in \Pi$ is deterministic, let $a = \pi_h (s)$.
Law of total variance states that $\bbV (Y) = \bbE[\bbV (Y | X)] + \bbV (\bbE [Y | X] )$, so
\begin{align*}
    \Var_h^\pi (s)
    &= \bbE_{r \sim R_h (s, a), s'\sim P_{s, a, h}} [\bbV (r + X_{h+1}^\pi (s'))] + \bbV_{r \sim R_h (s, a), s'\sim P_{s, a, h}} (\bbE [r + X_{h+1}^\pi (s')]) \\
    &= \bbE_{r \sim R_h (s, a), s'\sim P_{s, a, h}} [\bbV (X_{h+1}^\pi (s'))] + \bbV_{r \sim R_h (s, a), s'\sim P_{s, a, h}} (r + \bbE [X_{h+1}^\pi (s')]) \\
    &= \bbE_{s'\sim P_{s, a, h}} [\Var_{h + 1}^\pi (s')] + \bbV_{r \sim R_h (s, a)} (r) + \bbV_{s'\sim P_{s, a, h}} (V_{h + 1}^\pi (s')) \\
    &= P_{s, a, h} \Var_{h + 1}^\pi + \bbV (R_h (s, a)) + \bbV (P_{s, a, h}, V_{h + 1}^\pi).
\end{align*}
Let $d_h^\pi \in \Delta (\cS) (\cdot | s)$ denote the state visitation distribution at the $h$-th step conditioned on the first state being $s$, i.e., 
\begin{align*}
    d_h^\pi (s' | s) := \bbP_\pi [s_h = s'\ |\ s_1 = s].
\end{align*}
By induction, we can prove that (with $a_h = \pi_h (s_h)$)
\begin{align*}
    \Var_1^\pi (s)
    &= \sum_{h = 1}^H \bbE_{s_h \sim d_h^\pi (\cdot | s)} [\bbV (R_h (s_h, a_h)) + \bbV (P_{s_h, a_h, h}, V_{h + 1}^\pi)] \\
    &= \bbE_\pi \left[ \left.\sum_{h = 1}^H \left( \bbV (R_h (s_h, a_h)) + \bbV (P_{s_h, a_h, h}, V_{h + 1}^\pi) \right)\ \right| \ s_1 = s\right].
\end{align*}

\subsection{Model-Based Algorithm: \mvpalt~(\Cref{alg:mvp_alt})} \label{sec:mb_proof}

\paragraph{Summary of notations.}
Let $\sahk$ and $r_h^k$ denote the state, action and reward at the $h$-th step of the $k$-th episode.
Let $V_h^k (s)$, $Q_h^k (s, a)$, $n^k (s, a)$ and $\wh{P}_{s, a, s'}^k$ denote $V_h (s)$, $Q_h (s, a)$, $n (s, a)$ and $\wh{P}_{s, a, s'}$ at the beginning of the $k$-th episode.

Let $\cK$ be the set of indexes of episodes in which no update is triggered.
By the update rule, it is obvious that $\abs{\cK^C} \le S A (\log_2 (K H) + 1)$.
Let $h_0(k)$ be the first time an update is triggered in the $k$-th episode if there is an update in this episode and otherwise $H + 1$.
Define $\cX_0 := \{(k, h_0 (k))\ |\ k \in \cK^C\}$ and $\cX := \{(k, h)\ |\ k \in \cK^C, h_0(k) + 1 \le h\le H\}$.

Let $I(k, h) := \II [(k, h) \not\in \cX]$.
We use the ``check'' notation to denote the original value timed with $I(k, h)$, e.g., $\wc{V}_h^k := V_h^k I (k, h)$ and $\wc{\beta}_h^k := \beta_h^k I (k, h)$.

\begin{proof} [Proof of \Cref{thm:main_mb}]
We first run \mvpalt~(\Cref{alg:mvp_alt}) with $\iota = 99 (\ln (H S A K / \delta) + 1)$ which is large enough for all the probabilistic inequalities to hold.
This choice will make the success probability be $1 - \poly (S, A, H, K, \iota) \delta$.
The lemmas are also proved assuming this choice of $\iota$ at first.

Based on Lemma\,7 in \citet{paper:mvp}, by \Cref{lem:opt,lem:bellman_error} we have that
\begin{align*}
    \Regret (K)
    \le \underbrace{\sumkh (P_{\sahk} \wc{V}_{h + 1}^k - \wc{V}_{h + 1}^k (s_{h + 1}^k))}_{=: M_1} + \underbrace{\sumkh \wc{\beta}_h^k (\sahk)}_{=: M_2} + \underbrace{\sum_{k = 1}^K \left( \sum_{h = 1}^H r (\sahk) I(k, h) - V_1^{\pi^k} (s_1^k) \right)}_{=: M_3} + \abs{\cK^C}.
\end{align*}
We utilize Equation\,(39) in \citet{paper:mvp}:
for any non-negative sequence $(w_h^k)_{k \in [K], h \in [H]}$,
\begin{align*}
    \sumkh \frac{I (k, h) }{n^k (\sahk)} \le O ( S A \iota ), \quad \sumkh \sqrt{\frac{w_h^k I (k, h)}{n^k (\sahk)}} \le O \left( \sqrt{ S A \iota \sumkh w_h^k I(k, h)} + S A \iota \right).
\end{align*}
Thus by \Cref{lem:bellman_error} we have
\begin{align*}
    M_2
    &\le O \left(\rule{0cm}{1.3cm}\right. \sqrt{S A \underbrace{\sumkh (\bbV (R (\sahk)) + \bbV (P_{\sahk}, V_{h + 1}^k)) I (k, h)}_{\approx: M_4} } \iota + \sqrt{\Gamma S A \underbrace{\sumkh \bbV (P_{\sahk}, V_{h + 1}^k - V_{h + 1}^\star) I (k, h)}_{\approx: M_5} } \iota \\
    &\qquad\quad + \sqrt{S A \underbrace{\sumkh (\bbV (R (\sahk)) + \bbV (P_{\sahk}, V_{h + 1}^\star)) I (k, h)}_{\approx: M_6} } \iota + \Gamma S A \iota^2 \left.\rule{0cm}{1.3cm}\right).
\end{align*}
We need to substitute $I(k, h)$ with $I(k, h + 1)$ to get the precise definition of $M_4, M_5$ and $M_6$.
Such substitution only introduces an error of $O(\abs{\cK^C})$.
By $\bbV (X + Y) \le 2 \bbV (X) + 2 \bbV (Y)$,
\begin{align*}
    M_6 \le O(M_4 + M_5), \quad \textup{and} \quad M_4 \le O(M_5 + M_6).
\end{align*}

\begin{itemize}
    \item If we use the former relation, $M_2 \le O (\sqrt{S A M_4} \iota + \sqrt{\Gamma S A M_5} \iota + \Gamma S A \iota^2)$.
    Plugging in \Cref{lem:M4} and \Cref{lem:M5}, we have
    \begin{align*}
        M_2 \le O \left( \sqrt{\Gamma S A M_2} \iota + \sqrt{S A \sum_{k = 1}^K \Var^{\pi^k}} \iota + \Gamma S A \iota^2 \right).
    \end{align*}
    Solving the inequality gives
    \begin{align*}
        M_2 \le O \left( \sqrt{S A \sum_{k = 1}^K \Var^{\pi^k}} \iota + \Gamma S A \iota^2 \right).
    \end{align*}
    By Lemma\,8 of \citet{paper:mvp}, $M_1 \le O (\sqrt{M_4 \iota} + \iota)$.
    Further plugging in \Cref{lem:M3} gives
    \begin{align*}
        \Regret (K)
        \le M_1 + M_2 + M_3 + \abs{\cK^C}
        \le O \left( \sqrt{S A \sum_{k = 1}^K \Var^{\pi^k}} \iota + \Gamma S A \iota^2 \right)
        \le O (\sqrt{\maxvar S A K} \iota + \Gamma S A \iota^2).
    \end{align*}

    \item If we use the latter relation, $M_2 \le O (\sqrt{S A M_6} \iota + \sqrt{\Gamma S A M_5} \iota + \Gamma S A \iota^2)$.
    First plug in \Cref{lem:M5}, we have $M_2 \le O (\sqrt{\Gamma S A M_2} \iota + \sqrt{S A M_6} \iota + \Gamma S A \iota^2)$, which implies
    \begin{align*}
        M_2 \le O (\sqrt{S A M_6} \iota + \Gamma S A \iota^2).
    \end{align*}
    For the regret, we need $M_1 \le O (\sqrt{M_4 \iota} + \iota)$, \Cref{lem:M3,lem:M6} and $M_4 \le O (M_5 + M_6)$, so
    \begin{align*}
        \Regret (K)
        \le O (\sqrt{S A M_6} \iota + \Gamma S A \iota^2)
        \le O ( \sqrt{\multistepvar S A} \iota + \Gamma S A \iota^2 ).
    \end{align*}
\end{itemize}

The above results hold with probability at least $1 - 19 H S^2 A K \iota \delta$.
To establish the final result, we need to scale $\delta$ to make the success probability be $1 - \delta'$.
We upper-bound $\iota$ by $100 (H S A K / \sqrt{\delta} + 1)$ and solve the inequality:
\begin{align*}
    1900 H S^2 A K \left( \frac{H S A K}{\sqrt{\delta}} + 1 \right) \delta \le \delta'.
\end{align*}
Take $\delta = (\delta' / 3000 H^2 S^3 A^2 K^2)^2$.
By $\ln (H S A K / (\delta / 3000 H^2 S^3 A^2 K^2)^2) \le O (\iota)$, we conclude the proof.
\end{proof}

\begin{restatable}{lemma}{lemGoodEvents}\label{lem:good_events}
Define the following events:
\begin{align}
    \cE_1 &:= \left\{ \forall (s, a, h, k) \in \SA \times [H] \times [K],\ \abs{ (\wh{P}_{s, a, s'}^k - P_{s, a, s'}) V_{h + 1}^\star } \le 2 \sqrt{\frac{ \bbV (\wh{P}_{s, a}^k, V_{h + 1}^\star) \iota}{n^k (s, a)}} + \frac{14 \iota}{3 n^k (s, a)} \right\}, \label{eq:event1} \\
    \cE_2 &:= \left\{ \forall (s, a, k) \in \SA \times [K],\ \abs{\wh{r}^k (s, a) - r (s, a) } \le 2 \sqrt{\frac{\wh{\VarR}^k (s, a) \iota}{n^k (s, a)}} + \frac{14 \iota}{3 n^k (s, a)} \right\}, \label{eq:event2} \\
    \cE_3 &:= \left\{ \forall (s, a, s', k) \in \SAS \times [K],\ \abs{ \wh{P}_{s, a, s'}^k - P_{s, a, s'} } \le \sqrt{\frac{2 P_{s, a, s'} \iota }{n^k (s, a)}} + \frac{\II[P_{s, a, s'} > 0] \iota}{n^k (s, a)} \right\}, \label{eq:event3} \\
    \cE_4 &:= \left\{ \forall (s, a, h, k) \in \SA \times [H] \times [K],\ \abs{ (\wh{P}_{s, a, s'}^k - P_{s, a, s'}) V_{h + 1}^\star } \le \sqrt{\frac{ 2 \bbV (P_{s, a}, V_{h + 1}^\star) \iota}{n^k (s, a)}} + \frac{\iota}{n^k (s, a)} \right\}. \label{eq:event4}
\end{align}
We have that
\begin{align*}
    \bbP [\cE_1] \ge 1 - H S A K \iota \delta,
    \quad \bbP [\cE_2] \ge 1 - S A K \iota \delta,
    \quad \bbP [\cE_3] \ge 1 - S^2 A K \iota \delta,
    \quad \bbP [\cE_4] \ge 1 - H S A K \iota \delta.
\end{align*}
\end{restatable}

\begin{proof} [Proof of \Cref{lem:good_events}]
$\bbP[\cE_1]$ and $\bbP[\cE_2]$ are direct results by applying \Cref{lem:bennett_empirical} and $\frac{1}{x - 1} \le \frac{2}{x}$, taking union bounds over the mentioned quantifiers and that $n^k (s, a) \in \{1, 2, \ldots, \floor{\log_2 (H K)}\}$.
$\bbP[\cE_3]$ and $\bbP[\cE_4]$ are direct results by applying \Cref{lem:bennett} and that $P_{s, a, s'} = 0 \implies \wh{P}_{s, a, s'}^k = 0$, finally taking union bounds over the mentioned quantifiers and that $n^k (s, a) \in \{1, 2, \ldots, \floor{\log_2 (H K)}\}$.
\end{proof}

\begin{restatable} [Adapted from Lemma\,14 in \citet{paper:mvp} and Lemma\,16 in \citet{paper:eb_ssp}] {lemma}{lemMvpFunc}\label{lem:mvp_func}
For any fixed dimension $D$, let $\Upsilon := \{ v \in \bbR^{D}: v \ge 0, \norm{v}_{\infty} \leq B \}$.
For any two constants $c_1, c_2$ satisfying $c_1^2 \le c_2$, let
$f: \Delta([D]) \times \Upsilon \times \bbR \times \bbR \rightarrow \bbR$ with $f(p, v, n, \iota) = p v + \max\left\{ c_1 \sqrt{\frac{\bbV (p, v) \iota}{n}},\, c_2 \frac{ B \iota }{n} \right\}$. Then for all $p \in \Delta([D]), v \in \Upsilon$ and $n, \iota > 0$,
\begin{enumerate}
    \item $f(p, v, n, \iota)$ is non-decreasing in $v$, i.e., 
    \begin{align*}
        \forall (v, v') \in \Upsilon^2, v \le v', \textup{ it holds that } f(p, v, n, \iota) \le f(p, v', n, \iota);
    \end{align*}
    
    \item $f(p, v, n, \iota) \ge p v + \frac{c_1}{2} \sqrt{\frac{\mathbb{V}(p,v) \iota}{n}} + \frac{c_2}{2} \frac{ B \iota}{n}$.
\end{enumerate}
\end{restatable}

\begin{restatable}{lemma}{lemOpt}\label{lem:opt}
Conditioned on the successful events of \Cref{lem:good_events}, we have that for any $(s, a, h, k) \in \SA \times [H] \times [K]$, $Q_h^k (s, a) \ge Q_h^\star (s, a)$.
\end{restatable}

\begin{proof} [Proof of \Cref{lem:opt}]
Let $k$ be fixed and omit it for simplicity.
The proof is conducted by induction in the order of $h = H + 1, H, \ldots, 1$.
$Q_{H + 1} (s, a) = 0 \ge 0 = Q_{H + 1}^\star (s, a)$ holds trivially for any $(s, a) \in \SA$.
Now assume $Q_{h + 1} (s, a) \ge Q_{h + 1}^\star (s, a)$ for any $(s, a) \in \cS \times \cA$, hence $V_{h + 1} (s) = \max_{a \in \cA} Q_{h + 1} (s, a) \ge \max_{a \in \cA} Q_{h + 1}^\star (s, a) = V_{h + 1}^\star (s)$ for any $s \in \cS$.
\begin{align*}
    &\wh{r}(s,a)+\wh{P}_{s,a} V_{h+1} +b_h(s,a) \\
    &= \left( \wh{r} (s, a) + 2 \sqrt{ \frac{ \wh{\VarR}(s,a) \iota }{n(s,a) } } +  \frac{5 \iota}{ n(s,a)  } \right) + \left( \wh{P}_{s,a} V_{h+1} + 4 \sqrt{\frac{   \mathbb{ V}(\wh{P}_{s,a} ,V_{h+1}) \iota  }{ n(s,a) }}+  \frac{16 \iota}{ n(s,a) } \right) \\
    &\mygei r (s, a) + \underbrace{ \wh{P}_{s,a} V_{h+1} + \max \left\{ 4 \sqrt{\frac{  \mathbb{ V}(\wh{P}_{s,a} ,V_{h+1}) \iota  }{ n(s,a) }}, \frac{16 \iota}{ n(s,a) } \right\} }_{ = f(\wh{P}_{s, a}, V_{h + 1}, \iota, n (s, a))} \\
    &\mygeii r (s, a) + \underbrace{ \wh{P}_{s,a} V_{h+1}^\star + \max \left\{ 4 \sqrt{\frac{   \mathbb{ V}(\wh{P}_{s,a} ,V_{h+1}^\star) \iota  }{ n(s,a) }}, \frac{16 \iota}{ n(s,a) } \right\} }_{ = f(\wh{P}_{s, a}, V_{h + 1}^\star, \iota, n (s, a))} \\
    &\mygeiii r (s, a) + \wh{P}_{s,a} V_{h+1}^\star + 2 \sqrt{\frac{   \mathbb{ V}(\wh{P}_{s,a} ,V_{h+1}^\star) \iota  }{ n(s,a) }} + \frac{8 \iota}{ n(s,a) } \\
    &\mygeiv r (s, a) + P_{s, a} V_{h + 1}^\star \\
    &= Q_h^\star (s, a),
\end{align*}
where (i) is by $\cE_2$ (\Cref{eq:event2});
(ii) is by recognizing the last part as the function in \Cref{lem:mvp_func}, $(c_1, c_2, B) = (4, 16, 1)$ satisfying that $c_1^2 \le c_2$ and using the first property based on the induction that $V_{h + 1} \ge V_{h + 1}^\star$;
(iii) is by the second property in \Cref{lem:mvp_func};
(iv) is $\cE_1$ (\Cref{eq:event1}).
So $Q_h (s, a) \ge Q_h^\star (s, a)$.
\end{proof}

\begin{restatable}{lemma}{lemVarR} \label{lem:VarR}
With probability at least $1 - 2 S A K \iota \delta$, we have that for any $(s, a, k) \in \cS \times \cA \times [K]$,
\begin{align*}
    \wh{\VarR}^k (s, a) \le O \left( \bbV (R (s, a)) + \frac{\iota}{n^k (s, a)} \right).
\end{align*}
\end{restatable}

\begin{proof} [Proof of \Cref{lem:VarR}]
Let $s, a, k$ be fixed and omit $k$ for simplicity.
Assume that all the $n (s, a)$ realizations of $R (s, a)$ are $(r_{s, a}^{(i)})_{i = 1}^{n (s, a)}$.
We have that
\begin{align*}
    \wh{\VarR} (s, a) = \frac{1}{n (s, a)} \sum_{i = 1}^{n (s, a)} (r_{s, a}^{(i)})^2 - \left( \frac{1}{n (s, a)} \sum_{i = 1}^{n (s, a)} r_{s, a}^{(i)} \right)^2.
\end{align*}
From \Cref{lem:bennett},
\begin{align*}
    \bbP \left[ \abs{\frac{1}{n (s, a)} \sum_{i = 1}^{n (s, a)} (r_{s, a}^{(i)})^2 - \bbE[R (s, a)^2]} > \sqrt{\frac{2 \bbV (R (s, a)^2) \iota}{n (s, a)}} + \frac{\iota}{n (s, a)} \right] &\le \delta, \\
    \bbP \left[ \abs{\frac{1}{n (s, a)} \sum_{i = 1}^{n (s, a)} r_{s, a}^{(i)} - \bbE[R (s, a)]} > \sqrt{\frac{2 \bbV (R (s, a)) \iota}{n (s, a)}} + \frac{\iota}{n (s, a)} \right] &\le \delta.
\end{align*}
By \Cref{lem:var_xy}, $\bbV (R (s, a)^2) \le 4 \bbV (R (s, a))$.
So
\begin{align*}
    & \abs{\wh{\VarR} (s, a) - \bbV (R (s, a))} \\
    &\le \abs{\frac{1}{n (s, a)} \sum_{i = 1}^{n (s, a)} (r_{s, a}^{(i)})^2 - \bbE[R (s, a)^2]} + \abs{\left(\frac{1}{n (s, a)} \sum_{i = 1}^{n (s, a)} r_{s, a}^{(i)} \right)^2 - \bbE[R (s, a)]^2} \\
    &\le 2 \sqrt{\frac{2 \bbV (R (s, a)) \iota}{n (s, a)}} + \frac{\iota}{n (s, a)} + \abs{\frac{1}{n (s, a)} \sum_{i = 1}^{n (s, a)} r_{s, a}^{(i)} + \bbE[R (s, a)]} \abs{\frac{1}{n (s, a)} \sum_{i = 1}^{n (s, a)} r_{s, a}^{(i)} - \bbE[R (s, a)]} \\
    &\le 2 \sqrt{\frac{2 \bbV (R (s, a)) \iota}{n (s, a)}} + \frac{\iota}{n (s, a)} + 2 \sqrt{\frac{2 \bbV (R (s, a)) \iota}{n (s, a)}} + \frac{2 \iota}{n (s, a)} \\
    &= 4 \sqrt{\frac{2 \bbV (R (s, a)) \iota}{n (s, a)}} + \frac{3 \iota}{n (s, a)}.
\end{align*}
Using $2 \sqrt{x y} \le x + y$, we have
\begin{align*}
    \wh{\VarR} (s, a)
    \le \bbV (R (s, a)) + 4 \sqrt{\frac{2 \bbV (R (s, a)) \iota}{n (s, a)}} + \frac{3 \iota}{n (s, a)}
    \le 2 \bbV (R (s, a)) + \frac{11 \iota}{n (s, a)}.
\end{align*}
This completes the proof.
\end{proof}

\begin{restatable} {lemma}{lemBellmanError}\label{lem:bellman_error}
Conditioned on the successful events of \Cref{lem:good_events,lem:VarR}, we have that for any $(s, a, h, k) \in \SA \times [H] \times [K]$
\begin{align*}
    Q_h^k (s, a) - r (s, a) - P_{s, a} V_{h + 1}^k \le \beta_h^k (s, a),
\end{align*}
where
\begin{align*}
    \beta_h^k (s, a)
    = O \left( \sqrt{\frac{\bbV (P_{s, a}, V_{h + 1}^k) \iota}{n^k (s, a)}} + \sqrt{\frac{\bbV (P_{s, a}, V_{h + 1}^\star) \iota}{n^k (s, a)}} + \sqrt{\frac{\Gamma \bbV (P_{s, a}, V_{h + 1}^k - V_{h + 1}^\star) \iota}{n^k (s, a)}} + \sqrt{\frac{\bbV( R (s, a)) \iota}{n^k (s, a)}} + \frac{\Gamma \iota}{n^k (s, a)}\right).
\end{align*}
\end{restatable}

\begin{proof} [Proof of \Cref{lem:bellman_error}]
For any $(s, a, h, k) \in \SA \times [H] \times [K]$,
\begin{align*}
    &(\wh{P}_{s, a}^k - P_{s, a}) (V_{h + 1}^k - V_{h + 1}^\star) \\
    &= \sum_{s' \in \cS} (\wh{P}_{s, a, s'}^k - P_{s, a, s'}) [V_{h + 1}^k (s') - V_{h + 1}^\star (s') - P_{s, a} (V_{h + 1}^k - V_{h + 1}^\star)] \\
    &\mylei \sum_{s' \in \cS} \sqrt{\frac{2 P_{s, a, s'} \iota }{n^k (s, a)}} \abs{V_{h + 1}^k (s') - V_{h + 1}^\star (s') - P_{s, a} (V_{h + 1}^k - V_{h + 1}^\star)} + \sum_{s' \in \cS} \frac{\II[P_{s, a, s'} > 0] \iota}{n^k (s, a)}\\
    &\le \sqrt{\frac{2 \iota }{n^k (s, a)}} \sum_{s' \in \cS} \sqrt{ \II [P_{s, a, s'} > 0] P_{s, a, s'} [V_{h + 1}^k (s') - V_{h + 1}^\star (s') - P_{s, a} (V_{h + 1}^k - V_{h + 1}^\star)]^2} + \frac{\Gamma \iota}{n^k (s, a)} \\
    &\myleii \sqrt{\frac{2 \iota }{n^k (s, a)}} \sqrt{\sum_{s' \in \cS} \II [P_{s, a, s'} > 0]} \sqrt{\sum_{s' \in \cS} P_{s, a, s'} [V_{h + 1}^k (s') - V_{h + 1}^\star (s') - P_{s, a} (V_{h + 1}^k - V_{h + 1}^\star)]^2} + \frac{\Gamma \iota}{n^k (s, a)} \\
    &= \sqrt{\frac{2 \Gamma \bbV (P_{s, a}, V_{h + 1}^k - V_{h + 1}^\star)}{n^k (s, a)}} + \frac{\Gamma \iota}{n^k (s, a)},
\end{align*}
where (i) is by $\cE_3$ (\Cref{eq:event3});
(ii) is by Cauchy-Schwarz inequality.
While retaining most other steps in Appendix\,C.1 of \citet{paper:mvp} which require $\cE_4$ (\Cref{eq:event4}), we have
\begin{align*}
    \beta_h^k (s ,a)
    = O \left( \sqrt{\frac{\bbV (\wh{P}_{s, a}^k, V_{h + 1}^k) \iota}{n^k (s, a)}} + \sqrt{\frac{\bbV (P_{s, a}, V_{h + 1}^\star) \iota}{n^k (s, a)}} + \sqrt{\frac{\Gamma \bbV (P_{s, a}, V_{h + 1}^k - V_{h + 1}^\star) \iota}{n^k (s, a)}} + \sqrt{\frac{\wh{\VarR}^k (s, a) \iota}{n^k (s, a)}} + \frac{\Gamma \iota}{n^k (s, a)}\right).
\end{align*}
Similar as the steps above Equation\,(36) in \citet{paper:mvp} which require $\cE_3$ (\Cref{eq:event3}), we have that
\begin{align*}
    \bbV (\wh{P}_{s, a}^k, V_{h + 1}^k ) \le O \left( \bbV (P_{s, a}, V_{h + 1}^k) + \frac{\Gamma \iota}{n^k (s, a)} \right).
\end{align*}
Combined with \Cref{lem:VarR} we have the desired result.
\end{proof}

\begin{restatable}{lemma}{lemM4} \label{lem:M4}
Conditioned on the successful events of \Cref{lem:bellman_error}, with probability at least $1 - 5 K \iota \delta$, we have that
\begin{align*}
    M_4 \le O \left( \sum_{k = 1}^K \Var^{\pi^k} + M_2 + S A \iota^2 \right) \le O (\maxvar K + M_2 + S A \iota^2).
\end{align*}
\end{restatable}

\begin{proof} [Proof of \Cref{lem:M4}]
Define
\begin{align}
    \bc_h^k (s, a) := V_h^k (s) - P_{s, a} V_{h + 1}^k - r (s, a) \in [-1, 1], \label{eq:bc_def}
\end{align}
which stands for bonus-correction.
By \Cref{lem:bellman_error}, $\bc_h^k (s, a) \le \beta_h^k (s, a)$.
However, we make the distinction here to be more precise.
Let $\BC_h^k (s) := \bc_h^k (s, a) + P_{s, a} \BC_{h + 1}^k$ with $a = \pi_h^k (s)$ and boundary condition $\BC_{H+1}^k (s) := 0$.
We can prove by induction that
\begin{align}
    \BC_h^k (s) = (V_h^k - V_h^{\pi^k}) (s) \in [0, 1]. \label{eq:BC}
\end{align}
First, 
\begin{align*}
    \BC_H^k (s)
    = \bc_H^k (s, a)
    = V_H^k (s) - r (s, a)
    = V_H^k (s) - V_H^{\pi^k} (s).
\end{align*}
Then assume that $\BC_{h + 1}^k = V_{h + 1}^k - V_{h + 1}^{\pi^k}$, we have
\begin{align*}
    \BC_h^k (s)
    &= \bc_h^k (s, a) + P_{s, a} (V_{h + 1}^k - V_{h + 1}^{\pi^k}) \\
    &= V_h^k (s) - P_{s, a} V_{h + 1}^k - r (s, a) + P_{s, a} (V_{h + 1}^k - V_{h + 1}^{\pi^k}) \\
    &= V_h^k (s) - (r (s, a) + P_{s, a} V_{h + 1}^{\pi^k}) \\
    &= V_h^k (s) - V_h^{\pi^k} (s).
\end{align*}

Define a series of random variables and their truncated values: for any $k \in [K]$,
\begin{align*}
    W^k  &:= \sum_{h = 1}^H (\bbV (R (\sahk) ) + \bbV (P_{\sahk}, V_{h + 1}^{\pi^k})),
    \quad \ov{W^k} := \min\{ W^k, 50 \iota \}.
\end{align*}
Correspondingly, define the following event, which means there is no truncation:
\begin{align*}
    \cE_{W} &:= \{ W^k = \ov{W^k},\ \forall k \in [K] \}.
\end{align*}

We now calculate the probability of no truncation happens.
For any fixed $1 \le k \le K$,
\begin{align*}
    W^k
    & \mylei \sum_{h=1}^H [P_{\sahk} (V_{h+1}^{\pi^k})^2 - (V_{h+1}^{\pi^k} (s_{h+1}^k))^2] + \sum_{h=1}^H [ (V_h^{\pi^k}(s_h^k))^2 - (P_{\sahk}V_{h+1}^{\pi^k})^2]  + \sum_{h = 1}^H r(\sahk) - (V_1^{\pi^k} (s_1^k))^2 \\
    & \myleii 2 \sqrt{ 2 \sum_{h=1}^H \bbV(P_{\sahk}, (V_{h+1}^{\pi^k})^2 ) \iota } + 6 \iota + 2 \sum_{h=1}^H (V_h^{\pi^k}(s_h^k) - P_{\sahk}V_{h+1}^{\pi^k}) + \sum_{h = 1}^H r(\sahk) \\
    & \myleiii 4 \sqrt{ 2 \sum_{h=1}^H \bbV(P_{\sahk}, V_{h+1}^{\pi^k} ) \iota } + 3 \sum_{h=1}^H r(\sahk) + 6 \iota \\
    & \le 4 \sqrt{2 W^k \iota} + 3 + 6 \iota,
\end{align*}
where (i) is by \Cref{lem:bhatia_davis}, $\bbV( R (s, a)) \le \bbE [ R (s, a) ]$;
(ii) is by \Cref{lem:martingale_bound_primal} with $c = \epsilon = 1$, which happens with probability at least $1 - 2 \iota \delta$, and $a^2 - b^2 \le (a + b) \max \{a - b, 0\}$ when $a, b \ge 0$;
(iii) is by \Cref{lem:var_xy} with $C = 1$.
Solving the inequality of $W^k$, we have that
\begin{align*}
    W^k \le 50 \iota.
\end{align*}
This means $\bbP [\cE_W] \ge 1 - 2 K \iota \delta$.

From now on, we suppose $\cE_W$ holds.
We are ready to bound $M_4$:
\begin{align}
    M_4
    &= \sumkh (\bbV ( R (\sahk)) + \bbV (P_{\sahk}, V_{h + 1}^k)) I (k, h + 1) \notag \\
    &\mylei 2 \sumkh (\bbV ( R (\sahk)) + \bbV (P_{\sahk}, V_{h + 1}^{\pi^k})) I (k, h + 1) + 2 \underbrace{\sumkh \bbV (P_{\sahk}, \BC_{h + 1}^k) I (k, h + 1)}_{=: Z} \label{eq:Z_def} \\
    &\le 2 \sum_{k = 1}^K W^k + 2 Z \notag \\
    &\myeqii 2 \sum_{k = 1}^K \ov{W^k} + 2 Z \notag \\
    &\myleiii 6 \sum_{k = 1}^K \bbE[\ov{W^k}\ |\ \cF_k] + 2 Z + 100 \iota \notag \\
    &\le 6 \sum_{k = 1}^K \bbE[W^k\ |\ \cF_k] + 2 Z + 100 \iota^2 \notag \\
    &= 6 \sum_{k = 1}^K \Var_1^{\pi^k} (s_1^k) + 2 Z + 100 \iota^2 \notag \\
    &\le 6 \sum_{k = 1}^K \Var^{\pi^k} + 2 Z + 100 \iota^2 \notag \\
    &\le 6 \maxvar K + 2 Z + 100 \iota^2, \notag
\end{align}
where (i) is by \Cref{eq:BC} and $\bbV (X + Y) \le 2 \bbV (X) + 2 \bbV (Y)$;
(ii) is by $\cE_W$;
(iii) is by \Cref{lem:martingale_conc_mean} with $l = 50 \iota$, which happens with probability at least $1 - \delta$.

It remains to bound the quantity $Z$ we encountered:
\begin{align*}
    Z
    &= \sumkh [P_{\sahk} (\BC_{h + 1}^k)^2 - (\BC_{h + 1}^k (s_{h + 1}^k))^2] I(k, h + 1) \\
    &\quad + \sumkh [(\BC_h^k (s_h^k))^2 I (k, h) - (P_{\sahk} \BC_{h + 1}^k)^2 I(k, h + 1)] - \sum_{k = 1}^K (\BC_1^k (s_1^k))^2 \\
    &\le \sumkh [P_{\sahk} (\BC_{h + 1}^k)^2 - (\BC_{h + 1}^k (s_{h + 1}^k))^2] I(k, h + 1) \\
    &\quad + \sumkh [(\BC_h^k (s_h^k))^2 - (P_{\sahk} \BC_{h + 1}^k)^2] I (k, h + 1) + \abs{\cK^C} \\
    &\mylei 2 \sqrt{2 \sumkh \bbV(P_{\sahk}, (\BC_{h + 1}^k)^2) I(k, h + 1) \iota} + 6 \iota \\
    &\quad + 2 \sumkh \max\{ \BC_h^k (s_h^k) - P_{\sahk} \BC_{h + 1}^k, 0\} I(k, h + 1) + \abs{\cK^C} \\
    &\myleii 4 \sqrt{2 \sumkh \bbV(P_{\sahk}, \BC_{h + 1}^k) I(k, h + 1) \iota} + 6 \iota + 2 \sumkh \max\{ \bc_h^k (\sahk), 0 \} I(k, h + 1) + \abs{\cK^C} \\
    &\le 4 \sqrt{2 Z \iota} + 6 \iota + 2 \sumkh \wc{\beta}_h^k (\sahk) + 2 \abs{\cK^C} \\
    &\le 4 \sqrt{2 Z \iota} + 2 M_2 + 8 S A \iota,
\end{align*}
where (i) is by \Cref{lem:martingale_bound_primal} with $c = \epsilon = 1$, which happens with probability at least $1 - 2 \iota \delta$; (ii) is by \Cref{lem:var_xy} with $C = 1$.
Solving the inequality of $Z$, we have that
\begin{align}
    Z \le 4 M_2 + 48 S A \iota. \label{eq:Z}
\end{align}
So plugging back into the bound of $M_4$ gives the final result.
\end{proof}

\begin{lemma} \label{lem:M3}
Conditioned on the successful events of \Cref{lem:M4}, with probability at least $1 - 2 \iota \delta$, we have that
\begin{align*}
    M_3 \le O (\sqrt{M_4 \iota} +  \sqrt{M_2 \iota} + S A \iota).
\end{align*}
\end{lemma}

\begin{proof} [Proof of \Cref{lem:M3}]
\begin{align*}
    M_3
    &= \sum_{k = 1}^K \left( \sum_{h = 1}^H r (\sahk) I(k, h) - V_1^{\pi^k} (s_1^k) I(k, 1) \right) \\
    &= \sum_{k = 1}^K \left( \sum_{h = 1}^H ( V_h^{\pi^k} (s_h^k) - P_{\sahk} V_{h + 1}^{\pi^k} ) I(k, h) - V_1^{\pi^k} (s_1^k) I(k, 1) \right) \\
    &\le \sumkh (V_{h + 1}^{\pi^k} (s_{h + 1}^k) - P_{\sahk} V_{h + 1}^{\pi^k}) I(k, h + 1) + \abs{\cK^C} \\
    &\mylei 2 \sqrt{2 \sumkh \bbV (P_{\sahk}, V_{h + 1}^{\pi^k}) I(k, h + 1) \iota} + 6 \iota + S A \iota \\
    &\myleii 4 \sqrt{M_4 \iota} + 4 \sqrt{Z \iota} + 7 S A \iota \\
    &\myleiii 4 \sqrt{M_4 \iota} + 8 \sqrt{M_2 \iota} + 35 S A \iota,
\end{align*}
where (i) is by \Cref{lem:martingale_bound_primal} with $c = \epsilon = 1$, which happens with probability at least $1 - 2 \iota \delta$;
(ii) is by \Cref{eq:BC}, $\bbV (X + Y) \le 2 \bbV (X) + 2 \bbV (Y)$ and definition of $Z$ (\Cref{eq:Z_def});
(iii) is by \Cref{eq:Z}.
\end{proof}

\begin{lemma} \label{lem:M5}
Conditioned on the successful events of \Cref{lem:bellman_error}, with probability at least $1 - 2 \iota \delta$, we have that
\begin{align*}
    M_5 \le O (M_2 + S A \iota).
\end{align*}
\end{lemma}

\begin{proof} [Proof of \Cref{lem:M5}]
Define $\wt{V}_h^k = V_h^k - V_h^\star$.
\begin{align*} 
    M_5
    &= \sumkh [P_{\sahk} (\wt{V}_{h + 1}^k)^2 - (\wt{V}_{h + 1}^k (s_{h + 1}^k))^2] I(k, h + 1) \\
    &\quad + \sumkh [(\wt{V}_h^k (s_h^k))^2 I (k, h) - (P_{\sahk} \wt{V}_{h + 1}^k)^2 I(k, h + 1)] - \sum_{k = 1}^K (\wt{V}_1^k (s_1^k))^2 \\
    &\le \sumkh [P_{\sahk} (\wt{V}_{h + 1}^k)^2 - (\wt{V}_{h + 1}^k (s_{h + 1}^k))^2] I(k, h + 1) \\
    &\quad + \sumkh [(\wt{V}_h^k (s_h^k))^2 - (P_{\sahk} \wt{V}_{h + 1}^k)^2] I (k, h + 1) + \abs{\cK^C} \\
    &\mylei 2 \sqrt{2 \sumkh \bbV(P_{\sahk}, (\wt{V}_{h + 1}^k)^2) I(k, h + 1) \iota} + 6 \iota \\
    &\quad + 2 \sumkh \max\{ \wt{V}_h^k (s_h^k) - P_{\sahk} \wt{V}_{h + 1}^k, 0\} I(k, h + 1) + \abs{\cK^C} \\
    &\myleii 4 \sqrt{2 M_5 \iota} + 2 \sumkh \wc{\beta}_h^k (\sahk) + 2 \abs{\cK^C} \\
    &\le 4 \sqrt{2 M_5 \iota} + 2 M_2 + 8 S A \iota,
\end{align*}
where (i) is by \Cref{lem:martingale_bound_primal} with $c = \epsilon = 1$, which happens with probability at least $1 - 2 \iota \delta$;
(ii) is by \Cref{lem:var_xy} with $C = 1$ and the following argument:
by \Cref{lem:bellman_error},
\begin{align*}
    \wt{V}_h^k (s_h^k) - P_{\sahk} \wt{V}_{h + 1}^k
    \le \wt{Q}_h^k (\sahk) - P_{\sahk} \wt{V}_{h + 1}^k
    \le \beta_h^k (\sahk).
\end{align*}
Solving the inequality of $M_5$, we have that
\begin{align*}
    M_5 \le 4 M_2 + 48 S A \iota.
\end{align*}
This completes the proof.
\end{proof}

\begin{lemma}\label{lem:M6}
With probability at least $1 - 4 K \iota \delta$, we have that for any $k \in [K]$,
\begin{align*}
    \multistepvarepisode \le O (\iota).
\end{align*}
As a result,
\begin{align*}
    M_6 \le \multistepvar \le O ( K \iota ).
\end{align*}
\end{lemma}

\begin{proof} [Proof of \Cref{lem:M6}]
For any $k \in [K]$,
\begin{align*}
    \multistepvarepisode
    &\le \sum_{h = 1}^H [P_{\sahk} (V_{h + 1}^\star)^2 - (V_{h + 1}^\star (s_{h + 1}^k))^2] + \sum_{h = 1}^H [(V_h^\star (s_h^k))^2 - (P_{\sahk} V_{h + 1}^\star)^2] + \sum_{h = 1}^H r (\sahk) - (V_1^\star (s_1^k))^2 \\
    &\mylei 2 \sqrt{2 \sum_{h = 1}^H \bbV(P_{\sahk}, (V_{h + 1}^\star)^2) \iota} + 6 \iota + 2 \sum_{h = 1}^H \max\{ \underbrace{V_h^\star (s_h^k) - P_{\sahk} V_{h + 1}^\star}_{\ge Q_h^\star (\sahk) - P_{\sahk} V_{h + 1}^\star \ge 0}, 0\} + 1 \\
    &\myleii 4 \sqrt{2 \sum_{h = 1}^H \bbV(P_{\sahk}, V_{h + 1}^\star) \iota} + 7 \iota + 2 \sum_{h = 1}^H (V_{h + 1}^\star (s_{h + 1}^k) - P_{\sahk} V_{h + 1}^\star) + \underbrace{2 V_1^\star (s_1^k)}_{\le 2}\\
    &\myleiii 4 \sqrt{2 \multistepvarepisode \iota} + 9 \iota + 4 \sqrt{2 \sum_{h = 1}^H \bbV(P_{\sahk}, V_{h + 1}^\star) \iota} + 12 \iota\\
    &\le 8 \sqrt{2 \multistepvarepisode \iota} + 21 \iota,
\end{align*}
where (i) is by \Cref{lem:martingale_bound_primal} with $c = \epsilon = 1$, which happens with probability at least $1 - 2 \iota \delta$;
(ii) is by \Cref{lem:var_xy} with $C = 1$;
(iii) is by \Cref{lem:martingale_bound_primal} with $c = \epsilon = 1$, which happens with probability at least $1 - 2 \iota \delta$.
Solving the inequality of $\multistepvarepisode$, we have that
\begin{align*}
    \multistepvarepisode \le 170 \iota.
\end{align*}
So taking a union bound over $k$ we have the desired result.
\end{proof}

\subsection{Model-Free Algorithm: \ucbadvalt~(\Cref{alg:ucbadv_alt})} \label{sec:mf_proof}

\paragraph{Summary of notations.}
Let $\sahk$ and $r_h^k$ denote the state, action and reward at the $h$-th step of the $k$-th episode.
Let $V_h^k (s)$, $Q_h^k (s, a)$, $V_h^{\tref, k}$, $N_h^k (s, a)$ and $\wc{N}_h^k (s, a)$ denote $V_h (s)$, $Q_h (s, a)$, $N_h (s, a)$ and $\wc{N}_h (s, a)$ at the beginning of the $k$-th episode.
Let $V_h^{\REF} := V_h^{\tref, K + 1}$ denote the final reference value function.
Let $N_h^k (s) := \sum_{a \in \cA} N_h^k (s, a)$.
$N_h^{K + 1} (s, a)$ denotes the total number of visits of $(s, a, h)$ after all $K$ episodes are done.

Define $e_1 = H$ and $e_{i + 1} = \floor{(1 + 1 / H) e_i}$.
The definition of stages is with respect to the triple $(s, a, h)$.
For any fixed pair of $k$ and $h$, we say that $(k, h)$ falls in the $j$-th stage of $(s, a, h)$ if and only if $(s, a) = (\sahk)$ and the total visit number of $(\sahk)$ after the $k$-th episode is in $(\sum_{i = 1}^{j - 1} e_i, \sum_{i = 1}^j e_i]$.

Let $\wc{\upsilon}_h^k$, $\mc_h^k$, $\wc{\sigma}_h^k$, $\mu_h^{\tref, k}$, $\sigma_h^{\tref, k}$, $\wh{\VarR}_h^k$, $\bar{b}_h^k$, $\nu_h^{\tref, k}$, $\nc_h^k$ and $b_h^k$ denote $\wc{\upsilon}$, $\mc$, $\wc{\sigma}$, $\mu^{\tref}$, $\sigma^{\tref}$, $\wh{\VarR}$, $\bar{b}$, $\nu^{\tref}$, $\nc$ and $b$ calculated for the value of $Q_h^k (\sahk)$.

For each $k$ and $h$, let $n_h^k$ be the total number of visits to $(\sahk, h)$ prior to the current stage with respect to the same triple and let $n_h^k$ be the number of visits to the same triple during the stage immediately before the current stage.
Let $l_{h, i}^k$ and $\wc{l}_{h, i}^k$ denote the index of the $i$-th episode among the $n_h^k$ and $\wc{n}_h^k$ episodes defined above, respectively.
When $h$ and $k$ are clear from the context, we use $l_i$ and $\wc{l}_i$ for short.

\begin{proof} [Proof of \Cref{thm:main_mf}]
We first run \ucbadvalt~(\Cref{alg:ucbadv_alt}) with $\iota = 99 (\ln (H S A K / \delta) + 1)$ which is large enough for all the probabilistic inequalities to hold.
This choice will make the success probability be $1 - \poly (S, A, H, K, \iota) \delta$.
The lemmas are also proved assuming this choice of $\iota$ at first.

Define
\begin{align*}
    &\psi_{h + 1}^k := \frac{1}{n_h^k} \sum_{i = 1}^{n_h^k} P_{\sahk, h} (\Vrl - V_{h + 1}^{\REF}), \\
    &\xi_{h + 1}^k := \frac{1}{\wc{n}_h^k} \sum_{i = 1}^{\wc{n}_h^k} [P_{\sahk, h} (\Vrlc - V_{h + 1}^\star) - (\Vrlc (\slc) - V_{h + 1}^\star (\slc))], \\
    &\phi_{h + 1}^k := P_{\sahk, h} (V_{h + 1}^\star - V_{h + 1}^{\pi^k}) - (V_{h + 1}^\star (s_{h + 1}^k) - V_{h + 1}^{\pi^k} (s_{h + 1}^k)).
\end{align*}
Combining Section\,4.2 in \citet{paper:ucbadv} with \Cref{lem:mf_opt,lem:err_cnt,lem:sum_opt_gap,lem:db_err,lem:psi,lem:xi,lem:phi,lem:nu_ref,lem:mf_VarR,lem:mf_sum_bonus}, we have that with probability at least $1 - 35 H S A K \iota \delta$,
\begin{align*}
    \Regret (K)
    &\le \sumkh \left( 1 + \frac{1}{H} \right)^{h - 1} (\psi_{h + 1}^k + \xi_{h + 1}^k + \phi_{h + 1}^k + 2 b_h^k) \\
    &\le O (\sqrt{\multistepvar H S A \iota} + \sqrt{H^5 S A K \iota^2 / 2^{2 \is}} + H^5 S^2 A 2^\is \iota^2).
\end{align*}
Taking $\is = \ceil{1/2 \cdot \log_2 (K / H^5 S^3 A \iota^2)}$, we have:
\begin{align*}
    \Regret (K) \le O (\sqrt{\multistepvar H S A \iota} + \sqrt[4]{H^{15} S^5 A^3 K \iota^6}).
\end{align*}
Now we apply \Cref{lem:mf_sum_opt_var}.
With probability at least $1 - 46 H S A K \iota \delta$,
\begin{align*}
    \Regret (K)
    \le O \left( \sqrt{H S A K \iota \sum_{k = 1}^K \Var^{\pi^k}} + \sqrt[4]{H^{15} S^5 A^3 K \iota^6} \right)
    \le O (\sqrt{\maxvar H S A K \iota} + \sqrt[4]{H^{15} S^5 A^3 K \iota^6}).
\end{align*}

The final result is established by scaling $\delta$ to make the success probability be $1 - \delta'$.
We upper-bound $\iota$ by $100 (H S A K / \sqrt{\delta} + 1)$ and solve the inequality:
\begin{align*}
    4600 H S A K \left( \frac{H S A K}{\sqrt{\delta}} + 1 \right) \delta \le \delta'.
\end{align*}
Take $\delta = (\delta' / 7000 H^2 S^2 A^2 K^2)^2$.
By $\ln (H S A K / (\delta / 7000 H^2 S^2 A^2 K^2)^2) \le O (\iota)$, we conclude the proof.
\end{proof}

\begin{restatable}{lemma}{lemMfOpt}  \label{lem:mf_opt}
With probability at least $1 - 15 H S A K \iota \delta$, we have that for any $(s, a, h, k) \in \cS \times \cA \times [H] \times [K]$, 
\begin{align*}
    Q_h^\star (s, a) \le Q_h^{k + 1} (s, a) \le Q_h^k (s ,a).
\end{align*}
\end{restatable} 

\begin{proof} [Proof of \Cref{lem:mf_opt}]
Recall that the update rule is:
\begin{align}
    Q_h (\sahk) \gets \min \left\{\rule{0cm}{0.5cm}\right. \underbrace{\wh{r}_h (\sahk) + \frac{\wc{\upsilon}}{\wc{n}} + \bar{b}}_{\textup{\ding{172}}},\
    \underbrace{ \wh{r}_h (\sahk) + \frac{\mref}{n} + \frac{\mc}{\wc{n}} + b }_{\textup{\ding{173}}},\
    Q_h (\sahk) \left.\rule{0cm}{0.5cm}\right\}. \label{eq:mf_Q_upd}
\end{align}
We prove by induction on $k$.
Clearly for $k = 1$ the argument is true.

For case \ding{172} in \Cref{eq:mf_Q_upd}, we have that (omit the subscripts of $h$ and superscripts of $k$ for simplicity)
\begin{align*}
    Q_h^{k + 1} (s, a)
    &= r_h (s, a) + \frac{1}{\wc{n}} \sum_{i = 1}^{\wc{n}} \Vl (\sli) + (\wh{r}_h (s, a)  - r_h (s, a)) + \bar{b} \\
    &\mygei r_h (s, a) + \frac{1}{\wc{n}} \sum_{i = 1}^{\wc{n}} V_{h + 1}^\star (\sli) + (\wh{r}_h (s, a)  - r_h (s, a)) + \bar{b} \\
    &\mygeii r_h (s, a) + P_{s, a, h} V_{h + 1}^\star - \sqrt{ \frac{H^2 \iota}{2 \wc{n}} } + (\wh{r}_h (s, a)  - r_h (s, a)) + \bar{b} \\
    &\mygeiii r_h (s, a) + P_{s, a, h} V_{h + 1}^\star - \sqrt{ \frac{H^2 \iota}{2 \wc{n}} } - \sqrt{ \frac{ \iota}{2 n} } + \bar{b} \\
    &\ge Q_{h + 1}^\star (s, a),
\end{align*}
where (i) is by induction $V^u \ge V^\star$ for any $1 \le u \le k$;
(ii) is by \Cref{lem:hoeffding} with $b = H$, which holds with probability at least $1 - \delta$;
(iii) is by \Cref{lem:hoeffding} with $b = 1$, which holds with probability at least $1 - \delta$.

Define
\begin{align*}
    \chi_1 &:= \frac{1}{n} \sum_{i = 1}^n (\Vrl (\sli) - P_{s, a, h} \Vrl), \\
    \chi_2 &:= \frac{1}{\wc{n}} \sum_{i = 1}^n [ (\Vl - \Vrl) (\sli) - P_{s, a, h} (\Vl - \Vrl) ].
\end{align*}
For case \ding{173} in \Cref{eq:mf_Q_upd}, we have that
\begin{align*}
    Q_h^{k + 1} (s, a)
    &= \wh{r}_h (s, a) + P_{s, a, h} \left( \frac{1}{n} \sum_{i = 1}^n \Vrl \right) + P_{s, a, h} \left( \frac{1}{\wc{n}} \sum_{i = 1}^{\wc{n}} ( \Vlc - \Vrlc) \right) + \chi_1 + \chi_2 + b \\
    &\mygei r_h (s, a) + P_{s, a, h} \left( \frac{1}{\wc{n}} \sum_{i = 1}^{\wc{n}} \Vlc \right) + \chi_1 + \chi_2 + (r_h (s, a) - \wh{r}_h (s, a)) + b \\
    &\mygeii r_h (s, a) + P_{s, a, h} V_{h + 1}^\star + \chi_1 + \chi_2 + (r_h (s, a) - \wh{r}_h (s, a)) + b \\
    &= Q_h^\star (s, a) + \chi_1 + \chi_2 + (r_h (s, a) - \wh{r}_h (s, a)) + b,
\end{align*}
where (i) is by that $V_{h + 1}^{\tref, u}$ is non-increasing in $u$;
(ii) is by the induction $V^u \ge V^\star$ for any $1 \le u \le k$.

From \Cref{lem:martingale_bound_primal} with $c = H, \epsilon = c^2$, we have that with probability at least $1 - 2 \iota \delta$,
\begin{align}
    \abs{\chi_1} \le \frac{1}{n} \left( 2 \sqrt{2 \underbrace{\sum_{i = 1}^n \bbV (P_{s, a, h}, \Vrl)}_{=: X} \iota} + 6 H \iota \right). \label{eq:mf_chi1}
\end{align}
Define
\begin{align*}
    \chi_3 &:= \sum_{i = 1}^n [P_{s, a, h} (\Vrl)^2 - (\Vrl (\sli))^2], \\
    \chi_4 &:= \frac{1}{n} \left( \sum_{i = 1}^n \Vrl (\sli) \right)^2 - \frac{1}{n} \left( \sum_{i = 1}^n P_{s, a, h} \Vrl \right)^2, \\
    \chi_5 &:= \frac{1}{n} \left( \sum_{i = 1}^n P_{s, a, h} \Vrl \right)^2 - \sum_{i = 1}^n (P_{s, a, h} \Vrl)^2,
\end{align*}
then it is easy to verify that
\begin{align}
    X = n \nref + \chi_3 + \chi_4 + \chi_5. \label{eq:nu_X}
\end{align}
By \Cref{lem:martingale_bound_primal} with $c = H^2, \epsilon = c^2$, and \Cref{lem:var_xy} with $C = H$, we have that with probability at least $1 - 2 \iota \delta$,
\begin{align}
    \chi_3
    \le 2 \sqrt{2 \sum_{i = 1}^n \bbV (P_{s, a, h}, (\Vrl)^2) \iota} + 6 H^2 \iota
    \le 4 H \sqrt{2 X \iota} + 6 H^2 \iota. \label{eq:chi_3}
\end{align}
By \Cref{lem:martingale_bound_primal} with $c = H, \epsilon = c^2$, we have that with probability at least $1 - 2 \iota \delta$,
\begin{align}
    \chi_4
    \le \frac{1}{n} \abs{\sum_{i = 1}^n (\Vrl (\sli) + P_{s, a, h} \Vrl)} \abs{\sum_{i = 1}^n (\Vrl (\sli) - P_{s, a, h} \Vrl)}
    \le 2 H (2 \sqrt{2 X \iota} + 6 H \iota). \label{eq:chi_4}
\end{align}
By Cauchy-Schwarz inequality, $\chi_5 \le 0$.
Thus, $X \le n \nref + 18 H^2 \iota + 8 H \sqrt{2 X \iota}$.
Solving the inequality,
\begin{align*}
    X \le 2 n \nref + 164 H^2 \iota.
\end{align*}
Plugging back into \Cref{eq:mf_chi1}, we have
\begin{align*}
    \chi_1 \le 4 \sqrt{\frac{\nref \iota}{n}} + \frac{(4 \sqrt{82} + 6) H \iota}{n}.
\end{align*}

By a similar reasoning, we have that with probability at least $1 - 6 \iota \delta$,
\begin{align*}
    \chi_2 \le 4 \sqrt{\frac{\nc \iota}{\wc{n}}} + \frac{(4 \sqrt{82} + 6) H \iota}{\wc{n}}.
\end{align*}

By \Cref{lem:bennett_empirical} and $\frac{1}{x - 1} \le \frac{2}{x}$, we have that with probability at least $1 - \delta$,
\begin{align}
    \abs{r_h (s, a) - \wh{r}_h (s, a)}
    \le 2 \sqrt{\frac{\wh{\VarR}_h (s, a) \iota}{n}} + \frac{14 \iota}{3 n}.
\end{align}

Therefore, we have $b \ge \abs{\chi_1} + \abs{\chi_2} + \abs{r_h (s, a) - \wh{r}_h (s, a)}$, which means $Q_h^{k + 1} (s, a) \ge Q_h^\star (s, a)$.
\end{proof}

\begin{restatable}[Adapted from Lemma\,5 and Corollary\,6 in \citet{paper:ucbadv}, and Corollary\,6 in \citet{paper:lcb_ssp}]{lemma}{lemErrCnt} \label{lem:err_cnt}
Conditioned on the successful events of \Cref{lem:mf_opt}, with probability at least $1 - H K \delta$, we have that for any $\epsilon \in (0, H]$ and any $h \in [H]$,
\begin{align*}
    \sum_{k = 1}^K \II[V_h^k (s_h^k) - V_h^\star (s_h^k) \ge \epsilon] \le 60000 \frac{H^5 S A \iota}{\epsilon^2} =: N_0 (\epsilon).
\end{align*}
As a result, for every state $s$ we have that
\begin{align*}
    %
    N_h^k (s) \ge N_0 (\epsilon) &\implies 0 \le V_h^k (s) - V_h^\star (s) \le \epsilon.
\end{align*}
\end{restatable}

\begin{proof} [Proof of \Cref{lem:err_cnt}]
To derive the constant $60000$, we only need to solve the inequality:
\begin{align*}
    \sum_{k = 1}^K \II [\delta_h^k \ge \epsilon]
    \le \frac{\sum_{k = 1}^K \II [\delta_h^k \ge \epsilon] \delta_h^k}{\epsilon} 
    \le \frac{240 H^{5 / 2} \sqrt{\norm{w}_\infty S A \iota \sum_{k = 1}^K \II [\delta_h^k \ge \epsilon]} + 3 S A H^3 \norm{w}_\infty}{\epsilon}
\end{align*}
which is below Equation\,(48) in \citet{paper:ucbadv}, using $x \le a \sqrt{x} + b \implies x \le a^2 + 2 b$.

The second part can be proven in a similar way as Corollary\,6 in \citet{paper:lcb_ssp}.
\end{proof}

\begin{restatable}{lemma}{lemSumOptGap} \label{lem:sum_opt_gap}
Conditioned on the successful events of \Cref{lem:err_cnt}, we have that
\begin{gather*}
    \sumkh (V_h^k (s_h^k) - V_h^\star (s_h^k)) \le O ( \sqrt{H^7 S A K \iota} ), \\
    \sumkh (V_h^k (s_h^k) - V_h^\star (s_h^k))^2 \le O ( H^6 S A \iota^2 ).
\end{gather*}
\end{restatable}

\begin{proof} [Proof of \Cref{lem:sum_opt_gap}]
Let $c$ be a fixed constant, then
\begin{align*}
    \sumkh (V_h^k (s_h^k) - V_h^\star (s_h^k))
    &= \sumkh (V_h^k (s_h^k) - V_h^\star (s_h^k)) \II [V_h^k (s_h^k) - V_h^\star (s_h^k) < c] \\
    &\quad + \sumkh (V_h^k (s_h^k) - V_h^\star (s_h^k)) \II [V_h^k (s_h^k) - V_h^\star (s_h^k) \ge c] \\
    &\mylei c H K + \sumkh  \int_0^H \II [V_h^k (s_h^k) - V_h^\star (s_h^k) \ge x] \II [(V_h^k (s_h^k) - V_h^\star (s_h^k)) \ge c] \dif x \\
    &= c H K + \sumkh \int_c^H \II [V_h^k (s_h^k) - V_h^\star (s_h^k) \ge x] \dif x \\
    &= c H K + \int_c^H \left(\sumkh \II [V_h^k (s_h^k) - V_h^\star (s_h^k) \ge x] \right) \dif x \\
    &\myleii c H K + \int_c^H O \left( \frac{H^6 S A \iota}{x^2} \right) \dif x \\
    &\le O \left(c H K + \frac{H^6 S A \iota}{c} \right),
\end{align*}
where (i) is by $n = \int_0^\infty \II [n \ge x] \dif x$ for any non-negative real $n$ and $V_h^\star \le V_h^k \le H$ (\Cref{lem:mf_opt});
(ii) is by \Cref{lem:err_cnt}.
Taking $c = \sqrt{H^5 S A \iota / K}$ gives the first result.

Similarly,
\begin{align*}
    &\sumkh (V_h^k (s_h^k) - V_h^\star (s_h^k))^2 \\
    &= \sumkh (V_h^k (s_h^k) - V_h^\star (s_h^k))^2 \II [V_h^k (s_h^k) - V_h^\star (s_h^k) < c] \\
    &\quad + \sumkh (V_h^k (s_h^k) - V_h^\star (s_h^k))^2 \II [V_h^k (s_h^k) - V_h^\star (s_h^k) \ge c] \\
    &\le c^2 H K + \sumkh \left( \int_0^H \II [V_h^k (s_h^k) - V_h^\star (s_h^k) \ge x] \II [(V_h^k (s_h^k) - V_h^\star (s_h^k)) \ge c] \dif x \right)^2 \\
    &= c^2 H K + \sumkh \left( \int_c^H \II [V_h^k (s_h^k) - V_h^\star (s_h^k) \ge x] \dif x \right)^2 \\
    &= c^2 H K + \int_c^H \left( \int_c^H \left(\sumkh \II [V_h^k (s_h^k) - V_h^\star (s_h^k) \ge x] \II [V_h^k (s_h^k) - V_h^\star (s_h^k) \ge y] \right) \dif x \right) \dif y \\
    &= c^2 H K + \int_c^H \left( \int_c^y \left(\sumkh \II [V_h^k (s_h^k) - V_h^\star (s_h^k) \ge y] \right) \dif x +  \int_y^H \left(\sumkh \II [V_h^k (s_h^k) - V_h^\star (s_h^k) \ge x] \right) \dif x \right) \dif y \\
    &\le c^2 H K + \int_c^H \left( (y - c) O \left( \frac{H^6 S A \iota}{y^2} \right) + \int_y^H O \left( \frac{H^6 S A \iota}{x^2} \right) \dif x \right) \dif y \\
    &\le c^2 H K + O \left(  H^6 S A \iota \int_c^H \frac{\dif y}{y} \right) \\
    &= O \left(c^2 H K + H^6 S A \iota \ln \frac{H}{c} \right),
\end{align*}
and taking $c = 1 / \sqrt{H K}$ gives the second result.
\end{proof}

\begin{restatable}{lemma}{lemDbErr} \label{lem:db_err}
Define $\beta_i := H / 2^i$ for $i \in \{0, 1, \ldots, \is\}$, $N_0^0 := 0$ and $N_0^i := N_0 (\beta_i) = 60000 \cdot 2^{2 i} S A H^3 \iota$ (defined in \Cref{lem:err_cnt}) for $i \in [\is]$.
Define
\begin{align*}
    B_h^{\tref, k} (s) := \sum_{i = 1}^{\is} \beta_{i - 1} \II [ N_0^{i - 1} \le N_h^k (s) < N_0^i ].
\end{align*}
Conditioned on the successful events of \Cref{lem:err_cnt}, we have that
\begin{gather*}
    V_h^{\tref, k} (s) - V_h^{\REF} (s) \le B_h^{\tref, k} (s), \\
    V_h^{\tref, k} (s) - V_h^\star (s) \le B_h^{\tref, k} (s) + \beta_\is,
\end{gather*}
and
\begin{gather*}
    \sumkh B_h^{\tref, k} (s_h^k) \le O (H^5 S^2 A 2^{\is} \iota), \\
    \sumkh (B_h^{\tref, k} (s_h^k))^2 \le O (H^6 S^2 A \is \iota).
\end{gather*}
\end{restatable}

\begin{proof} [Proof of \Cref{lem:db_err}]
For $i \le \is - 1$, by \Cref{lem:err_cnt}, if $N_h^k (s) \ge N_0^i = N_0 (\beta_i)$ then $V_h^k (s) - V_h^\star (s) \le \beta_i$.
Let $k_0$ be the minimum $k$ such that $N_h^k (s) \ge N_0^i$.
By the updating rule in \Cref{alg:ucbadv_alt} and non-increasing property of $V^k$ (\Cref{lem:mf_opt}), it must satisfy that $V_h^k (s) \le V_h^{\tref, k} (s) \le V_h^{k_0} (s)$.
Since $V_h^\star (s) \le V_h^{\REF} (s)$ (\Cref{lem:mf_opt}), we have that
\begin{align*}
    V_h^{\tref, k} (s) - V_h^{\REF} (s) \le V_h^{k_0} (s) - V_h^\star (s) \le \beta_i.
\end{align*}
If $N_h^k (s) \ge N_0^\is = N_0 (\beta_\is)$ then $V_h^{\tref, k} (s) = V_h^{\REF} (s)$ and $V_h^{\tref, k} (s) - V_h^\star (s) \le \beta_\is$, which corresponds to $B_h^{\tref, k} (s) = 0$.
Since the indicator functions are disjoint, we have the first part of results.

The remaining result is proven by:
\begin{align*}
    \sumkh B_h^{\tref, k} (s_h^k)
    &= \sum_{s \in \cS} \sumkh \sum_{i = 1}^\is \frac{H}{2^{i - 1}} \II [s = s_h^k, N_0^{i - 1} \le N_h^k (s) < N_0^i] \\
    &\le \sum_{s \in \cS} \sum_{h = 1}^H \sum_{i = 1}^\is \frac{H}{2^{i - 1}} N_0^i\\
    &\le O \left( S H \sum_{i = 1}^\is \frac{H}{2^i} \cdot 2^{2 i} S A H^3 \iota \right) \\
    &\le O (H^5 S^2 A 2^{\is} \iota); \\
    \sumkh (B_h^{\tref, k} (s_h^k))^2
    &= \sumkh \sum_{i = 1}^\is \beta_{i - 1}^2 \II [ N_0^{i - 1} \le N_h^k (s) < N_0^i ]\\
    &\le O \left( S H \sum_{i = 1}^\is \frac{H^2}{2^{2 i}} \cdot 2^{2 i} S A H^3 \iota \right) \\
    &\le O (H^6 S^2 A \is \iota).
\end{align*}
This completes the proof.
\end{proof}

\begin{lemma} [Lemma\,11 in \citet{paper:ucbadv}] \label{lem:mf_sum_inv_cnt}
For any non-negative weights $(w_h (s, a))_{s \in \cS, a \in \cA, h \in [H]}$ and $\alpha \in (0, 1)$, it holds that
\begin{gather*}
    \sumkh \frac{w_h (\sahk)}{(n_h^k)^\alpha} \le \frac{2^\alpha}{1 - \alpha} \sum_{s, a, h} w_h (s, a) (N_h^{K + 1} (s, a))^{1 - \alpha}, \\
    \sumkh \frac{w_h (\sahk)}{(\wc{n}_h^k)^\alpha} \le \frac{2^{2 \alpha} H^\alpha}{1 - \alpha} \sum_{s, a, h} w_h (s, a) (N_h^{K + 1} (s, a))^{1 - \alpha}.
\end{gather*}
In the case $\alpha = 1$, it holds that
\begin{gather*}
    \sumkh \frac{w_h (\sahk)}{n_h^k} \le 2 \sum_{s, a, h} w_h (s, a) \ln N_h^{K + 1} (s, a), \\
    \sumkh \frac{w_h (\sahk)}{\wc{n}_h^k} \le 4 H \sum_{s, a, h} w_h (s, a) \ln N_h^{K + 1} (s, a).
\end{gather*}
\end{lemma}

\begin{restatable}{lemma}{lemMfBiasOfSum} \label{lem:mf_bias_of_sum}
For any non-negative sequence $(X_h^k)_{k \in [K], h \in [H]}$, we have that
\begin{gather*}
    \sumkh \frac{1}{n_h^k} \sum_{i = 1}^{n_h^k} X_h^{l_{h,i}^k} \le 2 \iota \sumkh X_h^k, \\
    \sumkh \frac{1}{\wc{n}_h^k} \sum_{i = 1}^{\wc{n}_h^k} X_h^{\wc{l}_{h,i}^k} \le \left( 1 + \frac{1}{H} \right) \sumkh X_h^k.
\end{gather*}
\end{restatable} 

\begin{proof} [Proof of \Cref{lem:mf_bias_of_sum}]
Refer to Equation\,(58) in \citet{paper:ucbadv} for the first inequality.
Refer to Equation\,(15) and the paragraph below it in \citet{paper:ucbadv} for the second inequality.
\end{proof}

\begin{restatable}{lemma}{lemPsi} \label{lem:psi}
Conditioned on the successful events of \Cref{lem:db_err}, with probability at least $1 - \delta$, we have that
\begin{align*}
    \sumkh \left(1 + \frac{1}{H} \right)^{h - 1} \psi_{h + 1}^k
    \le O (H^5 S^2 A 2^\is \iota).
\end{align*}
\end{restatable}

\begin{proof} [Proof of \Cref{lem:psi}]
Since $\psi_h^k$ is non-negative and $(1 + 1 / H)^{h - 1} \le 3$ when $h \le H$, we have that
\begin{align*}
    \sumkh \left(1 + \frac{1}{H} \right)^{h - 1} \psi_{h + 1}^k
    &\le O \left( \sumkh \psi_{h + 1}^k \right) \\
    &= O \left( \sumkh \frac{1}{n_h^k} \sum_{i = 1}^{n_h^k} P_{\sahk, h} (V_{h + 1}^{\tref, l_{h, i}^k} - V_{h + 1}^\REF) \right) \\
    &\mylei O \left( \sumkh \frac{1}{n_h^k} \sum_{i = 1}^{n_h^k} P_{\sahk, h} B_{h + 1}^{\tref, l_{h, i}^k} \right) \\
    &\myleii O \left( \sumkh P_{\sahk, h} B_{h + 1}^{\tref, k} \right) \\
    &\myleiii O \left( \sumkh B_h^{\tref, k} (s_h^k) + H \iota \right) \\
    &\myleiv O (H^5 S^2 A 2^\is \iota),
\end{align*}
where (i) is by \Cref{lem:db_err};
(ii) is by \Cref{lem:mf_bias_of_sum};
(iii) is by \Cref{lem:martingale_conc_mean} with $l = H$, which holds with probability at least $1 - \delta$;
(iv) is by \Cref{lem:db_err}.
\end{proof}

\begin{restatable}{lemma}{lemXi} \label{lem:xi}
Conditioned on the successful events of \Cref{lem:sum_opt_gap}, with probability at least $1 - 5 H S A \iota \delta$, we have that
\begin{align*}
    \sumkh \left(1 + \frac{1}{H} \right)^{h - 1} \xi_{h + 1}^k
    \le O (H^{7/2} S A \iota^{3/2}).
\end{align*}
\end{restatable}

\begin{proof} [Proof of \Cref{lem:xi}]
We borrow the beginning part of proof of Lemma\,15 in \citet{paper:ucbadv}, and perform more fine-grained analyses on the remaining part.
Let $x_h^k$ be the number of elements in current stage with respect to $(\sahk, h)$.
Define
\begin{align*}
    \theta_{h + 1}^j := \left(1 + \frac{1}{H} \right)^{h - 1} \sum_{k = 1}^K \frac{1}{\wc{n}_h^k} \sum_{i = 1}^{\wc{n}_h^k} \II [\wc{l}_{h, i}^k = j], \quad
    \wt{\theta}_{h + 1}^j := \left(1 + \frac{1}{H} \right)^{h - 1} \frac{\floor{(1 + 1 / H) x_h^j}}{x_h^j} \le 3,
\end{align*}
and
\begin{align*}
    \cK := \{ (k, h) \ |\ \theta_{h + 1}^k = \wt{\theta}_{h + 1}^k \}, \quad
    \cK_h^\perp (s, a) := \{ k \ |\ (\sahk) = (s, a), k \textup{ is in the second last stage of } (s, a, h) \}.
\end{align*}
Let $\theta_{h + 1} (s, a)$ and $\wt{\theta}_{h + 1} (s, a)$ denote $\theta_{h + 1}^k$ and $\wt{\theta}_{h + 1}^k$ respectively for some $k \in \cK_h^\perp (s, a)$.
By Equation\,(61) in \citet{paper:ucbadv},
\begin{align*}
    \sumkh \left(1 + \frac{1}{H} \right)^{h - 1} \xi_{h + 1}^k
    &\le \underbrace{\sumkh \wt{\theta}_{h + 1}^k [P_{\sahk, h} (V_{h + 1}^k - V_{h + 1}^\star) - (V_{h + 1}^k (s_{h + 1}^k) - V_{h + 1}^\star (s_{h + 1}^k) )]}_{=: \textup{\ding{172}}} \\
    &\quad + \underbrace{\sum_{(k, h) \in \ov{\cK}} (\theta_{h + 1}^k - \wt{\theta}_{h + 1}^k) [P_{\sahk, h} (V_{h + 1}^k - V_{h + 1}^\star) - (V_{h + 1}^k (s_{h + 1}^k) - V_{h + 1}^\star (s_{h + 1}^k) )]}_{=: \textup{\ding{173}}}.
\end{align*}

We now bound both terms:
\begin{align*}
    \textup{\ding{172}}
    &\mylei O \left( \sqrt{\underbrace{\sumkh \bbV (P_{\sahk, h}, (V_{h + 1}^k - V_{h + 1}^\star))}_{=: Z} \iota} + H \iota \right), \\
    \textup{\ding{173}}
    &\myeqii \sum_{s, a, h} (\theta_{h + 1} (s, a) - \wt{\theta}_{h + 1} (s, a)) \sum_{k \in \cK_h^\perp (s, a)} [P_{\sahk, h} (V_{h + 1}^k - V_{h + 1}^\star) - (V_{h + 1}^k (s_{h + 1}^k) - V_{h + 1}^\star (s_{h + 1}^k) )] \\
    &\le \sum_{s, a, h} \abs{\theta_{h + 1} (s, a) - \wt{\theta}_{h + 1} (s, a)} \abs{ \sum_{k \in \cK_h^\perp (s, a)} [P_{\sahk, h} (V_{h + 1}^k - V_{h + 1}^\star) - (V_{h + 1}^k (s_{h + 1}^k) - V_{h + 1}^\star (s_{h + 1}^k) )] } \\
    &\myleiii O \left( \sum_{s, a, h} \left( \sqrt{ \sum_{k \in \cK_h^\perp (s, a)} \bbV (P_{\sahk, h}, V_{h + 1}^k - V_{h + 1}^\star) \iota } + H \iota \right) \right) \\
    &\myleiv O \left( \sqrt{ H S A \iota \sum_{s, a, h}  \sum_{k \in \cK_h^\perp (s, a)} \bbV (P_{\sahk, h}, V_{h + 1}^k - V_{h + 1}^\star)} + H^2 S A \iota  \right) \\
    &\mylev O (\sqrt{H S A Z \iota} + H^2 S A \iota),
\end{align*}
where (i) is by \Cref{lem:martingale_bound_primal} with $c = 3 H, \epsilon = c^2$, which holds with probability at least $1 - 2 \iota \delta$;
(ii) is by the step above Equation\,(63) in \citet{paper:ucbadv};
(iii) is by $\abs{\theta_{h + 1} (s, a) - \wt{\theta}_{h + 1} (s, a)} \le 3$ and \Cref{lem:martingale_bound_primal} with $c = H, \epsilon = c^2$, which holds with probability at least $1 - 2 H S A \iota \delta$;
(iv) is by Cauchy-Schwarz inequality;
(v) is by the following argument: for any non-negative sequence $(X_h^k)_{k \in [K], h \in [H]}$,
\begin{align*}
    \sum_{s, a, h}  \sum_{k \in \cK_h^\perp (s, a)} X_h^k
    &= \sumkh X_h^k \sum_{s, a, h'} \sum_{k' \in \cK_{h'}^\perp (s, a)} \II [(k', h') = (k, h)] \\
    &= \sumkh X_h^k \sum_{s, a} \II[k \in \cK_h^\perp (s, a)] \\
    &\le \sumkh X_h^k \sum_{s, a} \II[(\sahk) = (s, a)] \\
    &\le \sumkh X_h^k.
\end{align*}

It remains to bound $Z$:
\begin{align*}
    Z
    &\le \sumkh P_{\sahk, h} (V_{h + 1}^k - V_{h + 1}^\star)^2 \\
    &\mylei O \left( \sumkh (V_{h + 1}^k (s_{h + 1}^k) - V_{h + 1}^\star (s_{h + 1}^k))^2 + H^2 \iota \right) \\
    &\myleii O (H^6 S A \iota^2),
\end{align*}
where (i) is by \Cref{lem:martingale_conc_mean} with $l = H^2$, which holds with probability at least $1 - \delta$;
(ii) is by \Cref{lem:sum_opt_gap}.
\end{proof}

\begin{restatable}{lemma}{lemPhi} \label{lem:phi}
With probability at least $1 - 6 \iota \delta$, we have that
\begin{align*}
    \sumkh \left(1 + \frac{1}{H} \right)^{h - 1} \phi_{h + 1}^k
    \le O (H \iota).
\end{align*}
\end{restatable}

\begin{proof} [Proof of \Cref{lem:phi}]
Since $(1 + 1 / H)^{h - 1} \le 3$ when $h \le H$, we have that
\begin{align*}
    \sumkh \left(1 + \frac{1}{H} \right)^{h - 1} \phi_{h + 1}^k
    &= \sumkh \left(1 + \frac{1}{H} \right)^{h - 1} [P_{\sahk, h} (V_{h + 1}^\star - V_{h + 1}^{\pi^k}) - (V_{h + 1}^\star (s_{h + 1}^k) - V_{h + 1}^{\pi^k} (s_{h + 1}^k) )]\\
    &\mylei O \left( \sqrt{\underbrace{\sumkh \bbV (P_{\sahk, h}, V_{h + 1}^\star - V_{h + 1}^{\pi^k})}_{=: Y} \iota} + H \iota \right),
\end{align*}
where (i) is by \Cref{lem:martingale_bound_primal} with $c = 3 H, \epsilon = c^2$, which happens with probability at least $1 - 2 \iota \delta$.
Next we bound $Y$:
\begin{align*}
    Y
    &= \sumkh [ P_{\sahk, h} (V_{h + 1}^\star - V_{h + 1}^{\pi^k})^2 - (V_{h + 1}^\star (s_{h + 1}^k) - V_{h + 1}^{\pi^k} (s_{h + 1}^k) )^2] \\
    &\quad + \sumkh \{(V_h^\star (s_h^k) - V_h^{\pi^k} (s_h^k) )^2 - [P_{\sahk, h} (V_{h + 1}^\star - V_{h + 1}^{\pi^k} )]^2 \} - (V_1^\star (s_1^k) - V_1^{\pi^k} (s_1^k) )^2 \\
    &\mylei 2 \sqrt{2 \sumkh \bbV (P_{\sahk, h}, (V_{h + 1}^\star - V_{h + 1}^{\pi^k})^2 ) \iota} + 6 H^2 \iota \\
    &\quad + 2 H \sumkh \max \{ \underbrace{V_h^\star (s_h^k) - V_h^{\pi^k} (s_h^k) - P_{\sahk, h} (V_{h + 1}^\star - V_{h + 1}^{\pi^k} )}_{\ge Q_h^\star (\sahk) - r_h (\sahk) - P_{\sahk, h} V_{h + 1}^\star = 0}, 0 \}\\
    &\myleii 4 H \sqrt{2 Y \iota} + 6 H^2 \iota + 2 H \sumkh [V_{h + 1}^\star (s_{h + 1}^k) - V_{h + 1}^{\pi^k} (s_{h + 1}^k) - P_{\sahk, h} (V_{h + 1}^\star - V_{h + 1}^{\pi^k} )] \\
    &\quad + \underbrace{2 H (V_1^\star (s_1^k) - V_1^{\pi^k} (s_1^k) ) }_{\le 2 H^2} \\
    &\myleiii 4 H \sqrt{2 Y \iota} + 8 H^2 \iota + 4 H \sqrt{2 Y \iota} + 12 H^2 \iota \\
    &\le 8 H \sqrt{2 Y \iota} + 20 H^2 \iota,
\end{align*}
where (i) is by \Cref{lem:martingale_bound_primal} with $c = H^2, \epsilon = c^2$, which happens with probability at least $1 - 2 \iota \delta$, and $a^2 - b^2 \le (a + b) \max \{a - b, 0\}$ when $a, b \ge 0$;
(ii) is by \Cref{lem:var_xy} with $C = H$;
(iii) is by by \Cref{lem:martingale_bound_primal} with $c = H, \epsilon = c^2$, which happens with probability at least $1 - 2 \iota \delta$.
Solving the inequality of $Y$, we have that
\begin{align*}
    Y \le 168 H^2 \iota.
\end{align*}
Plugging $Y$ back gives the desired result.
\end{proof}

\begin{restatable}{lemma}{lemNuRef} \label{lem:nu_ref}
Conditioned on the successful events of \Cref{lem:db_err}, with probability at least $1 - 4 \iota \delta$, we have that for any $(k, h) \in [K] \times [H]$
\begin{align*}
    \nu_h^{\tref, k} \le O \left( \bbV (P_{\sahk, h}, V_{h + 1}^\star) + \frac{H^2 \iota + \sum_{i = 1}^{n_h^k} P_{\sahk, h} (\Brl)^2}{n_h^k} + \beta_\is^2 \right).
\end{align*}
\end{restatable}

\begin{proof} [Proof of \Cref{lem:nu_ref}]
Let $(k, h)$ be fixed.
We prove by first bounding $\nu_n^{\tref, k} - \frac{1}{n_h^k} \sum_{i = 1}^{n_h^k} \bbV (P_{\sahk, h}, \Vrl)$.
By \Cref{eq:nu_X}, 
\begin{align*}
    \nu_n^{\tref, k} - \frac{1}{n_h^k} \underbrace{\sum_{i = 1}^{n_h^k} \bbV (P_{\sahk, h}, \Vrl)}_{= X \textup{(abusing notation)}} = - \frac{\chi_3 + \chi_4 + \chi_5}{n_h^k}.
\end{align*}
Since we can use \Cref{eq:chi_3,eq:chi_4}, we only need to bound $-\chi_5$.
\begin{align*}
    -\chi_5
    &= \sum_{i = 1}^{n_h^k} (P_{\sahk, h} \Vrl)^2 - \frac{1}{n_h^k} \left( \sum_{i = 1}^{n_h^k} P_{\sahk, h} \Vrl \right)^2 \\
    &\mylei \sum_{i = 1}^{n_h^k} (P_{\sahk, h} \Vrl)^2 - \frac{1}{n_h^k} \left( \sum_{i = 1}^{n_h^k} P_{\sahk, h} V_{h + 1}^\REF \right)^2 \\
    &= \sum_{i = 1}^{n_h^k} ( P_{\sahk, h} \Vrl + P_{\sahk, h} V_{h + 1}^\REF) (P_{\sahk, h} \Vrl - P_{\sahk, h} V_{h + 1}^\REF) \\
    &\le 2 H \sum_{i = 1}^{n_h^k} (P_{\sahk, h} \Vrl - P_{\sahk, h} V_{h + 1}^\REF) \\
    &\myleii 2 H \sum_{i = 1}^{n_h^k} P_{\sahk, h} \Brl,
\end{align*}
where (i) is by $\Vrl \ge V_{h + 1}^\REF$ (\Cref{lem:mf_opt});
(ii) is by \Cref{lem:db_err}.
Combining bounds of $\chi_3$ (\Cref{eq:chi_3}) and $\chi_4$ (\Cref{eq:chi_4}), we have:
\begin{align*}
    \nu_n^{\tref, k} - \frac{X}{n_h^k} \le \frac{8 H \sqrt{2 X \iota} + 18 H^2 \iota + 2 H \sum_{i = 1}^{n_h^k} P_{\sahk, h} \Brl}{n_h^k}.
\end{align*}
Since $8 H \sqrt{2 X \iota} \le X + 32 H^2 \iota$, we have:
\begin{align*}
    \nu_n^{\tref, k} - \frac{2 X}{n_h^k} \le O \left( \frac{H^2 \iota + H \sum_{i = 1}^{n_h^k} P_{\sahk, h} \Brl}{n_h^k} \right).
\end{align*}

For the desired result, we finally bound:
\begin{align*}
    \frac{X}{n_h^k} - 2 \bbV (P_{\sahk, h}, V_{h + 1}^\star)
    &= \frac{1}{n_h^k} \sum_{i = 1}^{n_h^k} (\bbV (P_{\sahk, h}, \Vrl) - 2 \bbV (P_{\sahk, h}, V_{h + 1}^\star)) \\
    &\mylei \frac{2}{n_h^k} \sum_{i = 1}^{n_h^k} \bbV (P_{\sahk, h}, \Vrl - V_{h + 1}^\star) \\
    &\le \frac{2}{n_h^k} \sum_{i = 1}^{n_h^k} P_{\sahk, h} (\Vrl - V_{h + 1}^\star)^2 \\
    &\myleii \frac{2}{n_h^k} \sum_{i = 1}^{n_h^k} P_{\sahk, h} (\Brl + \beta_\is)^2 \\
    &\le \frac{4}{n_h^k} \sum_{i = 1}^{n_h^k} P_{\sahk, h} (\Brl)^2 + 4 \beta_\is^2,
\end{align*}
where (i) is by $\bbV (X + Y) \le 2 \bbV (X) + 2 \bbV (Y)$;
(ii) is by \Cref{lem:db_err}.
So by $H P_{\sahk, h} \Brl \le O (H^2 + P_{\sahk, h} (\Brl)^2)$ we have the result.
\end{proof}

\begin{restatable}[Analogous to \Cref{lem:VarR}]{lemma}{lemMfVarR} \label{lem:mf_VarR}
With probability at least $1 - 2 H S A K \iota \delta$, we have that for any $(s, a, h, k) \in \cS \times \cA \times [H] \times [K]$,
\begin{align*}
    \wh{\VarR}_h^k (s, a) \le O \left( \bbV (R_h (s, a)) + \frac{\iota}{n_h^k (s, a)} \right).
\end{align*}
\end{restatable}

\begin{restatable}{lemma}{lemMfSumBonus} \label{lem:mf_sum_bonus}
Conditioned on the successful events of \Cref{lem:nu_ref,lem:mf_VarR}, with probability at least $1 - \delta$, we have that
\begin{align*}
    \sumkh \left(1 + \frac{1}{H} \right)^{h - 1} b_h^k
    \le O (\sqrt{\multistepvar H S A \iota} + \sqrt{H^5 S A K \iota^2 / 2^{2 \is}} + H^4 S^{3/2} A \is^{1/2} \iota^2).
\end{align*}
\end{restatable}

\begin{proof} [Proof of \Cref{lem:mf_sum_bonus}]
Since $b_h^k$ is non-negative and $(1 + 1 / H)^{h - 1} \le 3$ when $h \le H$, we have that
\begin{align*}
    \sumkh \left(1 + \frac{1}{H} \right)^{h - 1} b_h^k
    &\le O \left( \sumkh \left( \sqrt{\frac{ \nu_h^{\tref, k} \iota }{n_h^k}} + \sqrt{\frac{ \nc_h^k \iota }{\wc{n}_h^k}} + \sqrt{\frac{\wh{\VarR}_h^k \iota}{n_h^k}} + \frac{H \iota}{\wc{n}_h^k} \right) \right) \\
    &\le O \left( \sumkh \left( \sqrt{\frac{ \nu_h^{\tref, k} \iota }{n_h^k}} + \sqrt{\frac{ \nc_h^k \iota }{\wc{n}_h^k}} + \sqrt{\frac{\bbV(R_h (\sahk)) \iota}{n_h^k}} + \frac{H \iota}{\wc{n}_h^k} \right) \right),
\end{align*}
where the last step is by \Cref{lem:mf_VarR}.
Using \Cref{lem:mf_sum_inv_cnt}, we have that
\begin{align*}
    \sumkh \frac{H \iota}{\wc{n}_h^k} \le O (H^3 S A \iota^2).
\end{align*}
Next, we bound the terms of $\nu^\tref$ and $\nc$ separately.

Plugging in \Cref{lem:nu_ref}, we have
\begin{align*}
    &\sumkh \left( \sqrt{\frac{ \nu_h^{\tref, k} \iota }{n_h^k}} + \sqrt{\frac{\bbV(R_h (\sahk)) \iota}{n_h^k}} \right) \\
    &\le O \left(\rule{0cm}{0.9cm}\right. \sumkh \sqrt{\frac{ (\bbV(R_h (\sahk)) + \bbV (P_{\sahk, h}, V_{h + 1}^\star) ) \iota }{n_h^k}} + \sumkh \frac{H \iota}{n_h^k} \\
    &\quad\quad + \sumkh \frac{1}{n_h^k} \sqrt{\sum_{i = 1}^{n_h^k} P_{\sahk, h} (B_{h + 1}^{\tref, l_{h, i}^k})^2 \iota} + \sumkh \sqrt{\frac{\beta_\is^2 \iota}{n_h^k}} \left.\rule{0cm}{0.9cm}\right) \\
    &\mylei O \left(\rule{0cm}{0.9cm}\right. \sum_{s, a, h} \sqrt{N_h^{K + 1} (s, a) (\bbV(R_h (\sahk)) + \bbV (P_{s, a, h}, V_{h + 1}^\star)) \iota} + H^2 S A \iota^2 \\
    &\quad\quad + \sumkh \sqrt{\frac{\iota}{n_h^k}} \sqrt{\frac{1}{n_h^k} \sum_{i = 1}^{n_h^k} P_{\sahk, h} (B_{h + 1}^{\tref, l_{h, i}^k})^2} + \sqrt{\beta_\is^2 \iota} \sum_{s, a, h} \sqrt{N_h^{K + 1} (s, a)} \left.\rule{0cm}{0.9cm}\right) \\
    &\myleii O \left(\rule{0cm}{0.9cm}\right. \sqrt{H S A \iota \sum_{s, a, h} N_h^{K + 1} (s, a) (\bbV(R_h (s, a)) + \bbV (P_{s, a, h}, V_{h + 1}^\star) )} + H^2 S A \iota^2 + \\
    &\quad\quad + \sqrt{\sumkh \frac{\iota}{n_h^k}} \sqrt{\sumkh \frac{1}{n_h^k} \sum_{i = 1}^{n_h^k} P_{\sahk, h} (B_{h + 1}^{\tref, l_{h, i}^k})^2} + \sqrt{H S A \beta_\is^2 \iota \sum_{s, a, h} N_h^{K + 1} (s, a)} \left.\rule{0cm}{0.9cm}\right) \\
    &\myleiii O \left(\rule{0cm}{0.9cm}\right. \sqrt{H S A \iota \underbrace{\sumkh  (\bbV (R_h (\sahk)) + \bbV (P_{\sahk, h}, V_{h + 1}^\star))}_{= \multistepvar}} \\
    &\quad\quad + H^2 S A \iota^2 + \sqrt{H^2 S A \beta_\is^2 K \iota} + \sqrt{H S A \iota^2} \sqrt{\sumkh P_{\sahk, h} (B_{h + 1}^{\tref, k})^2} \left.\rule{0cm}{0.9cm}\right) \\
    &\myleiv O \left(\sqrt{\multistepvar H S A \iota } + H^2 S A \iota^2 + \sqrt{H^2 S A \beta_\is^2 K \iota} + \sqrt{H S A \iota^2} \sqrt{\sumkh (B_h^{\tref, k} (s_h^k))^2 + H \iota} \right) \\
    &\mylev O ( \sqrt{\multistepvar H S A \iota} + \sqrt{H^4 S A K \iota / 2^{2 \is}} + H^{7/2} S^{3/2} A \is^{1/2} \iota^{3/2} ).
\end{align*}
where (i) is by \Cref{lem:mf_sum_inv_cnt};
(ii) is by Cauchy-Schwarz inequality;
(iii) is by \Cref{lem:mf_sum_inv_cnt,lem:mf_bias_of_sum};
(iv) is by \Cref{lem:martingale_conc_mean} with $l = H$, which happens with probability at least $1 - \delta$;
(v) is by \Cref{lem:db_err}.

\begin{align*}
    \sumkh \nc_h^k
    &\le \sumkh \frac{1}{\wc{n}_h^k} \sum_{i = 1}^{\wc{n}_h^k} (V_{h + 1}^{\tref, \wc{l}_{h, i}^k} (s_{h + 1}^{\wc{l}_{h, i}^k}) - V_{h + 1}^{\wc{l}_{h, i}^k} (s_{h + 1}^{\wc{l}_{h, i}^k}))^2 \\
    &\le \sumkh \frac{1}{\wc{n}_h^k} \sum_{i = 1}^{\wc{n}_h^k} (V_{h + 1}^{\tref, \wc{l}_{h, i}^k} (s_{h + 1}^{\wc{l}_{h, i}^k}) - V_{h + 1}^\star (s_{h + 1}^{\wc{l}_{h, i}^k}))^2 \\
    &\mylei \sumkh \frac{1}{\wc{n}_h^k} \sum_{i = 1}^{\wc{n}_h^k} (B_{h + 1}^{\tref, \wc{l}_{h, i}^k} (s_{h + 1}^{\wc{l}_{h, i}^k}) + \beta_\is)^2 \\
    &\myleii O \left( \sumkh (B_h^{\tref, k} (s_h^k))^2 + \beta_\is^2 H K \right) \\
    &\myleiii O (H^6 S^2 A \is \iota + H^3 K / 2^{2 \is}).
\end{align*}
where (i) is by \Cref{lem:db_err};
(ii) is by \Cref{lem:mf_bias_of_sum};
(iii) is by \Cref{lem:db_err}.
So
\begin{align*}
    \sumkh \sqrt{\frac{ \nc_h^k \iota }{\wc{n}_h^k}}
    \le \sqrt{\sumkh \frac{ \iota }{\wc{n}_h^k}} \sqrt{\sumkh \nc_h^k}
    \mylei O (H^4 S^{3/2} A \is^{1/2} \iota^{3/2} + \sqrt{H^5 S A K \iota^2 / 2^{2 \is}}).
\end{align*}
where (i) is by \Cref{lem:mf_sum_inv_cnt}.
\end{proof}

\begin{restatable}{lemma}{lemMfSumOptVar}\label{lem:mf_sum_opt_var}
With probability at least $1 - 11 K \iota \delta$, the following results hold:
For any $k \in [K]$,
\begin{align*}
    \multistepvarepisode \le O ( H^2 \iota ),
\end{align*}
hence
\begin{align*}
    \multistepvar \le O ( H^2 K \iota ).
\end{align*}
Alternatively,
\begin{align*}
    \multistepvar
    \le O \left( \sum_{k = 1}^K \Var^{\pi^k} + H^2 \iota^2 \right)
    \le O (\maxvar K + H^2 \iota^2).
\end{align*}
\end{restatable}

\begin{proof} [Proof of \Cref{lem:mf_sum_opt_var}]
We first prove the result depending on $\multistepvar$ similar to \Cref{lem:M6}.
For any $k \in [K]$,
\begin{align*}
    \multistepvarepisode
    &\mylei \sum_{h = 1}^H [P_{\sahk, h} (V_{h + 1}^\star)^2 - (V_{h + 1}^\star (s_{h + 1}^k))^2] + \sum_{h = 1}^H [(V_h^\star (s_h^k))^2 - (P_{\sahk, h} V_{h + 1}^\star)^2] + \sum_{h = 1}^H r_h (\sahk) - (V_1^\star (s_1^k))^2 \\
    &\myleii 2 \sqrt{2 \sum_{h = 1}^H \bbV(P_{\sahk, h}, (V_{h + 1}^\star)^2) \iota} + 6 H^2 \iota + 2 H \sum_{h = 1}^H \max\{ \underbrace{V_h^\star (s_h^k) - P_{\sahk, h} V_{h + 1}^\star}_{\ge Q_h^\star (\sahk) - P_{\sahk, h} V_{h + 1}^\star \ge 0}, 0\} + H\\
    &\myleiii 4 H \sqrt{2 \sum_{h = 1}^H \bbV(P_{\sahk, h}, V_{h + 1}^\star) \iota} + 7 H^2 \iota + 2 H \sum_{h = 1}^H (V_{h + 1}^\star (s_{h + 1}^k) - P_{\sahk, h} V_{h + 1}^\star) + \underbrace{2 H V_1^\star (s_1^k)}_{\le 2 H^2}\\
    &\myleiv 4 H \sqrt{2 \multistepvarepisode \iota} + 9 H^2 \iota + 4 H \sqrt{2 \sum_{h = 1}^H \bbV(P_{\sahk, h}, V_{h + 1}^\star) \iota} + 12 H \iota\\
    &\le 8 H \sqrt{2 \multistepvarepisode \iota} + 21 H^2 \iota,
\end{align*}
where (i) is by by \Cref{lem:bhatia_davis}, $\bbV (R_h (s, a)) \le \bbE [R_h (s, a)]$;
(ii) is by \Cref{lem:martingale_bound_primal} with $c = H^2, \epsilon = c^2$, which happens with probability at least $1 - 2 \iota \delta$;
(iii) is by \Cref{lem:var_xy} with $C = H$;
(iv) is by \Cref{lem:martingale_bound_primal} with $c = H, \epsilon = c^2$, which happens with probability at least $1 - 2 \iota \delta$.
Solving the inequality of $\multistepvarepisode$, we have that
\begin{align*}
    \multistepvarepisode \le 170 H^2 \iota.
\end{align*}
So taking a union bound over $k$ we have the first result.

Next we prove the result depending on $\maxvar$.
This is similar to the proof of \Cref{lem:M4}.
Define a series of random variables and their truncated values: for any $k \in [K]$,
\begin{align*}
    W^k  &:= \sum_{h = 1}^H (\bbV (R_h (\sahk) ) + \bbV (P_{\sahk, h}, V_{h + 1}^{\pi^k})),
    \quad \ov{W^k} := \min\{ W^k, 50 H^2 \iota \}.
\end{align*}
By $\bbV (P_{s, a, h}, V_{h + 1}^\star) \le 2 \bbV (P_{s, a, h}, V_{h + 1}^{\pi^k}) + 2 \bbV (P_{s, a, h}, V_{h + 1}^\star - V_{h + 1}^{\pi^k})$, we know that
\begin{align*}
    \multistepvar \le 2 \sum_{k = 1}^K W^k + 2 Y,
\end{align*}
where $Y \le O (H^2 \iota)$ (with probability at least $1 - 4 \iota \delta$) is defined in the proof of \Cref{lem:phi}.
Correspondingly, define the following event, which means there is no truncation:
\begin{align*}
    \cE_{W} &:= \{ W^k = \ov{W^k},\ \forall k \in [K] \}.
\end{align*}
For any fixed $1 \le k \le K$,
\begin{align*}
    W^k
    & \le \sum_{h=1}^H [P_{\sahk, h} (V_{h+1}^{\pi^k})^2 - (V_{h+1}^{\pi^k} (s_{h+1}^k))^2] + \sum_{h=1}^H [ (V_h^{\pi^k}(s_h^k))^2 - (P_{\sahk, h}V_{h+1}^{\pi^k})^2]  + \sum_{h = 1}^H r_h(\sahk) - (V_1^{\pi^k} (s_1^k))^2 \\
    & \mylei 2 \sqrt{ 2 \sum_{h=1}^H \bbV(P_{\sahk, h}, (V_{h+1}^{\pi^k})^2 ) \iota } + 6 H^2 \iota + 2 H \sum_{h=1}^H (V_h^{\pi^k}(s_h^k) - P_{\sahk, h}V_{h+1}^{\pi^k}) + H \\
    & \myleii 4 H \sqrt{ 2 \sum_{h=1}^H \bbV(P_{\sahk, h}, V_{h+1}^{\pi^k} ) \iota } + 7 H^2 \iota + 2 H \sum_{h=1}^H r_h (\sahk) \\
    & \le 4 H \sqrt{2 W^k \iota} + 9 H^2 \iota,
\end{align*}
where (i) is by \Cref{lem:martingale_bound_primal} with $c = H^2, \epsilon = c^2$, which happens with probability at least $1 - 2 \iota \delta$;
(ii) is by \Cref{lem:var_xy} with $C = H$.
Solving the inequality, $W^k \le 50 H^2 \iota$.
This means, $\bbP [\cE_W] \ge 1 - 2 K \iota \delta$.
Now on suppose $\cE_W$ holds, then
\begin{align*}
    \sum_{k = 1}^K W^k
    &= \sum_{k = 1}^K \ov{W^k} \\
    &\mylei 3 \sum_{k = 1}^K \bbE[\ov{W^k} \ |\ \cF_k] + 50 H^2 \iota^2 \\
    &\le 3 \sum_{k = 1}^K \bbE[W^k\ |\ \cF_k] + 50 H^2 \iota^2 \\
    &= 3 \sum_{k = 1}^K \Var_1^{\pi^k} (s_1^k) + 50 H^2 \iota^2 \\
    &\le 3 \sum_{k = 1}^K \Var^{\pi^k} + 50 H^2 \iota^2 \\
    &\le 3 \maxvar K + 50 H^2 \iota^2,
\end{align*}
where (i) is by \Cref{lem:martingale_conc_mean} with $l = 50 H^2 \iota$, which happens with probability at least $1 - \delta$.
\end{proof}

\subsection{Proof of Lower Bounds}
\label{sec:lb_proof}

We modify Theorem\,9 in \citet{paper:episodic_lower_bound} for a bounded-reward, time-homogeneous lower bound (\Cref{thm:lower_bound_mvp}).
\Cref{thm:lower_bound} is much more straightforward.
To this end, we borrow necessary notations from \citet{paper:episodic_lower_bound}, adapted to time-homogeneous setting.

A policy $\pi$ interacting with an MDP $\cM$ defines a stochastic process denote by $((S_h^k, A_h^k, R_h^k)_{h \in [H]})_{k \ge 1}$, where $S_h^k, A_h^k$ and $R_h^k$ are the random variables representing the state, the action and the reward at time $h$ of episode $k$.
As explained by \citet{paper:bandit_algorithms}, the Ionescu-Tulcea theorem ensures the existence of probability space $(\Omega, \cF, \bbP_\cM)$ such that
\begin{align*}
    \bbP_\cM [S_{h + 1}^k = s | A_h^k, I_h^k] = P (s | S_h^k, A_h^k),
    \quad \textup{and} \quad
    \bbP_\cM [A_h^k = a | I_h^k] = \pi_h^k (a | I_h^k),
\end{align*}
where $\boldpi = (\pi_h^k)_{k \in [K], h \in [H]}$ and
\begin{align*}
    I_h^k &:= ( S_1^1, A_1^1, R_1^1, \ldots, S_H^1, A_H^1, R_H^1, S_1^2, A_1^2, R_1^2, \ldots, S_H^{k - 1}, A_H^{k - 1}, R_H^{k - 1}, S_1^k, A_1^k, R_1^k, \ldots, S_h^k)
\end{align*}
is the random vector containing all state-action pairs observed up to step $h$ of episode $k$, but not including $A_h^k$.
Here we assume the rewards are \emph{deterministic} as in \citet{paper:episodic_lower_bound}.
Next, we denote by $\bbP_\cM^{I_H^K}$ the pushforward measure of $I_H^K$ under $\bbP_\cM$,
\begin{align}
    \bbP_\cM^{I_H^K} [i_H^K]
    := \bbP_\cM [I_H^K = i_H^K]
    = \prod_{k = 1}^K \prod_{h = 1}^H \pi_h^k (a_h^k | i_h^k) P (s_{h + 1}^k | \satk) , \label{eq:pushforward_measure}
\end{align}
where $i_H^K$ is a realization of $I_H^K$.

\begin{restatable}{definition}{defKL} \label{def:KL}
The Kullback-Leibler divergence between two distributions $\bbP_1$ and $\bbP_2$ on a measurable space $(\Omega, \cG)$ is defined as
\begin{align*}
    \KL (\bbP_1, \bbP_2) := \int_\Omega \ln \left( \frac{\dif \bbP_1}{\dif \bbP_2} (\omega) \right) \dif \bbP_1 (\omega),
\end{align*}
if $\bbP_1 \ll \bbP_2$ and $+\infty$ otherwise.
For Bernoulli distributions, we define $\forall (p, q) \in [0, 1]^2$,
\begin{align*}
    \kl (p, q) := \KL (\cB(p), \cB(q)) = p \ln \left( \frac{p}{q} \right) + (1 - p) \ln \left( \frac{1 - p}{1 - q} \right).
\end{align*}
\end{restatable}

\begin{restatable}[Adapted from Lemma\,5 in \citet{paper:episodic_lower_bound}]{lemma}{lemKLPushforward} \label{lem:KL_pushforward}
Let $\cM$ and $\cM'$ be two MDPs that are identical except for their transition probabilities, denoted by $P$ and $P'$, respectively.
Assume that we have $\forall (s, a), P(\cdot | s, a) \ll P'(\cdot | s, a)$.
Then for any $K$,
\begin{align*}
    \KL \left( \bbP_\cM^{I_H^K}, \bbP_{\cM'}^{I_H^K} \right)
    = \sum_{(s, a) \in \SA} \bbE_\cM \left[N_{s, a}^K\right] \KL (P(\cdot | s, a), P'(\cdot | s, a)),
\end{align*}
where $N_{s, a}^K := \sum_{k = 1}^K \sum_{h = 1}^H \II[(S_h^k, A_h^k) = (s, a)]$.
\end{restatable}

\begin{restatable}[Lemma\,1 in \citet{paper:regret_bandit}]{lemma}{lemKLkl} \label{lem:KL_kl}
Consider a measurable space $(\Omega, \cF)$ equipped with two distributions $\bbP_1$ and $\bbP_2$.
For any $\cF$-measurable function $Z : \Omega \to [0, 1]$, we have
\begin{align*}
    \KL (\bbP_1, \bbP_2) \ge \kl (\bbE_1 [Z], \bbE_2 [Z]),
\end{align*}
where $\bbE_1$ and $\bbE_2$ are the expectations under $\bbP_1$ and $\bbP_2$ respectively.
\end{restatable}

\begin{proof} [Proof of \Cref{thm:lower_bound_mvp}]
We retain most of the proof of Theorem\,9 and Appendix\,C in \citet{paper:episodic_lower_bound}, while incorporating the hard instance design in Section\,5.5.1 in \citet{paper:horizon_free_lmdp}.
Namely, we change the $A$-ary tree in \citet{paper:episodic_lower_bound} with the binary tree in \citet{paper:horizon_free_lmdp}.
This change does not affect the proof, while circumventing the requirement of $S = 3 + (A^d - 1) / (A - 1)$ where $d$ is the tree height.
We can find $S' = 2 + 2^{\floor{\log_2 (S - 2)}} = \Omega(S)$ and replace $S$ with $S'$.
We still use $d = \floor{\log_2 (S - 2)}$ to denote the tree height.

We change the transition at $s_g$: $P (s_b | s_g, a) = 1$ for any $a \in \cA$.
This means for any trajectory, the agent can only get reward once, then loops at $s_b$.

Another important change in design is to scale the reward at $s_g$ by $t \le 1$, with $t$ depending on $\cV$ the variance we desire.
So $r (s_g, a) = t$.

To be precise, let $\bbE_0$ and $\bbE_{(\ell^\star, a^\star)}$ be the expectation taken with respect to the reference MDP (with no special leaf-action pair) and $\cM_{(\ell^\star, a^\star)}$.
We have that
\begin{align*}
    \cR_K (\boldpi, \cM_{(\ell^\star, a^\star)})
    \ge t K \varepsilon \left( 1 - \frac{1}{K} \bbE_{(\ell^\star, a^\star)} \left[ N_{(\ell^\star, a^\star)}^K \right]\right),
\end{align*}
where $N_{(\ell^\star, a^\star)}^K = \sum_{k = 1}^K \II [(S_{d + 1}^k, A_{d + 1}^k) = (s, a)]$.
Hence,
\begin{align}
    \max_{(\ell^\star, a^\star)} \cR_K (\boldpi, \cM_{(\ell^\star, a^\star)})
    \ge t K \varepsilon \left( 1 - \frac{1}{L A K} \sum_{(\ell^\star, a^\star)} \bbE_{(\ell^\star, a^\star)} \left[ N_{(\ell^\star, a^\star)}^K \right]\right). \label{eq:max_regret_mvp}
\end{align}
Since $N_{(\ell^\star, a^\star)}^K / K \in [0, 1]$, by \Cref{lem:KL_kl},
\begin{align*}
    \kl \left( \frac{1}{K} \bbE_0 \left[ N_{(\ell^\star, a^\star)}^K \right], \frac{1}{K} \bbE_{(\ell^\star, a^\star)} \left[ N_{(\ell^\star, a^\star)}^K \right] \right) 
    \le \KL (\bbP_0^{I_H^K}, \bbP_{(\ell^\star, a^\star)}^{I_H^K}).
\end{align*}
By \Cref{lem:KL_pushforward},
\begin{align*}
    \KL \left( \bbP_0^{I_H^K}, \bbP_{(\ell^\star, a^\star)}^{I_H^K} \right)
    = \bbE_0 \left[ N_{(\ell^\star, a^\star)}^K \right] \kl \left(\frac{1}{2}, \frac{1}{2} + \varepsilon \right).
\end{align*}
Assume that $\varepsilon \le 1 / 4$, then $\kl (1 / 2, 1 / 2 + \varepsilon ) \le 4 \varepsilon^2$.
By Pinsker's inequality, $(p - q)^2 \le \kl (p, q) / 2$, it implies
\begin{align*}
    \frac{1}{K} \bbE_{(\ell^\star, a^\star)} \left[ N_{(\ell^\star, a^\star)}^K \right]
    \le \frac{1}{K} \bbE_0 \left[ N_{(\ell^\star, a^\star)}^K \right] + \sqrt{2} \varepsilon \sqrt{\bbE_0 \left[ N_{(\ell^\star, a^\star)}^K \right]}.
\end{align*}
Since $\sum_{(h^\star, a^\star)} N_{(\ell^\star, a^\star)}^K = K$, by Cauchy-Schwarz inequality we have
\begin{align*}
    \frac{1}{K} \sum_{(\ell^\star, a^\star)} \bbE_{(\ell^\star, a^\star)} \left[ N_{(\ell^\star, a^\star)}^K \right]
    \le 1 + \sqrt{2} \varepsilon \sqrt{L A K}.
\end{align*}
Plugging this back to \Cref{eq:max_regret_mvp}, and taking $\varepsilon = (1 - 1 / L A) \sqrt{L A / 8 K}$, we have
\begin{align*}
    \max_{(\ell^\star, a^\star)} \cR_K (\boldpi, \cM_{(\ell^\star, a^\star)})
    \ge \Omega (t \sqrt{S A K}).
\end{align*}
To ensure that $\varepsilon \le 1 / 4$, we need $K \ge S A$.

Now we calculate the variances.
We know that $V_{d + 2}^\star (s_b) = 0$ and $V_{d + 2}^\star (s_g) = t$.
For any trajectory $\tau$, we look at step $h = d + 1$.
If $(s_h, a_h) \ne (\ell^\star, a^\star)$, then
\begin{align*}
    \Var_\tau^\Sigma \ge \bbV ((1/2, 1/2), (0, t)) = \Omega(t^2).
\end{align*}
If $(s_h, a_h) = (\ell^\star, a^\star)$, then
\begin{align*}
    \Var_\tau^\Sigma \ge \bbV ((1/2 - \varepsilon, 1/2 + \varepsilon), (0, t)) = \left( \frac{1}{4} - \varepsilon^2 \right) \Omega(t^2).
\end{align*}
Notice that $\varepsilon \le 1 / 4$, so $\Var_\tau^\Sigma \ge \Omega (t^2)$ for any $\tau$, and
\begin{align*}
    \maxvar
    \ge \Var^{\pi^\star}
    \ge \min_\tau \Var_\tau^\Sigma
    \ge \Omega (t^2).
\end{align*}
Since the total reward in each episode is upper-bounded by $t$, we know that $\Var_\tau^\Sigma, \maxvar \le O (t^2)$.
Thus,
\begin{align*}
    \Var_\tau^\Sigma, \maxvar = \Theta (t^2).
\end{align*}
For the desired result, we set $t = \Theta (\sqrt{\cV})$.
\end{proof}

\begin{proof} [Proof of \Cref{thm:lower_bound}]
We retain most of the proof of Theorem\,9 in \citet{paper:episodic_lower_bound}, while incorporating the hard instance design in Section\,5.5.1 in \citet{paper:horizon_free_lmdp}.
Namely, we change the $A$-ary tree in \citet{paper:episodic_lower_bound} with the binary tree in \citet{paper:horizon_free_lmdp}.
This change does not affect the proof, while circumventing the requirement of $S = 3 + (A^d - 1) / (A - 1)$ where $d$ is the tree height.
We can find $S' = 2 + 2^{\floor{\log_2 (S - 2)}} = \Omega(S)$ and replace $S$ with $S'$.
We still use $d = \floor{\log_2 (S - 2)}$ to denote the tree height.

Another important change in design is to scale the reward at $s_g$ by $t \le 1$, with $t$ depending on $\cV$ the variance we desire.
So $r_h (s_g, a) = t \II [h \ge \ov{H} + d + 1]$.
This modification does not affect the choice of $\varepsilon$ and $\ov{H}$ in \citet{paper:episodic_lower_bound}, only scales the optimal value and regret linearly, so we have that
\begin{align*}
    \max_{(h^\star, \ell^\star, a^\star)} \cR_K (\boldpi, \cM_{(h^\star, \ell^\star, a^\star)})
    \ge \Omega (t \sqrt{H^3 S A K}).
\end{align*}

Now we calculate the variances.
We know that $V_{\ov{H} + d + 1}^\star (s_b) = 0$ and $V_{\ov{H} + d + 1}^\star (s_g) = t (H - \ov{H} - d) = \Omega (t H)$.
For any trajectory $\tau$, we look at step $h = \ov{H} + d$.
If $(s_h, a_h) \ne (\ell^\star, a^\star)$, then
\begin{align*}
    \Var_\tau^\Sigma \ge \bbV ((1/2, 1/2), (0, \Omega(t H))) = \Omega(t^2 H^2).
\end{align*}
If $(s_h, a_h) = (\ell^\star, a^\star)$, then
\begin{align*}
    \Var_\tau^\Sigma \ge \bbV ((1/2 - \varepsilon, 1/2 + \varepsilon), (0, \Omega(t H))) = \left( \frac{1}{4} - \varepsilon^2 \right) \Omega(t^2 H^2).
\end{align*}
Notice that $\varepsilon \le 1 / 4$ in \citet{paper:episodic_lower_bound}, so $\Var_\tau^\Sigma \ge \Omega (t^2 H^2)$ for any $\tau$, and
\begin{align*}
    \maxvar
    \ge \Var^{\pi^\star}
    \ge \min_\tau \Var_\tau^\Sigma
    \ge \Omega (t^2 H^2).
\end{align*}
Since the total reward in each episode is upper-bounded by $O(t H)$, we know that $\Var_\tau^\Sigma, \maxvar \le O (t^2 H^2)$.
Thus,
\begin{align*}
    \Var_\tau^\Sigma, \maxvar = \Theta (t^2 H^2).
\end{align*}
For the desired result, we set $t = \Theta (\sqrt{\cV} / H)$.
\end{proof}


\end{document}